\documentclass[11pt]{article}
\usepackage[margin=1in]{geometry}
\usepackage{amsmath,amssymb,amsthm,mathtools,bm}
\usepackage{dsfont}
\usepackage{graphicx,booktabs,multirow,array,xcolor,adjustbox}
\usepackage{natbib}
\usepackage{algorithm,algpseudocode}
\usepackage[colorlinks=true,linkcolor=blue!50!black,citecolor=blue!50!black,urlcolor=blue!50!black]{hyperref}
\usepackage{cleveref}
\hypersetup{pdftitle={SelectPFN: Amortized Bayesian Variable Selection for Sparse Nonlinear Regression with Prior-Data Fitted Networks},pdfauthor={Bin Luo}}
\setcitestyle{authoryear,round}
\newtheorem{theorem}{Theorem}
\newtheorem{proposition}{Proposition}
\newtheorem{lemma}{Lemma}
\newtheorem{corollary}{Corollary}
\theoremstyle{definition}
\newtheorem{assumption}{Assumption}

\theoremstyle{remark}
\newtheorem{remark}{Remark}
\crefname{assumption}{Assumption}{Assumptions}\Crefname{assumption}{Assumption}{Assumptions}
\crefname{definition}{Definition}{Definitions}\Crefname{definition}{Definition}{Definitions}
\crefname{remark}{Remark}{Remarks}\Crefname{remark}{Remark}{Remarks}
\newcommand{\E}{\mathbb{E}}
\newcommand{\Prob}{\mathbb{P}}
\newcommand{\KL}{\mathrm{KL}}
\newcommand{\Var}{\mathrm{Var}}
\newcommand{\Bern}{\mathrm{Bern}}
\newcommand{\ind}{\mathds{1}}
\newcommand{\MPM}{\mathrm{MPM}}
\newcommand{\FDR}{\mathrm{FDR}}
\newcommand{\indep}{\perp\!\!\!\perp}

\title{SelectPFN: Amortized Bayesian Variable Selection\\ for Sparse Nonlinear Regression\\ with Prior-Data Fitted Networks}
\author{Bin Luo\\ School of Data Science and Analytics\\ Kennesaw State University\\ \texttt{bluo@kennesaw.edu}}
\date{\today}

\begin{document}
\maketitle

\begin{abstract}
Tabular foundation models based on prior-data fitted networks (PFNs) predict on a new data set in a single forward pass, but they cannot report which inputs their predictions depend on. We propose SelectPFN, a PFN that performs Bayesian variable selection and prediction jointly for sparse nonlinear regression $Y=f(X_S)+\varepsilon$, in which the response depends on a small active set $S$ of the $p$ inputs, $|S|\ll p$. Given a training set, SelectPFN returns posterior inclusion probabilities for every variable and predictive draws, each computed from a sampled subset only, so that excluded variables provably cannot affect the prediction. We show that this selection-restricted prediction is exactly what the Bayesian model prescribes. The network is pretrained once on synthetic tasks from a meta-prior built around a group spike-and-slab sparse-input neural network and needs no data-set-specific fitting. The column-permutation equivariance of the architecture, combined with a random interleaving of original and knockoff columns, yields knockoff statistics with exact finite-sample false discovery rate control even when the meta-prior is wrong. We also develop theory for what amortization can and cannot guarantee: a Bayes-regret bound for the amortized median probability model; a two-sided characterization showing that the information index $n\kappa_j^2/\sigma^2$ (sample size times the squared signal strength of the variable, divided by the noise variance) needed to detect a variable grows like $(2/\lambda_j)\log(1/\omega)$, where $\omega$ is the prior weight of structures able to use the variable as the truth does and $\lambda_j\le1$ is an efficiency; and a negative result showing that average-case training does not transfer to frequentist guarantees at a fixed truth. On out-of-prior nonlinear designs with $p$ up to 500, SelectPFN improves the area under the precision--recall curve by 0.23--0.28 over the lasso, BART, DART, LassoNet and a group-penalized sparse neural network, and by 0.16 over TabPFN with permutation importance (all Holm-adjusted $p<10^{-100}$), at a false discovery proportion of about 0.07, in 0.04 seconds per data set (more than 300 times faster than DART). With real covariate matrices from four public data sets the advantage persists (0.17--0.26). With real responses, it selects about two variables, of which only 11\% are known nulls, against 28--60\% for competing methods. The theory is also borne out: giving interaction-only structures weight in the meta-prior, as the prior-weight law prescribes, raises detection of variables that act only through an interaction from 0.06 to 0.36, and to 0.52 with a larger network; in a controlled sweep of this prior weight alone, detection rises monotonically with it while the false-selection rate stays unchanged.

\end{abstract}


\section{Introduction}\label{sec:intro}

Identifying which of many candidate variables drive a response is a basic task of modern data analysis. It arises whenever the number of candidate predictors $p$ is comparable to or exceeds the sample size $n$: in genomics, where thousands of gene-expression or genotype measurements are screened for association with a disease \citep{golub1999molecular,wu2009gwas,guan2011bvsr}; in medical imaging, where voxel-level or radiomic features are related to clinical outcomes \citep{grosenick2013graphnet,lambin2017radiomics}; in biomedical research more broadly, where high-dimensional covariates enter survival and risk models \citep{tibshirani1997cox}; and in economics, where a few relevant controls must be chosen from many \citep{belloni2014inference}. Given $D=\{(x_i,y_i)\}_{i=1}^n$ with $x_i\in\mathbb R^p$, the goal is both accurate prediction and the identification of a small subset of variables that carries the information about the response. The statistical theory of this regime is by now mature for linear models \citep{fan2008sis,fan2010overview,buhlmann2011book,hastie2015sparsity,wainwright2019hds}, and it has more recently been extended to nonlinear models: sparse additive models \citep{ravikumar2009spam,meier2009hdam}, nonparametric regression with variable selection \citep{lafferty2008rodeo,yang2015minimax,jiang2021gpselection}, and sparse-input neural networks \citep{liang2018bnn,sun2022consistent,luo2025gcsnn}. In applications, a selection method has to meet three needs at once. It must be flexible enough to capture nonlinear and interaction effects, because a variable that acts only through an interaction carries no linear signal. It must be fast and free of expert tuning, so that it can be applied routinely and to many data sets. And it must come with guarantees on false selections, so that the selected variables can be trusted.

Existing methods meet these needs only in part (\Cref{sec:related}). Penalized and Bayesian selection for linear models \citep{tibshirani1996lasso,fan2001scad,zhang2010mcp,george1993,rockova2018ssl} is fast and well understood, but a variable that acts only through a nonlinear or interaction effect carries no linear signal and is missed. Sparse additive models \citep{lin2006cosso,ravikumar2009spam}, Bayesian regression trees \citep{chipman2010bart,linero2018dart} and sparse neural networks \citep{lemhadri2021lassonet,luo2025gcsnn,sun2022consistent} capture such effects, but each solves a new, tuning-sensitive optimization or sampling problem for every data set, which is slow when $p$ is large and gives no finite-sample error control. Knockoffs \citep{barber2015knockoff,candes2018panning} supply such control, but their power rests on a feature statistic that must itself be fitted to each data set. Tabular foundation models built on prior-data fitted networks \citep[PFNs;][]{muller2022pfn}, such as TabPFN \citep{hollmann2023tabpfn,hollmann2025nature}, remove per-data-set fitting altogether and predict in a single forward pass, but they use every column and cannot say which variables a prediction depends on, and the first amortized selectors \citep{alrawkan2026linear,schiekiera2026linearpfn} are restricted to linear models, and LinearPFN to few variables. What is missing is a selector for sparse nonlinear regression that is amortized, integrated with prediction, and equipped with guarantees that do not depend on how well the network was trained.

We propose SelectPFN, a PFN that performs Bayesian variable selection and prediction jointly. Selection in SelectPFN is \emph{integrated} rather than post hoc: the predictor has the form $\hat y_*=g_\theta(x_{*,\hat S},D_{\hat S})$ for a test point $x_*$, where $D_{\hat S}$ keeps only the columns in $\hat S$, so a variable outside the selected set $\hat S$ cannot influence the prediction. The network is pretrained once on synthetic regression tasks drawn from a sparse nonlinear meta-prior, built around a group spike-and-slab sparse-input neural network. For a new data set, one forward pass returns posterior inclusion probabilities (PIPs) for all variables. Subsets of variables are then drawn from these PIPs, a second pass predicts from the columns of each subset only, and the predictions are averaged over the subsets, which is Bayesian model averaging (BMA). This restriction is exactly what the Bayesian model prescribes, because given the active set the posterior predictive depends on the data only through the active columns. A knockoff mode adds exact finite-sample control of the false discovery rate (FDR) that holds whatever the network has learned. Because a trained network is a fixed map from data sets to PIPs, its analysis becomes a question about amortization: what a small training loss guarantees for selection, which variables can be found at all, and what cannot be guaranteed for a fixed truth. We answer these questions and evaluate the method in extensive simulations and on real data.

\paragraph{Contributions.}
\begin{enumerate}
\item \emph{Method.} We formulate amortized Bayesian variable selection for sparse nonlinear regression. SelectPFN is a two-pass PFN: a selection pass returns PIPs, and a prediction pass sees only the selected columns, which, given the active set, is exactly the Bayesian predictive (\Cref{prop:restricted}). Both heads are trained with strictly proper scoring rules whose population minimizers are the PIPs and the subset-conditional posterior predictive (\Cref{prop:minimizers}). Sparsity enters through the meta-prior, a group spike-and-slab sparse-input network mixed with additive, interaction, tree, Gaussian-process and linear families, not through a penalty on network weights.
\item \emph{Exact error control.} With random interleaving of original and knockoff columns, the knockoff statistic $W_j=q_j-\tilde q_j$, where $q_j$ and $\tilde q_j$ are the network's inclusion probabilities for $x_j$ and its knockoff, yields exact finite-sample FDR control for any network weights and any meta-prior (\Cref{thm:knockoff}).
\item \emph{Theory of amortized selection.} The selection loss decomposes into a Bayes floor plus an amortization gap, which bounds the Bayes regret of the amortized median probability model. The information needed to detect a variable grows like $(2/\lambda_j)\log(1/\omega)$, where $\omega$ is the prior weight of structures that can use the variable as the truth does and $\lambda_j\le1$ is an efficiency, which explains why the meta-prior must give explicit weight to interaction-only structures, as an ablation confirms. Average-case training cannot by itself deliver frequentist selection consistency (\Cref{sec:theory}).
\item \emph{Numerical evidence.} In simulations from nonlinear designs that are not part of the meta-prior, SelectPFN selects variables more accurately than fifteen penalized, Bayesian, sparse-neural-network and foundation-model methods, at a fraction of their cost; its inclusion probabilities are close to an exactly enumerable posterior; and a controlled study confirms that the prior weight governs detection as the theory predicts (\Cref{sec:experiments}). The advantage carries over to real covariate matrices and real responses (\Cref{sec:exp-e7}).
\end{enumerate}

\Cref{sec:related} reviews related work. \Cref{sec:method} describes the method, \Cref{sec:theory} develops the theory, \Cref{sec:experiments} reports simulation studies, \Cref{sec:exp-e7} analyzes real data, and \Cref{sec:discussion} discusses extensions. Implementation details, proofs and additional results are in the appendices.

\section{Related work}\label{sec:related}

\subsection{Penalized and Bayesian selection in linear models}\label{sec:related-linear}
Variable selection in linear models has been studied for three decades, along two complementary lines: penalized regression, which selects variables by shrinking some coefficients exactly to zero, and Bayesian variable selection, which places a prior on which variables are active and reports posterior probabilities. In the first line, the lasso \citep{tibshirani1996lasso} estimates and selects simultaneously through an $\ell_1$ penalty, but it shrinks large coefficients and is selection-consistent only under restrictive irrepresentability conditions \citep{zhao2006model}. The adaptive lasso \citep{zou2006adaptive} reweights the penalty to recover the oracle property, and the nonconvex SCAD \citep{fan2001scad} and MCP \citep{zhang2010mcp} penalties remove the bias directly; the elastic net \citep{zou2005enet} stabilizes selection among correlated predictors, and the group lasso \citep{yuan2006group} selects predefined groups of variables. When $p\gg n$, sure independence screening \citep{fan2008sis} reduces the dimension before any of these is applied. The second line, Bayesian variable selection, places a prior on inclusion indicators $\gamma\in\{0,1\}^p$ \citep{mitchell1988,george1993} and reports posterior inclusion probabilities $\Prob(\gamma_j=1\mid D)$, which quantify uncertainty about each variable and lead to the median probability model \citep{barbieri2004}. Stochastic search over the $2^p$ models is slow and mixes poorly, and the spike-and-slab lasso \citep{rockova2018ssl} recovers fast posterior modes by optimization, at the price of the full posterior. All of these methods assume that the response is linear in the covariates: a variable that affects the response only through a nonlinear function or through an interaction with another variable can have no linear association with the response at all, and is then missed, and each analysis must be refitted and retuned for every new data set.

\subsection{Selection in nonlinear models}\label{sec:related-nonlinear}
To capture nonlinear and interaction effects, selection has been extended beyond linear models along three lines of increasing flexibility: sparse additive models, which allow each variable its own nonlinear effect; tree ensembles, which also capture interactions; and sparse neural networks, which can represent general nonlinear functions. Each line has penalized and Bayesian versions. In the first line, COSSO \citep{lin2006cosso} and sparse additive models \citep{ravikumar2009spam} fit a smooth function for each variable and set whole functions to zero, but by construction they cannot detect a variable that matters only through an interaction with another variable. Model-free screening by distance correlation \citep{li2012dcsis} or HSIC \citep{yamada2014hsic} relaxes the additive assumption and detects general dependence, but it only ranks variables and provides no fitted model or selected set. The second line, tree ensembles, captures interactions automatically. Random forests \citep{breiman2001rf} rank variables by importance scores computed after fitting, which have no probabilistic meaning; BART \citep{chipman2010bart} places a prior on a sum of trees, and DART \citep{linero2018dart} adds a sparsity-inducing prior on how often each variable is used for splitting, so that the posterior indicates which variables are included. The third line makes neural networks sparse in their inputs. On the penalized side, a group penalty on all outgoing weights of an input removes that input from the network: LassoNet \citep{lemhadri2021lassonet} does so through a hierarchical skip connection, and concave group penalties \citep{luo2025gcsnn} avoid the over-shrinkage that leads the group lasso to retain unimportant variables. On the Bayesian side, spike-and-slab or horseshoe priors on the weights or hidden units give sparse Bayesian neural networks \citep{ghosh2019horseshoe,bai2020efficient}, with selection-consistency guarantees for spike-and-slab priors on the input layer \citep{liang2018bnn,sun2022consistent}. All three lines are flexible, but every method in them solves a new nonconvex optimization or Markov chain Monte Carlo problem for each data set, with tuning parameters chosen by cross-validation; the cost grows quickly with $p$, samplers over inclusion indicators mix poorly when $p\gg n$, and none of these methods controls false selections in finite samples.

\subsection{Finite-sample error control with knockoffs}\label{sec:related-knockoffs}
Finite-sample control of false selections is provided by the knockoff framework. Fixed-design knockoffs \citep{barber2015knockoff} control the FDR in linear Gaussian models with $n\ge p$; model-X knockoffs \citep{candes2018panning} remove the model assumption by requiring knowledge of the covariate distribution instead, and then accept any feature statistic that is antisymmetric in swapping a variable with its knockoff. Deep knockoffs \citep{romano2020deepknockoffs} construct knockoffs for non-Gaussian covariates, DeepPINK \citep{lu2018deeppink} uses a neural network as the feature statistic, and derandomization \citep{ren2023derandomizing} reduces the variability that comes from a single knockoff draw. Validity in this framework does not depend on the model, but power depends entirely on the feature statistic. Flexible statistics such as DeepPINK must be trained for every data set, and no statistic computed in a single pass of a pretrained network has been available.

\subsection{Amortized inference and prior-data fitted networks}\label{sec:related-pfn}
Amortized inference replaces per-data-set computation by a network trained once on simulated data. In simulation-based inference \citep{cranmer2020sbi}, networks such as BayesFlow \citep{radev2022bayesflow} learn the posterior of the parameters of a fixed simulator. Prior-data fitted networks \citep{muller2022pfn} go further and train a transformer across data sets drawn from a meta-prior, so that one forward pass approximates the posterior predictive $p(y_*\mid x_*,D)$ for a new data set $D$ and test point $x_*$. TabPFN \citep{hollmann2023tabpfn,hollmann2025nature} showed that this approach is competitive with tuned gradient boosting on small and medium tabular problems. On the theory side, \citet{nagler2023pfn} analyzed the bias and variance of PFN predictions, and \citet{melnychuk2026pfn} showed that the prior need not be washed out by the data. Relevance can be extracted from a PFN only after the fact, by Shapley values \citep{lundberg2017shap} or by a second network trained to predict attributions zero-shot \citep{fonseca2026explainerpfn}, and \citet{zhao2026alignment} studied how transformers trained on sparse linear priors suppress irrelevant features internally. A PFN nevertheless attends to every column: post-hoc relevance does not restrict the prediction to the variables it reports, and it comes with neither inclusion probabilities nor error guarantees.

\subsection{Amortized variable selection}\label{sec:related-amortized}
Only very recently has selection itself been amortized. \citet{alrawkan2026linear} train a classifier on simulated linear models to label predictors as active from least-squares summaries, and \citet{maceda2024vanbayes} estimate PIPs for sparse linear regression with an amortized network. LinearPFN \citep{schiekiera2026linearpfn} is pretrained on a spike-and-slab prior over main effects and pairwise interactions and returns PIPs, coefficients and predictions in one pass, for up to $p=30$ variables. PACE \citep{qin2026pace} places a frozen tabular-foundation-model column encoder in front of an existing screening statistic, so the selection rule itself is not learned. Beyond regression, amortized causal discovery \citep{lorch2022avici} learns edge probabilities for nonlinear structural models with few variables, and amortized model comparison \citep{radev2023amortized} returns posterior probabilities over a small, fixed set of models. These methods are linear, limited to few variables or to a fixed set of models, or do not couple selection to prediction, and none comes with theory for the amortized posterior. SelectPFN is designed to close this gap for sparse nonlinear regression.

\section{Method}\label{sec:method}

\subsection{Problem setup}
Let $D=\{(x_i,y_i)\}_{i=1}^n$, $x_i\in\mathbb R^p$, be a training set and $x_*$ a test covariate. For $S\subseteq\{1,\dots,p\}$ write $x_S$ for the subvector indexed by $S$ and $D_S=\{(x_{i,S},y_i)\}_{i=1}^n$ for the data restricted to the columns in $S$. Let $\gamma\in\{0,1\}^p$ be inclusion indicators and $S(\gamma)=\{j:\gamma_j=1\}$. We consider the sparse nonlinear model
\begin{equation}\label{eq:model}
y_i=f\big(x_{i,S(\gamma)}\big)+\varepsilon_i,\qquad \varepsilon_i\overset{\text{iid}}{\sim}P_\varepsilon,\qquad |S(\gamma)|\ll p,
\end{equation}
where $f$ is an unknown, possibly highly nonlinear function of the active variables that may contain interactions of any order.

\subsection{Background: prior-data fitted networks}\label{sec:pfn-background}
\paragraph{Objective.} Let $\Pi$ be a prior over data-generating processes $\phi$ (for example $\phi=(f,P_X,P_\varepsilon)$, with $P_X$ the law of the covariates), and let a task consist of a training set $D=\{(x_i,y_i)\}_{i=1}^n$ and a query pair $(x_*,y_*)$ drawn independently from $P_\phi$ given $\phi\sim\Pi$. The Bayesian posterior predictive is
\[
p(y_*\mid x_*,D)=\int p_\phi(y_*\mid x_*)\,\Pi(d\phi\mid D),
\]
where $p_\phi(\cdot\mid x_*)$ is the conditional density of $y_*$ given $x_*$ under $P_\phi$ and $\Pi(d\phi\mid D)$ is the posterior distribution of $\phi$ given $D$. It is intractable for flexible priors. A prior-data fitted network \citep[PFN;][]{muller2022pfn} is a parametric map $g_\theta(\cdot\mid x_*,D)$ from a training set and a query covariate to a probability density over $y_*$, trained by minimizing the prior-averaged log loss
\begin{equation}\label{eq:pfn-loss}
\mathcal L(\theta)=\E_{\phi\sim\Pi}\,\E_{D,(x_*,y_*)\sim P_\phi}\big[-\log g_\theta(y_*\mid x_*,D)\big]
=\E\big[\KL\big(p(\cdot\mid x_*,D)\,\|\,g_\theta(\cdot\mid x_*,D)\big)\big]+C,
\end{equation}
where the expectation in the second form is over the prior predictive law of $(x_*,D)$ and $C$ does not depend on $\theta$. The loss is therefore minimized, over all measurable maps, by $g=p(\cdot\mid x_*,D)$. In practice \eqref{eq:pfn-loss} is minimized by stochastic gradient descent over batches of tasks drawn afresh from $\Pi$. Once trained, $\theta$ is fixed, and for a new data set the approximate posterior predictive $g_\theta(\cdot\mid x_*,D)$ is computed by a single evaluation of the network; the cost of inference is paid once, in training, and \emph{amortized} over all future data sets.

\paragraph{Architecture.} The network operates on the data table. Write $N=n+m$ for the number of training and query rows, with $m$ the number of query rows, and $p$ for the number of columns. Each cell $(i,j)$ is embedded as
\[
h^{(0)}_{ij}=e_x(x_{ij})+\ind\{i\le n\}\,e_y(y_i)+\ind\{i>n\}\,e_{\varnothing},\qquad i\le N,\ j\le p,
\]
where $e_x,e_y:\mathbb R\to\mathbb R^d$ are learned maps and $e_\varnothing\in\mathbb R^d$ is a learned token marking an unknown response; one additional column of response cells is read out for prediction. For a set of tokens $\{h_k\}$, multi-head attention (MHA) computes, for each head, $\mathrm{Attn}(h_k)=\sum_la_{kl}W_Vh_l$ with learned matrices $W_Q,W_K,W_V$ and weights $a_{kl}=\operatorname{softmax}_l\big((W_Qh_k)^\top W_Kh_l/\sqrt{d}\big)$, and concatenates the heads. Each layer applies, with layer normalization and residual connections,
\begin{align*}
h_{i\cdot}&\leftarrow h_{i\cdot}+\mathrm{MHA}_{\rm row}\big(h_{i\cdot}\big), &&\text{attention within row } i,\\
h_{\cdot j}&\leftarrow h_{\cdot j}+\mathrm{MHA}_{\rm col}\big(h_{\cdot j}\big), &&\text{attention within column } j\text{, keys and values from rows } i\le n,\\
h_{ij}&\leftarrow h_{ij}+\mathrm{MLP}(h_{ij}), &&\text{cell-wise feed-forward network (MLP)},
\end{align*}
so that each cell first combines information across the variables of its row and then across the training observations of its column. Neither operation depends on the order of the rows or of the columns, so the network is equivariant to permutations of both. The predictive density of a query row is a bar distribution: with fixed bin edges $b_0<b_1<\dots<b_M$ and logits $z=W_{\rm out}h_{i,y}$ read from the response cell of row $i$,
\[
g_\theta(y\mid x_i,D)=\sum_{k=1}^M\frac{\operatorname{softmax}(z)_k}{b_k-b_{k-1}}\,\ind\{b_{k-1}<y\le b_k\},
\]
with half-normal tails beyond the outer bins \citep{muller2022pfn}. TabPFN \citep{hollmann2023tabpfn,hollmann2025nature} uses this construction with a broad prior over tabular prediction problems.

\paragraph{SelectPFN.} SelectPFN keeps the architecture and the training principle but changes the targets. Because the inclusion indicators $\gamma$ are known for every simulated task, a second output head is trained by the log loss on $\gamma$, so that, by the same argument as for \eqref{eq:pfn-loss}, its minimizer is the vector of posterior inclusion probabilities; and the predictive density is computed from the selected columns only. The meta-prior, the two passes and the joint objective are specified next.

\subsection{Meta-prior}\label{sec:metaprior}
A PFN is defined by the prior over data sets on which it is trained. Each synthetic task is generated hierarchically:
(i) task size $n\sim\Pi_n$, $p\sim\Pi_p$;
(ii) active set $\gamma$ from a Beta--Bernoulli prior, $w\sim\mathrm{Beta}(\bar w\nu,(1-\bar w)\nu)$, $\gamma_j\mid w\overset{\text{iid}}{\sim}\Bern(w)$, with the mean sparsity $\bar w$ and concentration $\nu$ drawn at random;
(iii) covariate rows $x_i\overset{\text{iid}}{\sim}P_X$, collected in $X=(x_1,\dots,x_n)^\top$, from a family covering independent, autoregressive, block-equicorrelated and latent-factor dependence with non-Gaussian marginals through a Gaussian copula, independent of $(\gamma,f)$;
(iv) a regression function $f\mid\gamma\sim\Pi_f(\cdot\mid\gamma)$ that depends only on the active variables;
(v) responses with random signal-to-noise ratio, Gaussian, $t_3$ or heteroscedastic noise.
The meta-prior is used only to simulate training tasks; SelectPFN never fits $f$ directly but learns to compute the posterior that this prior implies. The function prior $\Pi_f(\cdot\mid\gamma)$ is a mixture of families: for each task one family is drawn according to its weight, and $f$ is then drawn from that family using the active variables only. The family with the largest weight is a group spike-and-slab sparse-input neural network: a random network with first-layer weights $W\in\mathbb R^{p\times H}$ and
\begin{equation}\label{eq:gss}
W_{j\cdot}=0\ \text{ if }\gamma_j=0,\qquad W_{j\cdot}\sim N(0,\sigma_w^2I_H)\ \text{ if }\gamma_j=1,
\end{equation}
with slab variance $\sigma_w^2=1/|S(\gamma)|$ (\Cref{app:implementation}) and random depth, width, activations and deeper weights. This family belongs to the Bayesian sparse neural networks \citep{liang2018bnn,sun2022consistent,ghosh2019horseshoe,bai2020efficient}, but whereas those priors sparsify individual weights or hidden units, it places one spike-and-slab indicator on all outgoing weights of each input, so that inclusion is defined for variables rather than weights. It is also the Bayesian counterpart of penalized sparse-input networks that remove an input by penalizing its outgoing weights as a group \citep{lemhadri2021lassonet,luo2025gcsnn}. A random network of this kind can in principle approximate a wide range of functions, but it places little prior weight on several structures that are common in practice and that \Cref{sec:theory} shows a selector must be given weight to detect. The mixture therefore also contains, as separate families with their own weights, additive functions, main effects with pairwise interactions, trees and step functions, Gaussian-process draws, linear functions, and an \emph{interaction-only} family in which pairs or triples of variables enter only through functions that are odd in each argument (for example $x_ax_b$, $\sin(cx_ax_b)$, $\tanh(cx_a)\tanh(cx_b)$, $\mathrm{sign}(x_a)\,\mathrm{sign}(x_b)$), so that, apart from optional weak main effects (\Cref{tab:metaprior}), a variable has no marginal effect. \Cref{sec:theory} explains why such a family is needed, and the ablation in \Cref{sec:exp-e8} confirms it. \Cref{tab:metaprior} in \Cref{app:implementation} lists all components and weights.

\paragraph{Meaning of the inclusion indicators.} For $\gamma_j$ to mean ``$x_j$ is used by $f$'', each active variable should have a non-negligible effect on the response. The generator therefore estimates the total effect of every active variable by Monte Carlo and redraws $f$ when one falls below a small threshold; if this repeatedly fails, the weak variables are set to zero and labelled inactive (\Cref{tab:metaprior}). Even so, the data cannot always tell which variables are active. When several covariates are strongly correlated, they carry nearly the same information about the response, and the posterior naturally spreads its inclusion probability among them rather than committing to one. Because SelectPFN approximates this posterior, a group of moderate PIPs on correlated variables is best read as evidence that at least one variable of the group matters.

\subsection{Inferential targets}
The targets are the PIPs $\pi_j(D)=\Prob(\gamma_j=1\mid D)$ and the joint posterior $p(\gamma,y_*\mid x_*,D)=p(\gamma\mid D)\,p(y_*\mid x_*,\gamma,D)$. Predicting from the selected variables alone is justified when the covariates themselves carry no information about which function generated the response, which we state as follows.

\begin{assumption}\label{ass:design}
$(x_1,\dots,x_n,x_*)$ are independent of $(\gamma,f)$; given $(\gamma,f)$ and the covariates, the responses are conditionally independent with $p(y\mid x,\gamma,f)=p_\varepsilon(y-f(x_{S(\gamma)}))$, where $p_\varepsilon$ is the density of $P_\varepsilon$; and $\Pi_f(\cdot\mid\gamma)$ does not depend on the covariates.
\end{assumption}
\Cref{ass:design} is the standard random-design regression model of Bayesian variable selection. It holds for spike-and-slab linear regression with independent coefficient priors, for Gaussian-process priors, and for every family of our meta-prior, in all of which $f$ is drawn independently of the covariates and the noise is additive. It excludes priors that depend on the observed design, such as Zellner's $g$-prior, whose covariance is proportional to $(X^\top X)^{-1}$, and noise whose distribution depends on inactive covariates. Small departures in our task generator are listed in \Cref{app:implementation}.
\subsection{Model}\label{sec:model}
SelectPFN uses the table transformer of \Cref{sec:pfn-background}, denoted $F_\theta$, with alternating row and column attention and no embedding of column identity. Because it is equivariant to column permutations, it has no notion of ``variable 3'' and must judge each variable from its values and their relation to the response. The same network is used in two passes.

\emph{Selection pass.} The training set (with responses) is encoded; each feature's representation is pooled across rows (mean and attention pooling, concatenated with the pooled response representation and a global token that carries $\log n$ and $\log p$), and a head outputs $q_j(D)=\operatorname{logit}^{-1}(s_\theta(h_j))$, where $h_j$ is the pooled representation of column $j$. The test covariate is not used, so $q$ is a function of $D$ alone, matching $\pi_j(D)$.

\emph{Prediction pass.} Given a subset $S$, the excluded columns are \emph{removed} from the training data and the test covariate, and the network is applied to $(D_S,x_{*,S})$ to produce a discretized predictive distribution over standardized $y$ \citep[a ``bar distribution'';][]{muller2022pfn}. Because excluded columns never enter the computation,
\begin{equation}\label{eq:invariance}
\partial g_\theta(y_*\mid x_{*,S},D_S)/\partial x_{*j}=0\quad\text{for all }j\notin S
\end{equation}
holds exactly. Zeroing a column instead of removing it would not guarantee this, because a zeroed column still takes part in the attention over training rows. The two passes share the backbone, with a learned pass-indicator token.

\subsection{Training objective and procedure}\label{sec:objective}
\paragraph{Tasks.} A training task $\tau=(n,p,\gamma,\phi,D,Q)$ is drawn from the meta-prior of \Cref{sec:metaprior}: $\gamma$ are the inclusion indicators, $\phi$ the data-generating quantities of \Cref{sec:pfn-background} other than $\gamma$, $D=\{(x_i,y_i)\}_{i=1}^{n}$ the training rows and $Q=\{(x_i,y_i)\}_{i=n+1}^{n+m}$ the query rows, obtained by splitting the simulated rows at random. Write $q_\theta(D)\in(0,1)^p$ for the output of the selection pass and $g_\theta(\cdot\mid x_{S},D_{S})$ for the predictive density returned by the prediction pass on the columns in $S$ (\Cref{sec:model}). We write $\Pi$ also for the induced law of $\tau$.

\paragraph{Training kernel.} The prediction pass is trained on subsets $\tilde S$ drawn from a kernel $K(\cdot\mid\gamma,q)$ that mixes three sources,
\[
K(\cdot\mid\gamma,q)=\rho_1\,\delta_{S(\gamma)}+\rho_2\,\mathcal C\big(\cdot\mid S(\gamma)\big)+(1-\rho_1-\rho_2)\prod_{j=1}^p\Bern\big(\operatorname{sg}(q_j)\big),
\]
where $\delta_{S(\gamma)}$ returns the true active set (teacher forcing); $\mathcal C(\cdot\mid S)$ drops each element of $S$ independently with probability $0.1$ and adds a $\mathrm{Poisson}(1)$ number of inactive variables chosen uniformly; and the last term samples a subset from the current selection head, with $\operatorname{sg}$ denoting that no gradient flows through the sampling (self-sampling). The corrupted and self-sampled subsets expose the prediction pass to the imperfect subsets it receives at inference.

\paragraph{Objective.} With per-variable and per-query averages, the training objective is
\begin{equation}\label{eq:objective}
L(\theta)=\E_{\tau\sim\Pi}\Big[\underbrace{\frac1p\sum_{j=1}^p\ell\big(q_{\theta,j}(D),\gamma_j\big)}_{L_{\rm sel}}+\alpha\,\E_{\tilde S\sim K(\cdot\mid\gamma,q_\theta(D))}\underbrace{\frac1m\sum_{i\in Q}-\log g_\theta\big(y_i\mid x_{i,\tilde S},D_{\tilde S}\big)}_{L_{\rm pred}}\Big],
\end{equation}
where $\alpha>0$ weights the prediction loss and $\ell(q,\gamma)=-\gamma\log q-(1-\gamma)\log(1-q)$ is the binary log loss. Both terms are prior-averaged log losses of the form \eqref{eq:pfn-loss}: the first for the labels $\gamma$, the second for the query responses.

\paragraph{Optimization.} The objective is minimized by stochastic gradient descent. At step $t$, a batch of $N_{\rm b}$ tasks $\tau_1,\dots,\tau_{N_{\rm b}}$ is drawn afresh from $\Pi$, one subset $\tilde S_b\sim K(\cdot\mid\gamma_b,q_\theta(D_b))$ is drawn per task, and $\theta$ is updated, with learning rate $\eta_t$, along the gradient of the Monte Carlo estimate
\[
\hat L_t(\theta)=\frac1{N_{\rm b}}\sum_{b=1}^{N_{\rm b}}\Big[L_{\rm sel}(\theta;\tau_b)+\alpha\,L_{\rm pred}(\theta;\tau_b,\tilde S_b)\Big],
\]
in which the sampled subsets are held fixed, as indicated by $\operatorname{sg}$ in the kernel.
Because every batch is newly simulated, no task is seen twice, and the network can only learn a rule that maps data sets to inclusion probabilities and predictions. Self-sampling is switched on after 10\% of training, when the selection head has become informative; before that the kernel uses its first two components only, renormalized. All tasks in a batch share $(n,p)$, and the maximal number of variables increases during training (from 50 to 200 to 500). \Cref{alg:training} summarizes the procedure, and the optimizer settings are listed in \Cref{app:implementation}. After training, $\theta$ is fixed and used unchanged for every analysis in this paper.

\begin{algorithm}[t]
\caption{SelectPFN training}\label{alg:training}
\begin{algorithmic}[1]
\Require meta-prior $\Pi$, kernel $K$, weight $\alpha$, number of steps $T$, batch size $N_{\rm b}$, learning rates $\eta_t$
\For{$t=1,\dots,T$}
  \State draw tasks $\tau_b=(n,p,\gamma_b,\phi_b,D_b,Q_b)\sim\Pi$, $b=1,\dots,N_{\rm b}$ \Comment{new tasks at every step}
  \State $q_b\gets q_\theta(D_b)$ \Comment{selection pass on all $p$ columns}
  \State draw $\tilde S_b\sim K(\cdot\mid\gamma_b,q_b)$
  \State compute $g_\theta(y_i\mid x_{i,\tilde S_b},D_{b,\tilde S_b})$ for $i\in Q_b$ \Comment{prediction pass on the columns in $\tilde S_b$}
  \State $\theta\gets\theta-\eta_t\nabla_\theta\hat L_t(\theta)$
\EndFor
\end{algorithmic}
\end{algorithm}

The prediction pass sees only the selected columns. The following elementary fact, which is used in \Cref{prop:minimizers}(b), shows that this restriction is exactly what the Bayesian model prescribes.

\begin{lemma}[Selection-restricted prediction]\label{prop:restricted}
Under \Cref{ass:design}, for every $\gamma$ with $S=S(\gamma)$, $p(y_*\mid x_*,\gamma,D)=p(y_*\mid x_{*,S},D_S,\gamma)$. Consequently $p(y_*\mid x_*,D)=\sum_\gamma p(\gamma\mid D)\,p(y_*\mid x_{*,S(\gamma)},D_{S(\gamma)},\gamma)$.
\end{lemma}
The proof is in \Cref{app:proofs-method}. The posterior over active sets depends on all columns, which is what makes selection non-trivial; given the active set, however, the predictive is exactly the Bayesian predictive computed from the active columns alone, and uncertainty about the active set is handled by averaging such predictions over the posterior of $\gamma$. SelectPFN implements this decomposition.

The next result shows that this objective targets the right quantities: at its population minimum, the two heads return the exact Bayesian inclusion probabilities and the exact subset-conditional predictive.

\begin{proposition}[Population minimizers]\label{prop:minimizers}
Suppose both heads can represent any measurable function of their inputs, and that $\Pi_f(\cdot\mid\gamma)$ depends on $\gamma$ only through $|S(\gamma)|$ and is exchangeable in the active variables, so that the subset-conditional predictive does not depend on the original column indices or on $p$. (a) The minimizer of $\E[L_{\rm sel}]$ is $q_j(D)=\pi_j(D)$. (b) Under teacher forcing, and assuming finite expected log loss, the minimizer of $\E[L_{\rm pred}]$ (over predictive densities; for the bar-distribution head, its bin probabilities) is $p(y_*\mid x_{*,S},D_S,S(\gamma)=S)$, which by \Cref{prop:restricted} equals $p(y_*\mid x_*,\gamma,D)$. (c) For a general kernel that does not depend on $\theta$, the minimizer is the Bayes predictive of $y_*$ given the information available to the prediction pass, $(x_{*,\tilde S},D_{\tilde S},\tilde S)$, and depends only on the columns in $\tilde S$.
\end{proposition}
The proof is in \Cref{app:proofs-method}.
The exchangeability condition holds whenever the prior treats the active variables symmetrically, as do spike-and-slab priors with a common slab, Gaussian processes with independently and identically distributed length-scales, and every family of our meta-prior, which assigns the active variables to their roles at random. The proposition implies that any gap between SelectPFN and the exact posterior comes from finite capacity and imperfect training, not from the objective. Because the selection head is supervised directly by $\gamma$, no sparsity penalty on network weights and no continuous relaxation of the gates \citep{jang2017gumbel,louizos2018l0} is needed, and both terms are proper, so $\alpha$ affects only how capacity is shared. \Cref{lem:decomp} in \Cref{sec:theory} quantifies the gap between the trained network and the Bayes optimum.

\subsection{Inference}\label{sec:inference}
\begin{algorithm}[t]
\caption{SelectPFN inference}\label{alg:inference}
\begin{algorithmic}[1]
\Require training data $D$, test covariates $x_*$, number of draws $B$
\State $q\gets q_\theta(D)$ \Comment{one forward pass; PIPs}
\For{$b=1,\dots,B$}
  \State $\gamma^{(b)}\sim\prod_j\Bern(q_j)$, $S^{(b)}\gets S(\gamma^{(b)})$
  \State $y_*^{(b)}\sim g_\theta(\cdot\mid x_{*,S^{(b)}},D_{S^{(b)}})$ \Comment{prediction pass on the selected columns only}
\EndFor
\State \Return $q$; pairs $\{(S^{(b)},y_*^{(b)})\}_{b=1}^B$; $\hat p(y_*\mid x_*,D)=B^{-1}\sum_bg_\theta(y_*\mid x_{*,S^{(b)}},D_{S^{(b)}})$
\end{algorithmic}
\end{algorithm}
After pretraining, a new data set needs no optimization (\Cref{alg:inference}). The outputs are the PIPs; paired draws, each attached to the subset that produced it, whose mixture approximates Bayesian model averaging (BMA); and a single reported model, the median probability model (MPM) $\hat S=\{j:q_j>1/2\}$ \citep{barbieri2004}. The $B$ subsets are drawn once per data set and shared by all test points, each prediction pass processes all test points at once, and it is cached per distinct subset; distinct subsets are few when the posterior is concentrated. The BMA samples subsets from the product of marginals; capturing posterior dependence between inclusions (``$x_3$ or $x_4$, but rarely both'') would require an autoregressive selection head (\Cref{sec:discussion}).

\subsection{Knockoff mode: exact finite-sample FDR control}\label{sec:knockoff}
PIPs are calibrated only to the extent that the meta-prior resembles the data. For a guarantee that does not depend on the prior, we combine SelectPFN with model-X knockoffs \citep{candes2018panning}. The FDR guarantee itself is the standard one; what is new is a construction that makes a \emph{pretrained} network a valid knockoff statistic without per-data-set training, in contrast to neural statistics such as DeepPINK \citep{lu2018deeppink}, which are fitted to each data set. Let $y=(y_1,\dots,y_n)^\top$ and let $\tilde X$ be a knockoff copy of $X$: $\tilde X\indep y\mid X$, and swapping any subset of columns between $X$ and $\tilde X$ leaves the law of $[X,\tilde X]$ unchanged.

\emph{Training.} Exact knockoffs are cheap for synthetic tasks, whose $P_X$ is known (Gaussian knockoffs for the latent Gaussian copula). The network receives the $2p$ columns $[X,\tilde X]$ in a random order, with responses, and the selection head is supervised with $\gamma_j$ for original columns and 0 for knockoff columns. No input marks which columns are knockoffs.

\emph{Inference.} Draw a uniformly random permutation of the $2p$ columns \emph{independently of the data}, apply the network, map the outputs back, and set $W_j=q_j-\tilde q_j$, where $\tilde q_j$ is the output for the knockoff copy of $x_j$; then apply the knockoff+ threshold at level $\alpha_{\FDR}$ \citep{barber2015knockoff}. \Cref{thm:knockoff} shows that the false discovery rate is at most $\alpha_{\FDR}$ in finite samples, for any network weights and whether or not the meta-prior is correct, provided the knockoffs are valid. Exact column equivariance in finite-precision arithmetic is not required: the random interleaving makes the flip-sign property hold in distribution. On real data, knockoffs are constructed approximately (second-order or deep knockoffs; \citealp{romano2020deepknockoffs}), and the guarantee holds to the extent that the approximation is accurate.

\subsection{Scaling to large $p$}\label{sec:blockwise}
Training cost grows with the number of cells, and feature attention is quadratic in $p$. We therefore pretrain at $p\le p_{\max}$ and, for larger $p$, (1) screen: randomly partition the variables into blocks of size at most $p_{\max}$, compute $q$ within each block, and average over $R$ random partitions; (2) refine: run one joint pass on the top-$k$ variables. Column equivariance means block assignment introduces no ordering artifacts. A variable that acts only through an interaction with a variable placed in another block can be missed in a single partition; repeated partitions mitigate this. This scheme is implemented but not yet validated beyond $p=500$ (\Cref{sec:discussion}).

\section{Theory}\label{sec:theory}

A trained PFN is a fixed function $q$ that maps a data set to a vector of inclusion probabilities. Training controls
only an average over the meta-prior of how far $q$ is from the exact posterior $\pi$. This section asks what that
average buys, and what it cannot buy. Four results answer the question.
\begin{itemize}
\item \Cref{sec:theory-regret}: the training loss controls the Bayes risk of the reported model.
\item \Cref{sec:theory-knockoff}: knockoff mode gives exact finite-sample FDR control whatever the network has learned.
\item \Cref{sec:theory-detect}: the meta-prior's weight on the right function class sets the sample size at which a
variable is found. This explains why the meta-prior must contain interaction-only structures, as the ablation in \Cref{sec:exp-e8} confirms.
\item \Cref{sec:theory-consistency}: frequentist selection consistency for a fixed truth requires a uniform
amortization assumption that the training loss cannot certify.
\end{itemize}
Every numbered result is proved analytically in \Cref{app:proofs}. A few published results are used as tools and are
marked \emph{cited} where they enter. Numerical computations appear only as illustrations, in paragraphs labelled as
such, and are never a step of a proof. Results that rest on assumptions we do not verify state those assumptions
explicitly.

\paragraph{Status of the assumptions.}
\begin{itemize}
\item \emph{Knockoff mode} (\Cref{thm:knockoff}) assumes only exact model-X knockoffs, the standard requirement of
\citet{candes2018panning}; it assumes nothing about the network, the meta-prior or training.
\item \emph{The regret bound} (\Cref{prop:regret}) needs no assumption for (i)--(ii); (iii) adds a standard margin
condition.
\item \emph{The detection theory} uses Gaussian noise with known variance and a fixed covariate law (\Cref{ass:regression}), the usual
setting of posterior-contraction theory; the likelihood factorizes, so the covariate law cancels from the posterior of the regression function (it enters only through
the definitions of the effect of a variable and of the norm); for a random covariate law this requires the regression function to be a priori independent of it. The conditions linking labels to functions hold by construction of the meta-prior, and the sieve condition is
verified for most of its families. The one high-level condition is \Cref{ass:c6}, a Bayes-factor bound on spurious
structures. It is used only for the lower bound in terms of the \emph{informative} weight: the upper bound
(\Cref{thm:csharp}(i)) does not use it, and with the trivial partition, which counts every structure that uses a variable as informative, \Cref{ass:c6} holds trivially, so the lower bound of \Cref{cor:prior-weight}(i) holds in terms of the total prior weight of structures that use the variable; the sharper constant of \Cref{thm:csharp}(ii) additionally needs a clipped-gap hypothesis for that class.
\item \emph{Consistency} composes Bayesian selection consistency (\Cref{ass:D1}), known for several priors, with uniform
amortization (\Cref{ass:uniform}). The latter is not a convenience: \Cref{prop:notransfer} shows that it cannot be derived from the training
loss, and identifying it is a result of this section.
\end{itemize}

\paragraph{Setting.}
A task drawn from the meta-prior $\Pi$ consists of $(P_X,f,\sigma,n,p)$, labels $\gamma\in\{0,1\}^p$ and a data set
$D=(X,y)$. Write $P_D$ for the induced law of $D$ (the prior predictive), $\pi_j(D)=\Pi(\gamma_j=1\mid D)$ for the exact
posterior inclusion probabilities, and $q(D)\in(0,1)^p$ for the output of the selection pass. The selection term of
\eqref{eq:objective}, averaged over variables, is
$L_{\rm sel}=\E_\Pi\big[p^{-1}\sum_j\ell(q_j(D),\gamma_j)\big]$ with
$\ell(q,\gamma)=-\gamma\log q-(1-\gamma)\log(1-q)$. Here $\Pi$ is the effective training distribution, which includes the curriculum
over $(n,p)$. Let $w(n,p)$ be its marginal probability of size $(n,p)$. Because $(n,p)$ is a function of $D$,
reweighting by $w$ leaves $\pi$ unchanged. We first relate this training loss to the exact posterior.

\begin{lemma}[Bayes floor plus amortization gap]\label{lem:decomp}
Let $\operatorname{Ent}$ be the Bernoulli entropy and $\operatorname{kl}$ the Bernoulli Kullback--Leibler divergence. Then
\[
L_{\rm sel}=\E\Big[\tfrac1p\textstyle\sum_j\operatorname{Ent}(\pi_j)\Big]+\Delta,\qquad
\Delta:=\E\Big[\tfrac1p\textstyle\sum_j\operatorname{kl}(\pi_j\,\|\,q_j)\Big]\ge0 .
\]
With $\Delta_{n,p}$ the same quantity conditional on size $(n,p)$, $\Delta=\sum_{n,p}w(n,p)\Delta_{n,p}\ge w(n,p)\Delta_{n,p}$.
\end{lemma}
The first term does not depend on the network, so $\Delta$ is the part of the selection loss that training can
reduce. The identity is the selection analogue of the basis of PFN training \citep{muller2022pfn} and of supervised
amortized structure learning \citep{lorch2022avici}. Because the training objective is joint, the network need not
minimize $L_{\rm sel}$ within its class. All results below are therefore stated in terms of the
achieved $\Delta$, whatever produced it. In particular, every reduction of the training loss is a reduction of the average Kullback--Leibler gap between the network and the exact posterior; the next subsection converts this gap into a guarantee for the selected model.

\subsection{The training loss controls the Bayes risk of the reported model}\label{sec:theory-regret}

Let $\MPM_q=\{j:q_j>\frac12\}$ and $\MPM_\pi=\{j:\pi_j>\frac12\}$ be the median probability models (MPMs) of $q$ and $\pi$. Under the per-variable Hamming loss
$L_H(S,\gamma)=p^{-1}\sum_j|\ind\{j\in S\}-\gamma_j|$, $\MPM_\pi$ is the Bayes rule
\citep[see also][for its predictive optimality]{barbieri2004}. Write $e_j=q_j-\pi_j$ and $m_j=\pi_j-\frac12$, let
$\mathcal E_j$ be the event that $q_j$ and $\pi_j$ lie on different sides of $\frac12$ (with the value $\frac12$ counted on the exclusion side), let $R(S)=\E_\Pi[L_H(S(D),\gamma)]$ be the Bayes risk
of a selection rule $S$, and let $\E_J$ average over $D\sim P_D$ and $J$ uniform on $\{1,\dots,p\}$. The following result bounds how much worse the network's median probability model can be than the exact one.

\begin{proposition}[Bayes regret of the amortized MPM]\label{prop:regret}
\begin{enumerate}
\item[(i)] $R(\MPM_q)-R(\MPM_\pi)=\E_J\big[|2\pi_J-1|\,\ind_{\mathcal E_J}\big]$, and $\mathcal E_j\subseteq\{|m_j|\le|e_j|\}$.
\item[(ii)] $R(\MPM_q)-R(\MPM_\pi)\le2\,\E_J\big[|e_J|\,\ind\{|m_J|\le|e_J|\}\big]\le2\,\E_J|e_J|\le\sqrt{2\Delta}$.
\item[(iii)] Suppose the margin condition $\Prob_J(|m_J|\le t)\le C_{\rm m}t^{\alpha_{\rm m}}$ holds for $t\in(0,\frac12]$ (this forces
$C_{\rm m}\ge2^{\alpha_{\rm m}}$), and let $\bar\delta^2=\E_Je_J^2\le\Delta/2$. Then the regret is at most
$4C_{\rm m}^{1/(2+\alpha_{\rm m})}\bar\delta^{2(1+\alpha_{\rm m})/(2+\alpha_{\rm m})}$.
\item[(iv)] $\Prob_J(\mathcal E_J)\le C_{\rm m}t^{\alpha_{\rm m}}+\bar\delta^2/t^2$ for every $t\in(0,\frac12]$.
\end{enumerate}
\end{proposition}

The margin condition in (iii) asks that few variables have a posterior inclusion probability close to $\frac12$; it is the analogue for selection of the margin condition of \citet{tsybakov2004aggregation} in classification. It holds with $\alpha_{\rm m}=1$ whenever the distribution of $\pi_J$ has a bounded density near $\frac12$, and it is mild in practice: on the exact-posterior tasks of \Cref{sec:exp-e1}, only 5\% of variables have $|\pi_j-\frac12|<0.1$. Each step of the proof is classical: plug-in classification, log-loss calibration and margin conditions. What the proposition
adds is the interpretation of the exact-posterior comparison in \Cref{sec:exp-e1}. A disagreement between $\MPM_q$ and $\MPM_\pi$ can occur
only where the amortization error exceeds the distance of $\pi_j$ from $\frac12$. Its cost, $|2\pi_j-1|$, is then at
most $2|e_j|$. Disagreements at near-ties are therefore cheap, and an exact-set-match statistic overstates their
importance. On those tasks, the measured regret ($0.015$ per variable for the longer-trained model) is well below
the refined bound (ii) ($0.031$). Near-ties make up 5\% of variables but 44\% of disagreements
(\Cref{app:proofs-regret}; these figures are a numerical illustration, not part of the proof). A cost-weighted
version, with false-inclusion cost $c$, replaces the factor 2 in (i)--(ii) by $1+c$. We claim (iii)--(iv) only for
$c=1$. In short, a network that is close to the exact posterior on average reports, on average, a model that is nearly as good as the exact Bayes rule, and the guarantee improves as training reduces $\Delta$.

\subsection{Exact finite-sample FDR control in knockoff mode}\label{sec:theory-knockoff}

\citet[Section 5.2.2]{candes2018panning} already observe that the difference of Bayesian inclusion probabilities
computed on $[X,\tilde X]$ is a valid knockoff statistic when the computation is exactly equivariant to swapping
columns. A trained network is not exactly equivariant in floating point. On our knockoff model, relying on
equivariance alone produced sign-flip deviations of order $7\times10^{-3}$ (\Cref{app:proofs-knockoff}). We
therefore randomize the column order and prove that this restores exact validity for \emph{any} map $F$, whatever
it has learned.

The only assumption is that exact model-X knockoffs are available. This is the standard requirement of the knockoff framework \citep{candes2018panning}: it holds when the covariate distribution is known, for example for Gaussian covariates, as in our training tasks and in the experiments of \Cref{sec:exp-e5}, and it is met approximately when the covariate distribution is estimated. Rows are i.i.d. Let $\mathcal H_0=\{j:Y\indep X_j\mid X_{-j}\}$ be the null set, where $(X,Y)$ denotes a generic observation, let $\tilde X$ be an exact model-X
knockoff copy of the design matrix $X$, and let $Z=[X,\tilde X]$. The procedure in \Cref{sec:knockoff} has four steps.
\begin{enumerate}
\item Draw a permutation $\Sigma$ of $\{1,\dots,2p\}$ uniformly at random, independently of $(Z,y)$.
\item Apply the selection pass $F(\cdot,\cdot,U)$ to the permuted table $(Z_{\cdot,\Sigma},y)$. Here $U$ is any
internal randomness, independent of $(Z,y,\Sigma)$; a deterministic $F$ is the case of constant $U$.
\item Undo the permutation to obtain scores $Q_1,\dots,Q_{2p}$.
\item Set $W_j=Q_j-Q_{j+p}$ and apply the knockoff+ threshold \citep{barber2015knockoff}.
\end{enumerate}

The procedure controls the FDR whatever the network computes, as the following theorem shows.

\begin{theorem}[Exact FDR control for any network]\label{thm:knockoff}
For any $F$ as above:
\begin{enumerate}
\item[(i)] for every $S\subseteq\mathcal H_0$, $(W,y)\overset d=(\mathrm{flip}_SW,y)$, where $\mathrm{flip}_S$ changes the
sign of the coordinates in $S$;
\item[(ii)] conditionally on $(|W_j|)_{j\le p}$ and on $(W_j)_{j\notin\mathcal H_0}$, the signs of the nonzero null
statistics are i.i.d.\ uniform on $\{\pm1\}$;
\item[(iii)] knockoff+ at level $\alpha_{\FDR}$ has $\FDR\le\alpha_{\FDR}$ in finite samples, and the knockoff filter satisfies
$\E[V/(R_{\rm rej}+1/\alpha_{\FDR})]\le\alpha_{\FDR}$, where $V$ and $R_{\rm rej}$ are the numbers of false and total rejections.
\end{enumerate}
\end{theorem}
The theorem means that a user can apply a pretrained SelectPFN off the shelf and control the false discovery rate at any chosen level without trusting the network, the meta-prior or the training: these affect only how many true variables are found.

The proof rests on a pointwise identity (\Cref{lem:swap}). Swapping a null column with its knockoff, while
composing the random permutation with the same swap, feeds $F$ the identical tensor. The exchangeability lemma of
\citet{candes2018panning} and the supermartingale argument of \citet{barber2015knockoff} are cited; the
interleaving identity is ours. Validity uses no property of $F$, of the meta-prior, or of training; power depends
on all three. Two requirements matter in practice.
\begin{itemize}
\item The permutation must be drawn independently of the data and afresh for each analysis. Reusing a
data-dependent permutation voids the guarantee.
\item The knockoffs must be exact. With estimated second-order knockoffs, control is approximate.
\end{itemize}
Averaging $Q$ over several independent interleavings preserves (i)--(iii) and reduces the extra randomness
(\Cref{app:proofs-knockoff}).

\subsection{Which variables are found: prior weight sets the detection threshold}\label{sec:theory-detect}

We now ask, for a fixed truth $f_0$, how much data the exact posterior $\pi$ needs before it includes a relevant
variable, and how the meta-prior affects that amount. We work in an idealized regression setting.
\begin{assumption}[Gaussian regression with a fixed covariate law]\label{ass:regression}
\begin{enumerate}
\item[(a)] $Y=f(X)+\varepsilon$ with $\varepsilon\sim N(0,\sigma^2)$ and $\sigma$ known.
\item[(b)] $P_X$ is known and common to all $f$.
\end{enumerate}
\end{assumption}
These are the standard conditions of posterior-contraction theory for nonparametric regression \citep{ghosal2000,vdvvz2008rkhs}. \Cref{ass:regression}(a) covers any regression with Gaussian errors of known variance; heavy-tailed noise and an unknown $\sigma$ are not covered here. \Cref{ass:regression}(b) is mild, because the covariates are ancillary: their distribution cancels from the posterior of $f$, and it enters only through the definitions of $N_j$ and of the norm below.
Let $N_j=\{f:f(x)=h(x_{-j})\ P_X\text{-a.s.}\}$ be the functions that ignore $x_j$, and $\|\cdot\|$ the $L_2(P_X)$ norm.
The \emph{conditional effect} of $x_j$ and the \emph{information index} are
\[
\kappa_j^2=\inf_{h\in N_j}\|f_0-h\|^2=\E\,\Var(f_0(X)\mid X_{-j}),\qquad t_j=n\kappa_j^2/\sigma^2 .
\]
The infimum is attained at $h_j^*=\E[f_0\mid X_{-j}]$. Note that $\kappa_j>0$ is possible when
$\mathrm{Cov}(X_j,Y)=0$, as for $\sin(x_3x_4)$; marginal screening then fails \citep{fan2008sis}.

\emph{Structured priors.} In this subsection $\Pi$ denotes the induced prior on $f$. Let $\Pi=\sum_{k\in\mathcal K}w_k\Pi_k$ be a mixture over \emph{structures}, where $\Pi_k$
is carried by functions that use exactly the variables $V_k$ (in the $P_X$-a.s.\ sense: $f\in L(V_k)$, and $f\notin N_j$ for every $j\in V_k$, where $L(V)$ is the set of functions of $x_V$ alone), and $\gamma_j=\ind\{j\in V_k\}$. Write
$\mathcal K_j=\{k:j\in V_k\}$ for the structures that use $x_j$. Partition them into \emph{informative} structures
$\mathcal K_j^I$, which can represent how $f_0$ uses $x_j$, and \emph{spurious} ones $\mathcal K_j^S$. Let
$\omega_j^I=\sum_{k\in\mathcal K_j^I}w_k$ be the informative weight. Write $D_n$ for $D$ when $n$ varies, $P_f^n$ for the law of $D_n$ under $f$, and $P_0^n=P_{f_0}^n$. Write $\Lambda_f=\prod_ip_f(X_i,Y_i)/p_{f_0}(X_i,Y_i)$, where $p_f$ is the joint density of $(X,Y)$ under \Cref{ass:regression} with regression function $f$, and
$\Lambda(k)=\int\Lambda_f\,d\Pi_k$. The lower bounds below must rule out that $x_j$ is included through structures that use it in the wrong way; we state this as an assumption.

\begin{assumption}[Spurious-inclusion control]\label{ass:c6}
There are $\tau_n\ge0$ and $\delta_n\in[0,1)$ such that, with $P_0^n$-probability at least $1-\delta_n$,
\[
\textstyle\sum_{k\in\mathcal K_j^S}w_k\Lambda(k)\le\tau_n\int_{N_j}\Lambda_f\,d\Pi .
\]
\end{assumption}
\Cref{ass:c6} bounds the posterior odds of ``$x_j$ included, but used the wrong way'' against ``$x_j$ excluded''. It is a
nested-model Bayes-factor statement under misspecification. We do not prove it; \citet{jiang2021gpselection}
handle the correctly specified Gaussian-process case. Because \Cref{ass:c6} bounds a Bayes factor by a constant, it can hold
only for zero-gain spurious structures, that is, those whose support fits $f_0$ no better than $N_j$ does. Any
structure that fits better must be counted as informative. Two examples satisfy it. The trivial partition $\mathcal K_j^S=\emptyset$
satisfies \Cref{ass:c6} with $\tau_n=\delta_n=0$, at the price of counting every structure that uses $x_j$ as informative. More interestingly, when $f_0$ uses $x_j$ only through an odd interaction and the covariates are independent, an additive term in $x_j$ fits $f_0$ no better than a function that ignores $x_j$; \Cref{ass:c6} then says that the Bayesian Occam factor penalizes such a useless component, which is the behaviour expected of nested Bayesian model comparison, although we do not prove it in this generality. We can now state when the exact posterior detects a variable.

\begin{proposition}[Information floor and structural-weight converse]\label{prop:detect}
\begin{enumerate}
\item[(a)] (Prior-free; uses \Cref{ass:regression}, $\E f_0(X)^2<\infty$ and $\kappa_j>0$.) For any selection rule $\hat S$,
$P_0^n(j\in\hat S)\le P_{h_j^*}^n(j\in\hat S)+\frac12\sqrt{t_j}$.
\item[(b)] (Uses \Cref{ass:regression} and~\Cref{ass:c6} and the structured setup.) Fix $\eta>0$ and $m_4<\infty$, and let
$B_j=\{f\in N_j:\|f-h_j^*\|\le\eta,\ \E(f_0-f)^4(X)\le m_4\}$ with $\Pi(B_j)>0$. Let
$\bar k_j=(\kappa_j+\eta)^2/(2\sigma^2)$ and $\bar v_j=(\kappa_j+\eta)^2/\sigma^2+m_4/(4\sigma^4)$. For every $c_{\rm M}>1$ and
$s>0$, with $P_0^n$-probability at least $1-1/c_{\rm M}-1/s^2-\delta_n$,
\[
\frac{\pi_j(D_n)}{1-\pi_j(D_n)}\le\tau_n+\frac{c_{\rm M}\omega_j^I}{\Pi(B_j)}\,\exp\big(n\bar k_j+s\sqrt{n\bar v_j}\big).
\]
\end{enumerate}
\end{proposition}
Part (a) says that no procedure, Bayesian or not, can detect $x_j$ reliably unless $t_j$ is of order one. Part (b)
says that the exact posterior misses $x_j$ while $\frac{t_j}2(1+\eta/\kappa_j)^2\lesssim\log(1/\omega_j^I)$: the data must
overcome the prior odds against the informative structures. With $\omega_j^I=0$ and $\tau_n<1$, $x_j$ is missed for every
$n$ with probability at least $1-\delta_n$. Two further non-asymptotic statements are given in \Cref{app:proofs-detect}: a sparsity-prior ceiling (the
trivial-partition case of (b)), and a detection bound based on Le~Cam--Birg\'e tests (cited) for sup-norm-bounded
priors. Together, (a) and (b) identify two separate barriers to detection: the information in the data, which limits every method, and the prior weight of the right kind of structure, which no amount of signal can overcome when it is zero.

Part (b) suggests that the information needed for detection grows with $\log(1/\omega_j^I)$. To make this precise, we vary the informative weight while keeping everything else fixed.

To isolate the role of $\omega_j^I$, fix a reference structured prior with informative weight $\omega_0$. For $\omega\le\omega_0$, let
$\Pi_\omega$ rescale the informative structures to total weight $\omega$ and all others to $1-\omega$, keeping proportions within
each group. Vary the information index only through $n$, and let
\[
t_{50}(\omega)=\inf\big\{t=n\kappa_j^2/\sigma^2:\ P_0^n\big(\pi_j(D_n)>\tfrac12\big)\ge\tfrac12\big\}
\]
be the smallest information index at which $x_j$ is detected at least half of the time under $\Pi_\omega$. Let
$\psi^2=\inf\{\|f_0-f\|^2:f\in\operatorname{supp}\Pi^I\}$ be the \emph{informative gap}, where $\Pi^I$ is the
normalized informative part of the prior and the support is taken in $L_2(P_X)$. Let $\lambda_j=1-\psi^2/\kappa_j^2$ be the
\emph{efficiency} of the informative structures.

\begin{corollary}[Linear dependence on the log prior weight]\label{cor:prior-weight}
\begin{enumerate}
\item[(i)] If $\E f_0(X)^4<\infty$, the moment condition (M4) holds (prior draws have finite fourth moments under $P_X$: $\E_Xf(X)^4<\infty$ for $\Pi_{\omega_0}$-almost every $f$), \Cref{ass:c6} holds under $\Pi_{\omega_0}$ with $\sup_n\tau_n<1$ and $\sup_n\delta_n<\frac12$, and $h_j^*$ lies in the
$L_2(P_X)$-support of $\Pi_{\omega_0}$ restricted to $N_j$, then $\liminf_{\omega\to0}t_{50}(\omega)/\log(1/\omega)\ge2$.
\item[(ii)] Let $b>0$, $u=b^2/(2\sigma^2)$, $c_b=(1-e^{-u})/(4u)$, and let $c_{\rm T}=9/32$ be the constant of a cited test construction (\Cref{lem:tests}). Assume, under $\Pi_{\omega_0}$: (U1) all prior draws lie in one class $\mathcal F$ with $\|f-f_0\|_\infty\le b$; (U2) the set $\{f\in\mathcal F\cap N_j:\|f-f_0\|\ge\kappa_j\}$ is covered by finitely many Hellinger balls of radius $\sqrt{c_b}\,\kappa_j/(4\sigma)$ centred at densities $p_f$ with $f$ in that set; (U3) $f_0$ is approximable by the informative structures: $\Pi^I$ gives positive mass to $\{f:\KL(p_{f_0}\|p_f)\le\bar\varepsilon^2,\ \Var_{f_0}\log(p_f/p_{f_0})\le\bar\varepsilon^2\}$ for every $\bar\varepsilon>0$. If moreover $\gamma_j=0$ implies $f\in N_j$, $\Pi$-a.s.\ (\Cref{ass:labels}(a)),
then $\limsup_{\omega\to0}t_{50}(\omega)/\log(1/\omega)\le1/(c_{\rm T}c_b)$, and $1/(c_{\rm T}c_b)\to128/9$ as $u\to0$.
\end{enumerate}
\end{corollary}

Taken together, (i) and (ii) show that the detection threshold grows linearly in $\log(1/\omega)$: for small $\omega$, each tenfold reduction of the informative prior weight raises the information needed for detection by roughly a constant amount. Apart from \Cref{ass:c6}, the conditions of (i) hold for Gaussian-process priors with bounded kernels and for the parametric families of our meta-prior with Gaussian weights, whenever the prior restricted to $N_j$ can approximate $h_j^*$. The constants in (i) and (ii) differ by a factor of about seven, and (ii) excludes Gaussian-process priors. The next
theorem removes both limitations. It replaces the tests by a uniform upper bound on the integrated likelihood
(\Cref{lem:ULLN}). That bound is in the spirit of \citet{shalizi2009dynamics} and of the Bayes-factor asymptotics of
\citet{chatterjee2018bayesfactor}, but is proved directly. The device that removes every sup-norm assumption is to
clip both $y$ and $f$ at a level $c$, which can only shrink each residual. The theorem uses, besides the moment condition (M4) of \Cref{cor:prior-weight}(i), two conditions on a set $A$ of functions and a radius $r\ge0$, stated
in full in \Cref{app:proofs-detect}: a clipped population gap (U$^\sharp$1), that after clipping $y$ and $f$ at a large level and restricting $x$ to a large cube, every $f\in A$ has expected squared residual at least $\sigma^2+r^2$ up to an arbitrarily small tolerance; and a sieve condition (U$^\sharp$2), that the prior restricted to $A$ puts all but mass $e^{-\beta n}$, for some $\beta>r^2/\sigma^2$, on sets whose clipped versions have sup-norm covering numbers on cubes of order $e^{o(n)}$. By \Cref{lem:sieve}, (U$^\sharp$2) holds for finite mixtures of Gaussian
processes with, for example, squared-exponential kernels, and of parametric families such as neural networks with Gaussian weights, linear models and random-feature expansions, provided kernels and architectures range over finitely many values. Families whose draws are discontinuous in their parameters (trees with random split points) and Gaussian processes with continuously distributed length-scales are not covered.

\begin{theorem}[Sharp detection threshold]\label{thm:csharp}
Assume \Cref{ass:regression}, $\E f_0(X)^4<\infty$ and (M4).
\begin{enumerate}
\item[(i)] (Upper half; no \Cref{ass:c6}, no sup-norm bound.) If $\psi<\kappa_j$ and the prior restricted to $N_j$ satisfies
the sieve condition (U$^\sharp$2) with $A=N_j$ and $r=\kappa_j$, then $\limsup_{\omega\to0}t_{50}(\omega)/\log(1/\omega)\le2/\lambda_j$.
\item[(ii)] (Lower half.) Assume \Cref{ass:c6} as in \Cref{cor:prior-weight}(i) and the support condition on $h_j^*$. If $\psi=0$, then
$\liminf_{\omega\to0}t_{50}(\omega)/\log(1/\omega)\ge2$. If $\psi>0$, assume in addition the \emph{clipped-gap hypothesis}: the
informative part is carried by a set $A$ satisfying (U$^\sharp$1)--(U$^\sharp$2) with $r=\psi$. Then
$\liminf_{\omega\to0}t_{50}(\omega)/\log(1/\omega)\ge2/\lambda_j$ (read $2/\lambda_j=\infty$ if $\psi\ge\kappa_j$).
\end{enumerate}
Consequently, under (i) and (ii), $\lim_{\omega\to0}t_{50}(\omega)/\log(1/\omega)=2/\lambda_j=2\kappa_j^2/(\kappa_j^2-\psi^2)$.
\end{theorem}
The theorem identifies the constant. When the informative structures can represent how $f_0$ uses $x_j$ ($\psi=0$), detection requires an information index of about $2\log(1/\omega)$ for small $\omega$; when they can only approximate it ($\psi>0$), the requirement is inflated by the factor $1/\lambda_j$, and a variable whose effect they cannot capture better than by ignoring it ($\psi\ge\kappa_j$) is not detected once $\omega$ is small. The moment condition (M4) holds for Gaussian processes with bounded kernels, for bounded function classes, and for neural networks with Lipschitz activations when $\E|X|^4<\infty$.

\paragraph{The $\log p$ scale as a special case.} The prior-weight law contains the familiar dimension dependence of high-dimensional selection. The informative structures for $x_j$ all have $\gamma_j=1$, so $\omega_j^I\le\Pi(\gamma_j=1)$, which equals the mean sparsity $\bar w$ under the Beta--Bernoulli prior of \Cref{sec:metaprior}. If the prior is calibrated so that $\bar w\approx|S_0|/p$, \Cref{prop:detect}(b) shows that the exact posterior misses $x_j$ unless, up to the lower-order terms there, $t_j=n\kappa_j^2/\sigma^2\gtrsim2\log(p/|S_0|)$. For a linear truth, $\kappa_j^2=\beta_j^2\,\E\Var(X_j\mid X_{-j})$, so this is the beta-min condition $n\beta_{\min}^2\gtrsim\sigma^2\log p$ of sparse linear regression, recovered for nonlinear effects with $\kappa_j$ in place of $|\beta_j|$ and with an explicit constant.

The clipped-gap hypothesis requires (U$^\sharp$1), which holds when the informative structures are saturated, meaning that some minimizing
structure contains $\E[f_0\mid X_{V_k}]$ in its support (\Cref{lem:U0}), and (U$^\sharp$2) on the informative part, which holds for finite mixtures of Gaussian processes and of smooth parametric families (\Cref{lem:sieve}). It need not hold for restricted informative
classes, and \Cref{app:proofs-detect} gives a counterexample in which the $L_2$ gap is not the quantity that governs
the evidence. The theorem is asymptotic in $\omega$ and gives no explicit rate in $n$; the test-based
\Cref{cor:prior-weight}(ii) remains as a non-asymptotic statement with explicit constants. For an exactly enumerable
Gaussian-process toy prior with the truth $\sin(x_1x_2)+0.8x_3$, the theorem gives $\limsup t_{50}/\log(1/\omega)\le2$
without \Cref{ass:c6}, and limit~2 under \Cref{ass:c6} (\Cref{cor:toy}).

\paragraph{What this predicts for SelectPFN.}
The theorem concerns the exact posterior, and the network approximates it. It nevertheless gives a sharp
qualitative prediction. A variable that acts only through a function class carrying little meta-prior weight
needs an information index of about $2\log(1/\omega_j^I)/\lambda_j$ before it is detected, however strong its effect. When
that class has no weight, the variable is never detected.
A meta-prior in which every interacting variable also has a main effect gives odd interactions such as $\sin(x_3x_4)$, whose variables have no marginal effect, almost no informative weight; the interaction-only family of \Cref{sec:metaprior} gives them substantial weight. The ablation in \Cref{sec:exp-e8} confirms the prediction: without the interaction-only family, $x_3,x_4$ in design~C are included in 0.06 of the data sets at $n=800$ (averaged over the two variables), against 0.36 for SelectPFN and 0.52 for SelectPFN-L, and in a controlled study in which only the weight of the family varies, detection rises monotonically with the weight. For the exact posterior on the toy prior, $t_{50}$ grows
linearly in $\log(1/\omega)$, with the ratio $t_{50}/\log(1/\omega)$ falling from 7.4 to 4.0 and staying above 2
(\Cref{app:proofs-detect}; a numerical illustration).

\subsection{What amortization can and cannot guarantee for a fixed truth}\label{sec:theory-consistency}

Fix a truth $P_0$ with active set $S_0$ and indicator $\gamma^0$, and let $n\to\infty$, with $p=p_n$ allowed to grow. Consistency of the amortized selector needs two assumptions, one about the exact posterior and one about the trained network.

\begin{assumption}[Bayesian selection consistency with margin]\label{ass:D1}
$P_0^n\big(\max_j|\pi_j(D_n)-\gamma^0_j|\le\eta_n\big)\to1$ for some $\eta_n\to0$.
\end{assumption}
\begin{assumption}[Uniform amortization, sign version]\label{ass:uniform}
With $P_0^n$-probability tending to one, $\operatorname{sign}(q_j-\frac12)=\operatorname{sign}(\pi_j-\frac12)$ for
all $j$.
\end{assumption}
\Cref{ass:D1} is known for specific priors: linear spike-and-slab models under beta-min conditions, Gaussian-process
regression \citep{jiang2021gpselection}, and sparse deep networks \citep{sun2022consistent}. For a mixture meta-prior
it must still be verified. \Cref{ass:uniform} concerns the trained network at the truth. It holds, for example, whenever the network's inclusion probabilities stay within a fixed distance $\epsilon_0<\frac12$ of the exact ones on data sets that are typical under $P_0$, since under \Cref{ass:D1} the exact probabilities are then far from $\frac12$. Under these two assumptions, consistency of the amortized selector is immediate.

\begin{proposition}[Frequentist consistency of the amortized MPM]\label{thm:consistency}
Under \Cref{ass:D1} and \Cref{ass:uniform}, $P_0^n(\MPM_q=S_0)\to1$.
\end{proposition}
No growth condition on $p_n$ is imposed beyond \Cref{ass:D1}: the admissible rate is that of the exact-posterior result at hand, for linear spike-and-slab priors typically $\log p_n=o(n)$ together with a beta-min condition, the scale identified after \Cref{thm:csharp}. \Cref{ass:uniform} concerns the trained network, which is pretrained for $p\le500$; larger $p$ is handled by the block-wise procedure of \Cref{sec:blockwise}, which the proposition does not cover.

The proposition shows that an amortized selector inherits the consistency of the Bayesian procedure it approximates, provided the approximation is accurate at the truth. The composition is elementary; the content lies in \Cref{ass:uniform}. One might hope that a small training loss implies \Cref{ass:uniform}. The
next result shows that it does not. The reason is that training certifies $q$ on data sets that are typical under
the prior predictive $P_D$, and these can carry little $P_D$-mass near a fixed truth. The construction uses the sieve-and-test condition (\Cref{ass:D3}) of \Cref{app:proofs-consistency}: there are sieves $\mathcal F_n$ on which $\|f-f_0\|_\infty\le b$ for some $b>0$, and radii $r_n\to0$ such that the tests of \Cref{lem:tests} separate $f_0$ from $\{f\in\mathcal F_n:\|f-f_0\|\ge r_n\}$ with vanishing error. It is the standard condition of posterior-contraction theory \citep{ghosal2000}, which holds for smooth parametric families and for bounded nonparametric classes with polynomial entropy (\Cref{lem:rates}).

\begin{proposition}[Average-case training does not transfer to a fixed truth]\label{prop:notransfer}
Work in the regression model of \Cref{ass:regression} with a joint prior on $(f,\gamma)$, at fixed size $(n,p)$.
\begin{enumerate}
\item[(1)] (Construction.) Assume \Cref{ass:D1}, and the sieve-and-test condition (\Cref{ass:D3}) of \Cref{app:proofs-consistency}. Assume $\pi_j\in(0,1)$. Then
there are events $A_n$ with $P_0^n(A_n)\to1$, and maps $q^{(n)}$ (which may depend on $f_0$) such that
$\MPM_{q^{(n)}}\ne\MPM_\pi$ with $P_0^n$-probability tending to one, while
$\Delta(q^{(n)})\le w(n,p)\,\mu_n\log(1/\varpi_n)$. Here $\varpi_n\in(0,\frac12)$ is arbitrary, $\mathcal F_n$ and $r_n$ are the sieves and radii of \Cref{ass:D3}, $c_{\rm T}=9/32$ and $c_b=(1-e^{-u})/(4u)$ with $u=b^2/(2\sigma^2)$, $b$ the sup-norm bound of \Cref{ass:D3}, and
$\mu_n=\Pi(\|f-f_0\|<r_n)+\Pi(\mathcal F_n^c)+2e^{-c_{\rm T}c_bnr_n^2/\sigma^2}$. In a parametric model of dimension $d_\theta$ satisfying the conditions of \Cref{lem:rates}(b),
$\Delta=O(n^{-d_\theta/2+\epsilon})$ for every $\epsilon>0$ is compatible with failure of \Cref{ass:uniform}. When the prior mass of $r_n$-balls around $f_0$ decays faster than every power of $n$
(\Cref{lem:rates}(a)), so does $\mu_n$.
\item[(2)] (Lower bound.) If $E_n$ is an event on which $\max_j|\pi_j-\gamma^0_j|\le\eta_n<\frac12$ and
$\MPM_q\ne\MPM_\pi$, then for every $\bar\varepsilon_n>0$, every set $B_n$ of functions $f$ with $\KL(p_{f_0}\|p_f)\le\bar\varepsilon_n^2$ and $\Var_{f_0}\log(p_f/p_{f_0})\le\bar\varepsilon_n^2$ (as in \Cref{lem:evidence}), and every $C>0$,
\[
\Delta(q)\ge w(n,p)\,\frac{2(\frac12-\eta_n)^2}{p}\,\Pi(B_n)\,e^{-(1+C)n\bar\varepsilon_n^2}
\Big(P_0^n(E_n)-\frac1{C^2n\bar\varepsilon_n^2}\Big).
\]
\end{enumerate}
\end{proposition}

Part (2) shows that a training error below the prior-mass level $\Pi(B_n)e^{-(1+C)n\bar\varepsilon_n^2}$ does rule out
sign flips on $P_0$-typical data. Part (1) shows that errors only slightly above $\mu_n$ do not. \Cref{ass:uniform} must therefore come
from structure (the architecture and its approximation properties), not from the loss alone. This is the
selection analogue of the observation that the prior is not washed out in PFNs \citep{melnychuk2026pfn}.
A constructive route to \Cref{ass:uniform} for linear spike-and-slab priors at fixed $p$ is sketched in \Cref{app:proofs-consistency}.
We state it as a conjecture. It needs the network to know $n$, which SelectPFN's selection readout receives
through its global token.

\paragraph{Related theory.}
\citet{nagler2023pfn} shows that the bias of a PFN need not vanish without locality. \citet{melnychuk2026pfn} show
that PFN treatment-effect posteriors can be frequentist-inconsistent. \citet{frazier2024neural} study the
concentration of neural posterior estimators, and \citet{wakayama2025icl} give non-asymptotic Bayes-gap bounds for
in-context learning. In-context learning theory covers the Lasso \citep{bai2023transformers}, minimax nonparametric
regression \citep{kim2024transformers}, and single-index targets with low-dimensional structure \citep{oko2024pretrained}.
\citet{maceda2024vanbayes} compare neural inclusion probabilities with MCMC. LinearPFN \citep{schiekiera2026linearpfn}
trains a PFN on a linear spike-and-slab prior with interactions and gives no theory. To our knowledge, none of these
works gives selection guarantees for an amortized posterior, a detection threshold in terms of prior structural
weight, or finite-sample FDR control for a floating-point network.

\providecommand{\EoneMAE}{}\renewcommand{\EoneMAE}{0.037}
\providecommand{\EoneMPM}{}\renewcommand{\EoneMPM}{0.820}
\providecommand{\EoneMPMlo}{}\renewcommand{\EoneMPMlo}{0.803}
\providecommand{\EoneMPMhi}{}\renewcommand{\EoneMPMhi}{0.837}
\providecommand{\EoneNT}{}\renewcommand{\EoneNT}{0.916}
\providecommand{\EoneNTlo}{}\renewcommand{\EoneNTlo}{0.903}
\providecommand{\EoneNThi}{}\renewcommand{\EoneNThi}{0.928}
\providecommand{\EoneRef}{}\renewcommand{\EoneRef}{0.86}
\providecommand{\EoneNearTieShare}{}\renewcommand{\EoneNearTieShare}{44}
\providecommand{\EoneFracNoTie}{}\renewcommand{\EoneFracNoTie}{84}
\providecommand{\EoneAgreeSzero}{}\renewcommand{\EoneAgreeSzero}{0.98}
\providecommand{\EoneAgreeSseven}{}\renewcommand{\EoneAgreeSseven}{0.49}
\providecommand{\EoneLMAE}{}\renewcommand{\EoneLMAE}{0.028}
\providecommand{\EoneLMPM}{}\renewcommand{\EoneLMPM}{0.846}
\providecommand{\EoneLNT}{}\renewcommand{\EoneLNT}{0.942}
\providecommand{\EoneLNLL}{}\renewcommand{\EoneLNLL}{0.016}
\providecommand{\EoneRegret}{}\renewcommand{\EoneRegret}{0.0154}
\providecommand{\EoneRegretBound}{}\renewcommand{\EoneRegretBound}{0.031}
\providecommand{\EoneDisagree}{}\renewcommand{\EoneDisagree}{4.2}
\providecommand{\EoneLRegret}{}\renewcommand{\EoneLRegret}{0.0095}
\providecommand{\EoneLRegretBound}{}\renewcommand{\EoneLRegretBound}{0.021}
\providecommand{\ExactDevSmall}{}\renewcommand{\ExactDevSmall}{0.08}
\providecommand{\ExactDevLarge}{}\renewcommand{\ExactDevLarge}{0.47}
\providecommand{\EtwoNstrata}{}\renewcommand{\EtwoNstrata}{26}
\providecommand{\EtwoNobs}{}\renewcommand{\EtwoNobs}{10}
\providecommand{\EtwoFailing}{}\renewcommand{\EtwoFailing}{$|S|=0$, $|S|=1-3$, $|S|\ge 10$, SNR=null, family=null}
\providecommand{\EthreeAUPRCOne}{}\renewcommand{\EthreeAUPRCOne}{0.728}
\providecommand{\EthreeAUPRCTwos}{}\renewcommand{\EthreeAUPRCTwos}{0.754}
\providecommand{\EthreeAUPRCTwol}{}\renewcommand{\EthreeAUPRCTwol}{0.771}
\providecommand{\EthreeAUPRCDart}{}\renewcommand{\EthreeAUPRCDart}{0.499}
\providecommand{\EthreeAUPRCLasso}{}\renewcommand{\EthreeAUPRCLasso}{0.512}
\providecommand{\EthreeFDRTwol}{}\renewcommand{\EthreeFDRTwol}{0.070}
\providecommand{\EthreeFDRDart}{}\renewcommand{\EthreeFDRDart}{0.632}
\providecommand{\EthreePriorGain}{}\renewcommand{\EthreePriorGain}{+0.026}
\providecommand{\EthreePriorGainlo}{}\renewcommand{\EthreePriorGainlo}{0.024}
\providecommand{\EthreePriorGainhi}{}\renewcommand{\EthreePriorGainhi}{0.027}
\providecommand{\EthreeSizeGain}{}\renewcommand{\EthreeSizeGain}{+0.017}
\providecommand{\EthreeSizeGainlo}{}\renewcommand{\EthreeSizeGainlo}{0.016}
\providecommand{\EthreeSizeGainhi}{}\renewcommand{\EthreeSizeGainhi}{0.018}
\providecommand{\EthreeMinGap}{}\renewcommand{\EthreeMinGap}{0.16}
\providecommand{\EthreeMaxGap}{}\renewcommand{\EthreeMaxGap}{0.41}
\providecommand{\EthreeCorrFrac}{}\renewcommand{\EthreeCorrFrac}{100}
\providecommand{\EthreeWorstIndepGap}{}\renewcommand{\EthreeWorstIndepGap}{+0.004}
\providecommand{\RuntimeSPTwos}{}\renewcommand{\RuntimeSPTwos}{0.037}
\providecommand{\RuntimeSpeedTwos}{}\renewcommand{\RuntimeSpeedTwos}{416}
\providecommand{\RuntimeSPTwol}{}\renewcommand{\RuntimeSPTwol}{0.042}
\providecommand{\RuntimeSpeedTwol}{}\renewcommand{\RuntimeSpeedTwol}{363}
\providecommand{\RuntimeDart}{}\renewcommand{\RuntimeDart}{15.4}
\providecommand{\RuntimeBMA}{}\renewcommand{\RuntimeBMA}{0.21}
\providecommand{\EfourGMTwos}{}\renewcommand{\EfourGMTwos}{0.946}
\providecommand{\EfourWorstTwos}{}\renewcommand{\EfourWorstTwos}{1.039}
\providecommand{\EfourExceedTwos}{}\renewcommand{\EfourExceedTwos}{0}
\providecommand{\EfourGMTwol}{}\renewcommand{\EfourGMTwol}{0.950}
\providecommand{\EfourWorstTwol}{}\renewcommand{\EfourWorstTwol}{1.229}
\providecommand{\EfourExceedTwol}{}\renewcommand{\EfourExceedTwol}{7}
\providecommand{\EfourWorstSetting}{}\renewcommand{\EfourWorstSetting}{design B, $n=300$, $p=200$, $\rho=0.8$, SNR$=5$}
\providecommand{\DetCxThreeOne}{}\renewcommand{\DetCxThreeOne}{0.07}
\providecommand{\DetCxFourOne}{}\renewcommand{\DetCxFourOne}{0.06}
\providecommand{\ExactCeightOne}{}\renewcommand{\ExactCeightOne}{0.01}
\providecommand{\FPRmaxOne}{}\renewcommand{\FPRmaxOne}{0.013}
\providecommand{\DetCxThreeTwos}{}\renewcommand{\DetCxThreeTwos}{0.36}
\providecommand{\DetCxFourTwos}{}\renewcommand{\DetCxFourTwos}{0.37}
\providecommand{\ExactCeightTwos}{}\renewcommand{\ExactCeightTwos}{0.24}
\providecommand{\FPRmaxTwos}{}\renewcommand{\FPRmaxTwos}{0.013}
\providecommand{\DetCxThreeTwol}{}\renewcommand{\DetCxThreeTwol}{0.52}
\providecommand{\DetCxFourTwol}{}\renewcommand{\DetCxFourTwol}{0.51}
\providecommand{\ExactCeightTwol}{}\renewcommand{\ExactCeightTwol}{0.41}
\providecommand{\FPRmaxTwol}{}\renewcommand{\FPRmaxTwol}{0.012}
\providecommand{\EfiveFDRbig}{}\renewcommand{\EfiveFDRbig}{0.061}
\providecommand{\EfivePowbig}{}\renewcommand{\EfivePowbig}{0.47}
\providecommand{\EfiveFDRbigTwo}{}\renewcommand{\EfiveFDRbigTwo}{0.163}
\providecommand{\EfivePowbigTwo}{}\renewcommand{\EfivePowbigTwo}{0.64}
\providecommand{\EfivePowLassobig}{}\renewcommand{\EfivePowLassobig}{0.33}
\providecommand{\EfivePowsmallTwo}{}\renewcommand{\EfivePowsmallTwo}{0.49}
\providecommand{\EfivePowLassosmallTwo}{}\renewcommand{\EfivePowLassosmallTwo}{0.18}
\providecommand{\SeedSD}{}\renewcommand{\SeedSD}{0.0033}
\providecommand{\LowAUPRCLinearPFN}{}\renewcommand{\LowAUPRCLinearPFN}{0.866}
\providecommand{\LowAUPRCCLinearPFN}{}\renewcommand{\LowAUPRCCLinearPFN}{0.895}
\providecommand{\LowAUPRCTwos}{}\renewcommand{\LowAUPRCTwos}{0.916}
\providecommand{\LowAUPRCCTwos}{}\renewcommand{\LowAUPRCCTwos}{0.892}
\providecommand{\LowAUPRCTwol}{}\renewcommand{\LowAUPRCTwol}{0.923}
\providecommand{\LowAUPRCCTwol}{}\renewcommand{\LowAUPRCCTwol}{0.910}
\providecommand{\EsevenMinGap}{}\renewcommand{\EsevenMinGap}{0.17}
\providecommand{\EsevenMaxGap}{}\renewcommand{\EsevenMaxGap}{0.26}
\providecommand{\EsevenMaxP}{}\renewcommand{\EsevenMaxP}{$8.0\times10^{-174}$}
\providecommand{\EsevenBest}{}\renewcommand{\EsevenBest}{LassoNet}
\providecommand{\EsevenBestAUPRC}{}\renewcommand{\EsevenBestAUPRC}{0.365}
\providecommand{\EsevenAUPRCTwol}{}\renewcommand{\EsevenAUPRCTwol}{0.535}
\providecommand{\EsevenFDRTwol}{}\renewcommand{\EsevenFDRTwol}{0.23}
\providecommand{\EsevenFDRBest}{}\renewcommand{\EsevenFDRBest}{0.81}
\providecommand{\EsevenbNullTwol}{}\renewcommand{\EsevenbNullTwol}{0.110}
\providecommand{\EsevenbSelTwol}{}\renewcommand{\EsevenbSelTwol}{2.0}
\providecommand{\EsevenbNullTwos}{}\renewcommand{\EsevenbNullTwos}{0.116}
\providecommand{\EsevenbSelTwos}{}\renewcommand{\EsevenbSelTwos}{1.6}
\providecommand{\EsevenbNullLasso}{}\renewcommand{\EsevenbNullLasso}{0.390}
\providecommand{\EsevenbSelLasso}{}\renewcommand{\EsevenbSelLasso}{18.9}
\providecommand{\EsevenbNullLassoNet}{}\renewcommand{\EsevenbNullLassoNet}{0.276}
\providecommand{\EsevenbSelLassoNet}{}\renewcommand{\EsevenbSelLassoNet}{28.2}
\providecommand{\EsevenbNullBART}{}\renewcommand{\EsevenbNullBART}{0.499}
\providecommand{\EsevenbSelBART}{}\renewcommand{\EsevenbSelBART}{142.1}

\section{Simulation studies}\label{sec:experiments}

We evaluate SelectPFN in simulations that address four questions: how accurately it selects variables in nonlinear models that it never saw during training, how closely its amortized inclusion probabilities track the exact Bayesian posterior, whether restricting prediction to the selected variables costs accuracy, and whether knockoff mode controls the false discovery rate. \Cref{sec:exp-e7} then turns to real covariate matrices and real responses.

\subsection{Simulation setup}\label{sec:exp-setup}

\paragraph{SelectPFN models.} SelectPFN is trained on the meta-prior of \Cref{sec:metaprior} (\Cref{tab:metaprior}) at two sizes: \emph{SelectPFN}, with $d=128$, 6 layers and $1.9$M parameters, trained for 80k steps (4.9 GPU-hours on one RTX~4090), and \emph{SelectPFN-L}, with $d=256$, 8 layers and $9.5$M parameters, trained for 120k steps (17.8 GPU-hours). Training is done once; no model is retrained or tuned for any data set below. For specific studies we also train a model on a Gaussian-process prior whose posterior can be enumerated exactly, a knockoff model ($d=256$, 8 layers) on $2p$-column tables containing exact Gaussian knockoffs, and the ablation models of \Cref{sec:exp-e8}. Implementation details are in \Cref{app:implementation}.

\paragraph{Data-generating models.} None of the following designs is drawn from the meta-prior. Covariates are $\mathrm{AR}(1)$ Gaussian with correlation $\rho\in\{0,0.5,0.8\}$, transformed to uniform margins for design A, and the active set $S_0$ is placed at random positions. The response is $y=f(x_{S_0})+\varepsilon$ with Gaussian noise of variance $\sigma^2$ and signal-to-noise ratio $\mathrm{SNR}=\Var f/\sigma^2$.
\begin{itemize}
\item \textbf{A} (Friedman; \citealp{friedman1991mars}): $10\sin(\pi x_1x_2)+20(x_3-\tfrac12)^2+10x_4+5x_5$.
\item \textbf{B} (additive): $2\sin(2x_1)+x_2^2+x_3+e^{-x_4}+|x_5|$.
\item \textbf{C} (interactions and a threshold): $x_1x_2+\sin(x_3x_4)+\ind\{x_5>0\}x_6$. Here $x_3$ and $x_4$ affect the response only through $\sin(x_3x_4)$, so neither has a marginal association with $y$.
\item \textbf{D} (neural network): a two-hidden-layer $\tanh$ network (widths 16 and 8) on 8 inputs with a linear output, whose weights are Gaussian with variance $1/\text{fan-in}$ and biases $N(0,0.25)$, drawn once and kept fixed across replicates.
\item \textbf{L} and \textbf{L10} (sparse linear): $x_1-x_2+x_3-x_4+x_5$, and $\sum_{k=1}^{10}(-1)^{k+1}\beta_kx_k$ with $\beta_k$ equally spaced from $1$ to $0.3$.
\item \textbf{L20} and \textbf{B20} ($|S_0|=20$): $\sum_{k=1}^{20}(-1)^{k+1}\beta_kx_k$ with $\beta_k$ equally spaced from $1$ to $0.3$, and $\sum_{k=1}^{20}g_{k\bmod5}(x_k)$ with the five centred components of design B, $g_0(t)=2\sin(2t)$, $g_1(t)=t^2-1$, $g_2(t)=t$, $g_3(t)=e^{-t}-e^{1/2}$ and $g_4(t)=|t|-\sqrt{2/\pi}$, each applied to four variables.
\end{itemize}
The main study crosses $n\in\{100,300\}$, $p\in\{50,200,500\}$, the three values of $\rho$ and $\mathrm{SNR}\in\{1,5\}$, with 50 replicates per setting: 7{,}200 data sets. Further studies use the following settings.
\begin{itemize}
\item \emph{Support recovery:} $n$ up to $800$, plus the sparse linear designs L and L10; 30 replicates.
\item \emph{Low-dimensional slice:} $p=20$ and otherwise the main grid, so that LinearPFN, which accepts at most 30 predictors, can be compared on the same data sets.
\item \emph{Exact posterior:} 2{,}000 tasks from a Gaussian-process subset prior with $p\in\{6,\dots,12\}$, a shared squared-exponential kernel and hyperparameters on a 27-point grid, for which the posterior is computed by enumeration (\Cref{app:gp-subset}).
\item \emph{Calibration:} 10{,}000 held-out tasks from the meta-prior, about $1.06$ million variable--task pairs.
\item \emph{Knockoffs:} $n=300$, $p\in\{100,200\}$, $\rho=0.5$ (Gaussian margins also for design A), $\mathrm{SNR}\in\{1,5\}$ and 200 replicates, for designs A--D and the designs L20 and B20 with $|S_0|=20$.
\item \emph{Ablation:} design C, and for the prior-weight study design C with $p=50$, $\rho=0$, $n\in\{100,\dots,1000\}$ and $\mathrm{SNR}\in\{1,2,5,10,20\}$, 100 replicates (\Cref{sec:exp-e8}).
\end{itemize}

\subsection{Evaluation metrics}\label{sec:exp-metrics}
\begin{itemize}
\item \emph{Ranking accuracy:} the area under the precision--recall curve (AUPRC) of the variable scores against $S_0$, our primary metric because active variables are few, and the area under the ROC curve (AUROC).
\item \emph{Selected set:} the true positive rate (TPR), the false discovery proportion (FDP), and exact recovery, $\ind\{\hat S=S_0\}$. The selected set of SelectPFN is the median probability model (MPM), $\hat S=\{j:q_j>\frac12\}$; the other methods use their default or cross-validated rule.
\item \emph{Posterior accuracy:} the mean absolute difference $|q_j-\pi_j|$ between amortized and exact PIPs, the agreement of their MPMs, and the realized Bayes regret of \Cref{prop:regret}; for calibration, the expected calibration error (ECE) and the largest deviation of the reliability curve.
\item \emph{Prediction:} the held-out RMSE, reported as a ratio to TabPFN fitted on all columns.
\item \emph{Knockoff mode:} the empirical FDR and power of knockoff+ at levels $\alpha_{\FDR}\in\{0.1,0.2\}$.
\item \emph{Computation:} wall-clock time per data set.
\end{itemize}
All intervals are 95\% percentile bootstrap intervals over data sets, with 2{,}000 bootstrap resamples. Paired comparisons use the same data sets for all methods; they are tested with the two-sided Wilcoxon signed-rank test, with Holm-adjusted $p$-values within each table.

\subsection{Competing methods}\label{sec:exp-baselines}
We compare SelectPFN with fifteen methods that cover every family reviewed in \Cref{sec:related}.
\begin{itemize}
\item Penalized regression: the lasso, elastic net, SCAD and MCP \citep{tibshirani1996lasso,zou2005enet,fan2001scad,zhang2010mcp}, with penalties chosen by cross-validation.
\item Model-free screening: DC-SIS \citep{li2012dcsis} and HSIC-Lasso \citep{yamada2014hsic}, which rank variables only.
\item Gaussian-process regression with automatic relevance determination (GP-ARD).
\item Tree ensembles: BART and DART \citep{chipman2010bart,linero2018dart}. Their PIP is the fraction of posterior draws that use a variable; DART ranks variables by this PIP and BART by its average share of splits \citep{bleich2014bart}.
\item Sparse neural networks: LassoNet \citep{lemhadri2021lassonet}, a lightweight group-MCP sparse-input network$^\dagger$ in the spirit of GCRNet \citep{luo2025gcsnn}, and a sparse Bayesian network with a spike-and-slab prior on every weight$^\dagger$, after \citet{sun2022consistent}.
\item Foundation models: TabPFN v2 \citep{hollmann2025nature} with permutation importance, and with SHAP \citep{lundberg2017shap} on the low-dimensional slice.
\item Amortized selection: LinearPFN \citep{schiekiera2026linearpfn}, the closest amortized competitor, on the low-dimensional slice.
\end{itemize}
$^\dagger$ marks our own implementations: the sparse Bayesian network has no public code, and the group-MCP network is a fast, simplified version of GCRNet that fits the full benchmark at low cost; it is not tuned as in \citet{luo2025gcsnn}. For the sparse Bayesian network the sparsity level is chosen on a validation split; at these sample sizes it still selects conservatively, so its results are best read as a lower bound for the method. PACE \citep{qin2026pace} has no public implementation. For knockoff mode, the SelectPFN statistic $W_j=q_j-\tilde q_j$, with $\tilde q_j$ the inclusion probability of the knockoff of $x_j$, is compared with the lasso coefficient-difference and random-forest importance statistics, using exact or second-order knockoffs (Gaussian, with a Ledoit--Wolf covariance estimate; \citealp{ledoit2004}).

\subsection{Results}\label{sec:exp-results}

\subsubsection{Variable selection in nonlinear models}\label{sec:exp-e3}

\Cref{tab:e3main} summarizes selection over the 7{,}200 data sets of the main study, \Cref{tab:e3paired} gives paired comparisons, and \Cref{fig:auprc-design} breaks AUPRC down by design.

Both SelectPFN models rank variables more accurately than every competing method. SelectPFN-L reaches AUPRC 0.771, and SelectPFN 0.754, against 0.543 for the best competitor evaluated on all 7{,}200 data sets (LassoNet); its paired advantage over the competitors ranges from $\EthreeMinGap$ to $\EthreeMaxGap$, with every Holm-adjusted $p<10^{-100}$, and it is $0.16$ ahead of TabPFN with permutation importance. The advantage holds in every design (\Cref{fig:auprc-design}), and SelectPFN beats DART in $\EthreeCorrFrac$\% of the correlated settings and is never below it. Its selected sets are also far more reliable. The MPM of SelectPFN-L has FDP $\EthreeFDRTwol$, against $\EthreeFDRDart$ for DART and $0.38$--$0.92$ for the other methods with a selection rule (the sparse Bayesian network has a lower FDP only because it selects almost no variables), and it recovers the active set exactly in 21.5\% of the data sets, against at most 1.2\% for any competitor.

Two features of the method explain these gains. First, SelectPFN's inclusion probabilities average over all structures in a broad nonlinear meta-prior, so a variable is credited whether it acts additively, through an interaction or through a threshold; penalized and screening methods credit mainly the effects they model. Second, the MPM thresholds posterior probabilities rather than tuning a penalty for prediction, which is why its false discovery proportion is five to thirteen times lower than those of the cross-validated competitors. The larger network adds a further $\EthreeSizeGain$ [$\EthreeSizeGainlo$, $\EthreeSizeGainhi$] in AUPRC, mostly on design C, whose interaction-only variables are the subject of \Cref{sec:exp-e6,sec:exp-e8}.

\begin{table}[t]
\centering\scriptsize\setlength{\tabcolsep}{3pt}
\caption{Variable selection on designs A--D ($n\in\{100,300\}$, $p\in\{50,200,500\}$, $\rho\in\{0,0.5,0.8\}$, $\mathrm{SNR}\in\{1,5\}$; 50 replicates). TPR, FDP and exact recovery refer to each method's selected set: the MPM for SelectPFN, cross-validated or default selection for the others. Methods with fewer than 7{,}200 data sets were run on a subset of replicates; compare them through the paired results in \Cref{tab:e3paired}. $^\dagger$Our implementation. Bold: best value in each of the AUPRC, AUROC and exact-recovery columns.}\label{tab:e3main}
\begin{adjustbox}{max width=\linewidth}
\begin{tabular}{lrccccc}
\toprule
Method & data sets & AUPRC & AUROC & TPR (MPM) & FDP (MPM) & exact recovery \\
\midrule
SelectPFN & 7200 & 0.754 {\scriptsize[0.748, 0.759]} & 0.922 {\scriptsize[0.919, 0.924]} & 0.612 {\scriptsize[0.605, 0.619]} & 0.070 {\scriptsize[0.066, 0.073]} & 0.192 {\scriptsize[0.183, 0.202]} \\
SelectPFN-L & 7200 & \textbf{0.771} {\scriptsize[0.765, 0.776]} & \textbf{0.929} {\scriptsize[0.926, 0.931]} & 0.633 {\scriptsize[0.625, 0.640]} & 0.070 {\scriptsize[0.066, 0.073]} & \textbf{0.215} {\scriptsize[0.205, 0.224]} \\
DART & 7200 & 0.499 {\scriptsize[0.491, 0.506]} & 0.775 {\scriptsize[0.771, 0.780]} & 0.539 {\scriptsize[0.530, 0.547]} & 0.632 {\scriptsize[0.626, 0.639]} & 0.012 {\scriptsize[0.009, 0.014]} \\
BART & 7200 & 0.538 {\scriptsize[0.530, 0.545]} & 0.804 {\scriptsize[0.800, 0.808]} & 0.835 {\scriptsize[0.829, 0.842]} & 0.916 {\scriptsize[0.915, 0.918]} & 0.000 {\scriptsize[0.000, 0.000]} \\
Lasso & 7200 & 0.512 {\scriptsize[0.507, 0.518]} & 0.767 {\scriptsize[0.764, 0.771]} & 0.524 {\scriptsize[0.518, 0.531]} & 0.648 {\scriptsize[0.642, 0.653]} & 0.000 {\scriptsize[0.000, 0.000]} \\
Elastic net & 7200 & 0.490 {\scriptsize[0.485, 0.496]} & 0.769 {\scriptsize[0.766, 0.772]} & 0.566 {\scriptsize[0.560, 0.573]} & 0.756 {\scriptsize[0.751, 0.762]} & 0.000 {\scriptsize[0.000, 0.000]} \\
SCAD & 7200 & 0.513 {\scriptsize[0.507, 0.519]} & 0.763 {\scriptsize[0.759, 0.766]} & 0.500 {\scriptsize[0.493, 0.506]} & 0.570 {\scriptsize[0.563, 0.577]} & 0.000 {\scriptsize[0.000, 0.001]} \\
MCP & 7200 & 0.509 {\scriptsize[0.503, 0.515]} & 0.750 {\scriptsize[0.747, 0.754]} & 0.451 {\scriptsize[0.445, 0.458]} & 0.375 {\scriptsize[0.368, 0.382]} & 0.000 {\scriptsize[0.000, 0.001]} \\
DC-SIS & 7200 & 0.499 {\scriptsize[0.494, 0.505]} & 0.834 {\scriptsize[0.831, 0.837]} & -- & -- & -- \\
HSIC-Lasso & 7200 & 0.520 {\scriptsize[0.515, 0.526]} & 0.809 {\scriptsize[0.806, 0.811]} & -- & -- & -- \\
GP-ARD & 7200 & 0.503 {\scriptsize[0.496, 0.510]} & 0.835 {\scriptsize[0.832, 0.837]} & -- & -- & -- \\
Group-MCP NN$^\dagger$ & 7200 & 0.486 {\scriptsize[0.480, 0.492]} & 0.763 {\scriptsize[0.760, 0.767]} & 0.671 {\scriptsize[0.665, 0.676]} & 0.870 {\scriptsize[0.867, 0.873]} & 0.000 {\scriptsize[0.000, 0.000]} \\
LassoNet & 7200 & 0.543 {\scriptsize[0.536, 0.549]} & 0.805 {\scriptsize[0.802, 0.808]} & 0.573 {\scriptsize[0.566, 0.581]} & 0.628 {\scriptsize[0.621, 0.634]} & 0.000 {\scriptsize[0.000, 0.001]} \\
Sparse BNN$^\dagger$ & 7200 & 0.358 {\scriptsize[0.352, 0.365]} & 0.691 {\scriptsize[0.687, 0.695]} & 0.039 {\scriptsize[0.036, 0.041]} & 0.004 {\scriptsize[0.003, 0.005]} & 0.000 {\scriptsize[0.000, 0.000]} \\
TabPFN + perm.\ imp. & 1440 & 0.603 {\scriptsize[0.588, 0.618]} & 0.809 {\scriptsize[0.801, 0.818]} & -- & -- & -- \\
\bottomrule
\end{tabular}
\end{adjustbox}

\end{table}

\begin{table}[t]
\centering\scriptsize
\caption{Paired AUPRC differences on matched data sets. Holm adjustment is over all paired comparisons in \Cref{tab:e3-paired-full}.}\label{tab:e3paired}
\begin{adjustbox}{max width=\linewidth}
\begin{tabular}{lrccc}
\toprule
Comparison & matched & $\Delta$AUPRC [95\% CI] & Wilcoxon $p$ & Holm $p$ \\
\midrule
SelectPFN-L $-$ DART & 7200 & 0.271 [0.266, 0.277] & $<10^{-300}$ & $<10^{-300}$ \\
SelectPFN-L $-$ BART & 7200 & 0.233 [0.228, 0.237] & $<10^{-300}$ & $<10^{-300}$ \\
SelectPFN-L $-$ Lasso & 7200 & 0.259 [0.254, 0.263] & $<10^{-300}$ & $<10^{-300}$ \\
SelectPFN-L $-$ Elastic net & 7200 & 0.280 [0.276, 0.285] & $<10^{-300}$ & $<10^{-300}$ \\
SelectPFN-L $-$ SCAD & 7200 & 0.257 [0.252, 0.262] & $<10^{-300}$ & $<10^{-300}$ \\
SelectPFN-L $-$ MCP & 7200 & 0.261 [0.256, 0.266] & $<10^{-300}$ & $<10^{-300}$ \\
SelectPFN-L $-$ DC-SIS & 7200 & 0.271 [0.267, 0.275] & $<10^{-300}$ & $<10^{-300}$ \\
SelectPFN-L $-$ HSIC-Lasso & 7200 & 0.250 [0.246, 0.254] & $<10^{-300}$ & $<10^{-300}$ \\
SelectPFN-L $-$ GP-ARD & 7200 & 0.267 [0.262, 0.272] & $<10^{-300}$ & $<10^{-300}$ \\
SelectPFN-L $-$ Group-MCP NN$^\dagger$ & 7200 & 0.284 [0.280, 0.289] & $<10^{-300}$ & $<10^{-300}$ \\
SelectPFN-L $-$ LassoNet & 7200 & 0.228 [0.224, 0.232] & $<10^{-300}$ & $<10^{-300}$ \\
SelectPFN-L $-$ Sparse BNN$^\dagger$ & 7200 & 0.412 [0.406, 0.418] & $<10^{-300}$ & $<10^{-300}$ \\
SelectPFN-L $-$ TabPFN + perm.\ imp. & 1440 & 0.164 [0.156, 0.172] & $1.8\times10^{-190}$ & $5.5\times10^{-190}$ \\
SelectPFN-L $-$ SelectPFN & 7200 & 0.017 [0.016, 0.018] & $6.1\times10^{-133}$ & $6.1\times10^{-133}$ \\
\bottomrule
\end{tabular}
\end{adjustbox}

\end{table}

\begin{figure}[t]
\centering
\includegraphics[width=0.85\textwidth]{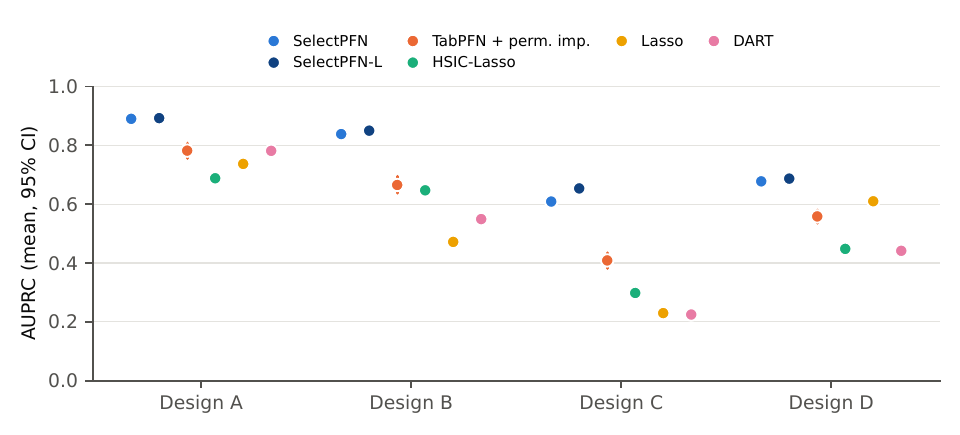}
\caption{Mean AUPRC by design with 95\% bootstrap intervals; most intervals are narrower than the markers. The two SelectPFN models are shades of one hue. Numerical values are in \Cref{tab:e3-design}.}\label{fig:auprc-design}
\end{figure}

\subsubsection{Comparison with an amortized linear selector}\label{sec:exp-lowdim}

On the low-dimensional slice ($p=20$; \Cref{tab:lowdim,tab:lowdim-paired}), every method can be compared with LinearPFN, which models main effects and pairwise products explicitly. Averaged over the four designs, SelectPFN outperforms it: AUPRC $\LowAUPRCTwos$ and $\LowAUPRCTwol$ for SelectPFN and SelectPFN-L, against $\LowAUPRCLinearPFN$ for LinearPFN, a paired advantage of $0.050$ [$0.046$, $0.054$] and $0.057$ [$0.054$, $0.061$]. The advantage comes from the designs that are not products of variables. On design C, where $x_1x_2$ and $\sin(x_3x_4)\approx x_3x_4$ near the origin suit LinearPFN's explicit product terms, the two are comparable ($\LowAUPRCCTwos$ for SelectPFN and $\LowAUPRCCLinearPFN$ for LinearPFN), and SelectPFN-L is ahead ($\LowAUPRCCTwol$). SelectPFN's selected set is also more conservative (FDP 0.068--0.070, against 0.113 for LinearPFN). Among the remaining methods, TabPFN with SHAP comes closest (0.825); both SelectPFN models exceed it (Holm-adjusted $p<10^{-84}$) and exceed every non-amortized method by at least 0.13 in mean AUPRC.

\begin{table}[t]
\centering\scriptsize\setlength{\tabcolsep}{3pt}
\caption{Low-dimensional slice: $p=20$, designs A--D, $n\in\{100,300\}$, three values of $\rho$, two SNRs, 50 replicates (2{,}400 data sets; TabPFN+SHAP 20 replicates). For LinearPFN, the variable score is the maximum PIP over the variable's main effect and all interactions involving it. Bold: best value in the AUPRC and per-design columns.}\label{tab:lowdim}
\begin{adjustbox}{max width=\linewidth}
\begin{tabular}{lrccccccc}
\toprule
Method & data sets & AUPRC & TPR (MPM) & FDP (MPM) & A & B & C & D \\
\midrule
SelectPFN & 2400 & 0.916 {\scriptsize[0.911, 0.920]} & 0.756 {\scriptsize[0.747, 0.766]} & 0.070 {\scriptsize[0.065, 0.075]} & 0.956 & 0.934 & 0.892 & 0.881 \\
SelectPFN-L & 2400 & \textbf{0.923} {\scriptsize[0.919, 0.927]} & 0.779 {\scriptsize[0.770, 0.789]} & 0.068 {\scriptsize[0.063, 0.073]} & \textbf{0.958} & \textbf{0.938} & \textbf{0.910} & \textbf{0.886} \\
LinearPFN & 2400 & 0.866 {\scriptsize[0.861, 0.871]} & 0.684 {\scriptsize[0.674, 0.693]} & 0.113 {\scriptsize[0.106, 0.120]} & 0.891 & 0.802 & 0.895 & 0.876 \\
TabPFN + SHAP & 960 & 0.825 {\scriptsize[0.815, 0.835]} & -- & -- & 0.913 & 0.831 & 0.739 & 0.818 \\
DART & 2400 & 0.743 {\scriptsize[0.734, 0.751]} & 0.684 {\scriptsize[0.670, 0.698]} & 0.397 {\scriptsize[0.385, 0.410]} & 0.869 & 0.742 & 0.635 & 0.724 \\
Lasso & 2400 & 0.719 {\scriptsize[0.712, 0.725]} & 0.632 {\scriptsize[0.621, 0.643]} & 0.442 {\scriptsize[0.433, 0.451]} & 0.846 & 0.676 & 0.532 & 0.821 \\
BART & 2400 & 0.736 {\scriptsize[0.728, 0.744]} & 1.000 {\scriptsize[1.000, 1.000]} & 0.700 {\scriptsize[0.698, 0.703]} & 0.871 & 0.743 & 0.619 & 0.711 \\
Group-MCP NN$^\dagger$ & 2400 & 0.766 {\scriptsize[0.760, 0.772]} & 0.853 {\scriptsize[0.845, 0.861]} & 0.574 {\scriptsize[0.566, 0.581]} & 0.847 & 0.708 & 0.686 & 0.822 \\
Sparse BNN$^\dagger$ & 2400 & 0.640 {\scriptsize[0.630, 0.649]} & 0.286 {\scriptsize[0.274, 0.297]} & 0.126 {\scriptsize[0.117, 0.136]} & 0.764 & 0.593 & 0.454 & 0.748 \\
\bottomrule
\end{tabular}
\end{adjustbox}

\end{table}

\subsubsection{Support recovery and interaction-only variables}\label{sec:exp-e6}

\Cref{tab:e6} reports exact recovery and the false-negative rate (FNR; fraction of active variables missed) of the MPM as $n$ grows; the full grid is in \Cref{tab:e6-full}. Exact recovery improves steadily with $n$, and false positives are rare throughout: the false-positive rate (FPR; fraction of inactive variables selected) is at most $\FPRmaxTwol$ in every design and at every $n$, so the remaining errors are almost all misses of weak effects.

Design C is the most demanding case, because $x_3$ and $x_4$ act only through $\sin(x_3x_4)$ and have no marginal association with $y$. Even so, SelectPFN includes them in $0.36$ of the data sets at $n=800$ (averaged over $x_3$ and $x_4$), and SelectPFN-L in $0.52$, with exact recovery on design C of $\ExactCeightTwos$ and $\ExactCeightTwol$. The next subsection shows that this ability comes from the interaction-only family of the meta-prior.

\begin{table}[t]
\centering\footnotesize
\caption{Exact recovery of $S_0$ by the MPM and its false-negative rate (FNR) (30 replicates per setting, averaged over $p\in\{50,200,500\}$, $\rho$ and SNR). The full grid is in \Cref{tab:e6-full}.}\label{tab:e6}
\begin{adjustbox}{max width=\linewidth}
\begin{tabular}{lrcccc}
\toprule
Design & $n$ & SelectPFN exact & SelectPFN FNR & SelectPFN-L exact & SelectPFN-L FNR \\
\midrule
C & 300 & 0.07 & 0.38 & 0.15 & 0.32 \\
C & 800 & 0.24 & 0.22 & 0.41 & 0.16 \\
D & 300 & 0.02 & 0.35 & 0.04 & 0.34 \\
D & 800 & 0.23 & 0.16 & 0.36 & 0.13 \\
L10 & 300 & 0.28 & 0.16 & 0.33 & 0.16 \\
L10 & 800 & 0.53 & 0.05 & 0.59 & 0.05 \\
\bottomrule
\end{tabular}
\end{adjustbox}

\end{table}

\subsubsection{Ablation: what the meta-prior must contain}\label{sec:exp-e8}

\Cref{thm:csharp} predicts that a variable can be detected only if the meta-prior gives weight to the kind of structure through which it acts. We test this prediction with two ablations at $d=128$.

\paragraph{Removing the interaction-only family.} We trained SelectPFN on a meta-prior without the interaction-only family and the weak-effect regime, in which every interacting variable of the pairwise family also has a main effect and the remaining families are reweighted (\Cref{tab:metaprior}). This model essentially never finds the interaction-only variables of design C (\Cref{tab:ablation}, \Cref{fig:design-c}): it includes $x_3,x_4$ in $0.06$ of the data sets at $n=800$, against $0.36$ for SelectPFN, and recovers the active set of design C exactly in $\ExactCeightOne$ of them, against $\ExactCeightTwos$. Over all designs, the interaction-only family adds $\EthreePriorGain$ [$\EthreePriorGainlo$, $\EthreePriorGainhi$] in AUPRC, almost entirely on design C, without lowering accuracy on the other designs.

\paragraph{Varying the prior weight.} Because this ablation changes more than one aspect of the meta-prior, we also trained nine models for 40k steps that differ \emph{only} in the weight $\omega\in\{0,0.01,0.03,0.12,0.30\}$ of the interaction-only family (two training seeds for each $\omega>0$), and evaluated each on 3{,}000 data sets from design C. Detection of $x_3,x_4$ is summarized by the information index $t_j=n\kappa_j^2/\sigma^2$ of \Cref{thm:csharp} and by $t_{50}$, the index at which detection reaches one half (\Cref{tab:e8,fig:e8}).

Detection rises monotonically with the prior weight, as the theory predicts. At $n=800$ and SNR 5, the seed-averaged detection is 0.01, 0.06, 0.12, 0.57 and 0.82 for $\omega=0,0.01,0.03,0.12,0.30$; at the largest index in the grid it is 0.09, 0.30, 0.40, 0.99 and 1.00. With $\omega=0$ the variables are essentially never found, at any sample size or SNR, as \Cref{prop:detect}(b) predicts. The threshold $t_{50}$ lies beyond the grid for $\omega\le0.03$, and falls to 536--622 at $\omega=0.12$ and to 275--394 at $\omega=0.30$. The extra power costs nothing in false selections: null variables enter the MPM at a rate of 0.003--0.005 for every $\omega$, and AUPRC on design C rises from 0.795 to 0.895.

The trained networks' threshold grows faster as $\omega$ falls than the logarithmic rate that \Cref{thm:csharp} gives for the exact posterior. An amortized model must also learn to recognize a structure from the training tasks, and rarely generated structures are learned less well. For an amortized selector, the weight of a structure in the meta-prior therefore matters at least as much as the theory for the exact posterior indicates, and the meta-prior should give substantial weight to every structure one wants to detect, as SelectPFN's does.

\begin{table}[t]
\centering\scriptsize
\caption{Ablation of the interaction-only family ($d=128$; main simulation study, low-dimensional slice and support-recovery study). Detection: how often the MPM includes $x_3,x_4$, which act only through $\sin(x_3x_4)$ in design C. Bold: better of the two models.}\label{tab:ablation}
\begin{adjustbox}{max width=\linewidth}
\begin{tabular}{lccccc}
\toprule
Model ($d{=}128$) & AUPRC, all designs & AUPRC, design C & AUPRC, design C, $p{=}20$ & detection of $x_3,x_4$, C, $n{=}800$ & exact recovery, C, $n{=}800$ \\
\midrule
SelectPFN (full meta-prior) & \textbf{0.754} {\scriptsize[0.748, 0.759]} & \textbf{0.609} & \textbf{0.892} & \textbf{0.36} & \textbf{0.24} \\
SelectPFN, no interaction-only family & 0.728 {\scriptsize[0.722, 0.734]} & 0.503 & 0.823 & 0.06 & 0.01 \\
\bottomrule
\end{tabular}
\end{adjustbox}

\end{table}

\begin{figure}[t]
\centering
\includegraphics[width=0.46\textwidth]{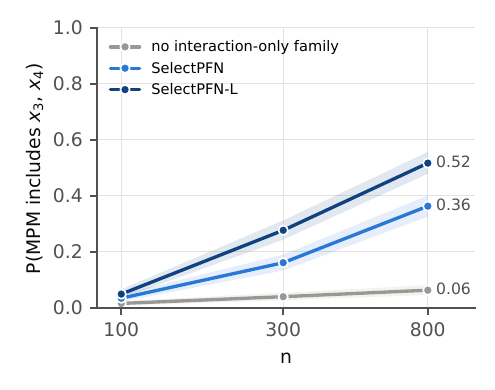}
\caption{Design C: how often the MPM includes the interaction-only variables $x_3,x_4$ (average of the two; band: 95\% bootstrap interval over data sets). Without the interaction-only family the meta-prior gives such structures no weight, and the variables are essentially never found, as \Cref{cor:prior-weight} and \Cref{thm:csharp} predict.}\label{fig:design-c}
\end{figure}

\begin{figure}[t]
\centering
\includegraphics[width=0.9\textwidth]{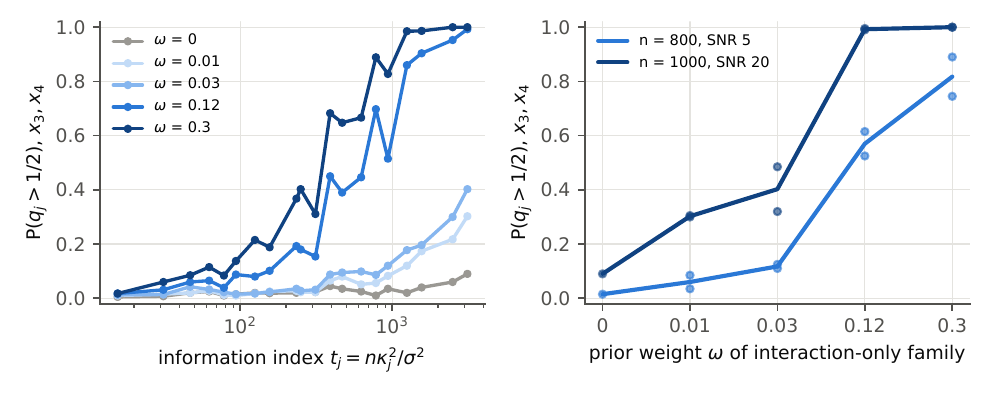}
\caption{Models that differ only in the prior weight $\omega$ of the interaction-only family. Left: detection of $x_3,x_4$ (acting only through $\sin(x_3x_4)$ in design C) against the information index $t_j$, seeds pooled. Right: detection at two fixed $(n,\mathrm{SNR})$; dots are training seeds, lines their mean.}\label{fig:e8}
\end{figure}

\begin{table}[t]
\centering\footnotesize
\caption{Prior-weight study on design C ($p=50$, $\rho=0$; 3{,}000 data sets per model). $t_{50}$: information index $t_j=n\kappa_j^2/\sigma^2$ at which detection ($q_j>\tfrac12$) reaches one half, from a logistic fit in $\log t_j$ with a bootstrap 95\% CI over data sets; ``not reached'': beyond the grid; $<56$: detection exceeds one half already at the smallest index. Det.: detection of $x_3,x_4$ at the stated $(n,\mathrm{SNR})$. ``Reference'' models are the main models, which differ in more than $\omega$.}\label{tab:e8}
\begin{adjustbox}{max width=\linewidth}
\begin{tabular}{lccccc}
\toprule
Model & $t_{50}$, $x_3,x_4$ & $t_{50}$, $x_1,x_2$ & det.\ $x_3,x_4$, $n{=}800$, SNR 5 & det.\ $x_3,x_4$, $n{=}1000$, SNR 20 & AUPRC \\
\midrule
$\omega=0.00$, seed 0 & not reached & 91 [79, 99] & 0.01 & 0.09 & 0.795 \\
$\omega=0.01$, seed 0 & not reached & 111 [101, 121] & 0.04 & 0.30 & 0.810 \\
$\omega=0.01$, seed 1 & not reached & 96 [87, 105] & 0.09 & 0.30 & 0.814 \\
$\omega=0.03$, seed 0 & not reached & 92 [81, 101] & 0.12 & 0.32 & 0.823 \\
$\omega=0.03$, seed 1 & not reached & 86 [76, 94] & 0.11 & 0.48 & 0.825 \\
$\omega=0.12$, seed 0 & 536 [513, 562] & 60 & 0.61 & 0.99 & 0.870 \\
$\omega=0.12$, seed 1 & 622 [591, 655] & 61 & 0.53 & 0.99 & 0.873 \\
$\omega=0.30$, seed 0 & 275 [262, 289] & $<56$ & 0.89 & 1.00 & 0.901 \\
$\omega=0.30$, seed 1 & 394 [376, 415] & $<56$ & 0.74 & 1.00 & 0.888 \\
no interaction-only family (reference) & not reached & 86 [77, 94] & 0.01 & 0.04 & 0.791 \\
SelectPFN-L (reference) & 158 [149, 166] & $<56$ & 0.99 & 1.00 & 0.924 \\
SelectPFN (reference) & 249 [236, 262] & $<56$ & 0.97 & 1.00 & 0.905 \\
SelectPFN, seed 1 (reference) & 405 [386, 427] & 57 & 0.76 & 1.00 & 0.888 \\
SelectPFN, seed 2 (reference) & 354 [338, 373] & 58 & 0.80 & 1.00 & 0.894 \\
\bottomrule
\end{tabular}
\end{adjustbox}

\end{table}

\subsubsection{Accuracy of the amortized posterior}\label{sec:exp-e1}

\paragraph{Agreement with an exact posterior.} On the Gaussian-process subset prior the posterior can be enumerated, so the amortized PIPs can be compared with the exact ones (\Cref{tab:e1}). The amortized posterior is close to the exact one and approaches it steadily as training compute grows. The mean absolute PIP error is $0.047$ after 25k training steps, $\EoneMAE$ after 100k steps and $\EoneLMAE$ for the larger model ($d=256$, 8 layers), and the gap in predictive log loss to the exact posterior falls from $0.036$ to $0.024$ and $\EoneLNLL$. The selected sets agree exactly with the exact MPM on $\EoneMPM$ of the tasks for the $d=128$ model and on $\EoneLMPM$ for the larger one, and on $\EoneLNT$ of the tasks without a near-tie. The remaining disagreements are cheap. Agreement is an exact-match statistic over all variables of a task, and it falls as the true support grows even though $95.8\%$ of individual variables agree. The disagreements concentrate on near-ties: the 5\% of variables with $|\pi_j-\frac12|<0.1$ account for $\EoneNearTieShare$\% of them, and there the choice hardly matters for the Bayes risk. Indeed, the realized Bayes regret of the amortized MPM is $\EoneLRegret$ per variable for the larger model, within the bound $\EoneLRegretBound$ of \Cref{prop:regret}. DART and GP-ARD, which estimate the same quantities by per-data-set inference, are much further from the exact posterior.

\begin{table}[t]
\centering\footnotesize
\caption{Amortized vs.\ exact posterior on 2{,}000 Gaussian-process subset tasks. MPM agreement: exact match of $\{j:q_j>\tfrac12\}$ and $\{j:\pi_j>\tfrac12\}$ on a task. No near-ties: tasks with all $|\pi_j-\tfrac12|\ge0.1$. $^\ddagger$Exact $\pi$ plus random errors of the network's size (sign-randomized $|q-\pi|$, permuted within the task). NLL $-$ exact: mean predictive negative log-likelihood minus that of the exact posterior. $^\ast$GP-ARD, fitted by type-II maximum likelihood (ML-II), has no PIP; its selected set is the variables with ARD length-scale $\ell_j<5$. Bold: best value in each column, excluding the reference column.}\label{tab:e1}
\begin{adjustbox}{max width=\linewidth}
\begin{tabular}{lccccc}
\toprule
Method & mean $|q-\pi|$ & MPM agreement & agreement, no near-ties & reference$^\ddagger$ & NLL $-$ exact \\
\midrule
SelectPFN, $d{=}128$, 25k steps & 0.047 {\scriptsize[0.046, 0.049]} & 0.780 {\scriptsize[0.762, 0.799]} & 0.875 {\scriptsize[0.860, 0.891]} & 0.832 {\scriptsize[0.816, 0.848]} & 0.036 \\
SelectPFN, $d{=}128$, 100k steps & 0.037 {\scriptsize[0.036, 0.039]} & 0.820 {\scriptsize[0.803, 0.837]} & 0.916 {\scriptsize[0.903, 0.928]} & 0.861 {\scriptsize[0.846, 0.875]} & 0.024 \\
SelectPFN, $d{=}256$, 8 layers, 80k steps & \textbf{0.028} {\scriptsize[0.027, 0.030]} & \textbf{0.846} {\scriptsize[0.831, 0.862]} & \textbf{0.942} {\scriptsize[0.931, 0.953]} & 0.884 {\scriptsize[0.871, 0.898]} & \textbf{0.016} \\
DART & 0.335 & 0.156 & 0.167 & -- & 0.180 \\
GP-ARD (ML-II)$^\ast$ & -- & 0.271 & 0.307 & -- & 0.117 \\
\bottomrule
\end{tabular}
\end{adjustbox}

\end{table}

\paragraph{Calibration.} On 10{,}000 held-out meta-prior tasks (\Cref{tab:e2}, \Cref{fig:reliability}), the inclusion probabilities are very well calibrated: the ECE is $0.0005$ for SelectPFN and $0.0009$ for SelectPFN-L, and the reliability curve stays within $\pm0.05$ in every stratum defined by an observable quantity ($n$, $p$ or the dependence type). Within strata defined by latent quantities, such as the true number of active variables, calibration is not expected even of the exact posterior, which we confirmed on the exact-posterior tasks; per-stratum results are in \Cref{tab:e2-strata}.

\begin{table}[t]
\centering\footnotesize
\caption{Calibration of $q$ on 10{,}000 held-out meta-prior tasks (equal-mass bins). The last column counts the strata, defined by $n$, $p$ or the dependence type, in which the reliability curve stays within $\pm0.05$. Max dev.: largest absolute deviation of the reliability curve from the diagonal.}\label{tab:e2}
\begin{adjustbox}{max width=\linewidth}
\begin{tabular}{lcccc}
\toprule
Model & ECE & max dev. & AUPRC & strata within $\pm0.05$ \\
\midrule
SelectPFN & 0.0005 & 0.0034 & 0.691 & 10/10 \\
SelectPFN-L & 0.0009 & 0.0035 & 0.706 & 10/10 \\
\bottomrule
\end{tabular}
\end{adjustbox}

\end{table}

\begin{figure}[t]
\centering
\includegraphics[width=0.48\textwidth]{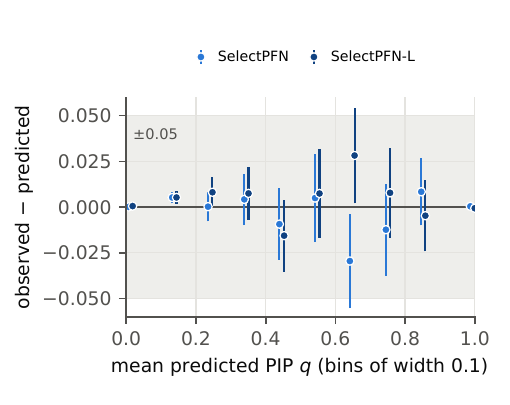}
\caption{Reliability on held-out meta-prior tasks: observed frequency of activity minus mean $q$ in bins of width 0.1, with 95\% binomial intervals. The shaded band is $\pm0.05$.}\label{fig:reliability}
\end{figure}

\subsubsection{Prediction from the selected variables}\label{sec:exp-e4}

Restricting prediction to the selected variables does not cost accuracy. Over the 144 settings of the main study, the model-averaged prediction of SelectPFN is \emph{more} accurate than TabPFN fitted on all $p$ columns, with geometric-mean RMSE ratios $\EfourGMTwos$ for SelectPFN and $\EfourGMTwol$ for SelectPFN-L (\Cref{tab:e4}). For SelectPFN the ratio stays below $1.05$ in every setting. Discarding irrelevant columns removes noise that a model using all columns must learn to ignore, which is why selection can improve prediction at these sample sizes. Relative to an oracle that is given the true active set, predicting from the MPM costs $8.7\%$ in RMSE, and averaging over sampled subsets reduces this to $7.3\%$. For SelectPFN-L, accuracy drops on design B at $n=300$ and $\mathrm{SNR}=5$ in the oracle-subset prediction as well, so the loss lies in the prediction pass rather than in selection; about half of it comes from extrapolation to extreme responses (\Cref{tab:e6-rmse}). We therefore use SelectPFN ($d=128$) for prediction.

\begin{table}[t]
\centering\footnotesize
\caption{Prediction on the 144 settings of the main study (geometric means of RMSE ratios over settings; 95\% bootstrap interval over settings). BMA: SelectPFN's model-averaged prediction over 64 sampled subsets; TabPFN-all: TabPFN fitted on all $p$ columns; settings $>1.05$ and worst: number of settings with BMA/TabPFN-all ratio above $1.05$, and the largest such ratio. MPM/oracle subset: prediction from the MPM relative to prediction from the true active set. Oracle subset/TabPFN-$S_0$ compares SelectPFN's prediction pass with TabPFN when both are given the true active set. BMA CRPS: mean continuous ranked probability score of the BMA predictive distribution.}\label{tab:e4}
\begin{adjustbox}{max width=\linewidth}
\begin{tabular}{lcccccc}
\toprule
Model & BMA/TabPFN-all (geo-mean) & settings $>1.05$ & worst & MPM/oracle subset & oracle subset/TabPFN-$S_0$ & BMA CRPS \\
\midrule
SelectPFN & 0.946 {\scriptsize[0.936, 0.955]} & 0/144 & 1.039 & 1.087 & 0.974 & 1.228 \\
SelectPFN-L & 0.950 {\scriptsize[0.936, 0.963]} & 7/144 & 1.229 & 1.087 & 0.981 & 1.216 \\
\bottomrule
\end{tabular}
\end{adjustbox}

\end{table}

\subsubsection{False discovery rate control in knockoff mode}\label{sec:exp-e5}

With exact knockoffs, the FDR is at or below the nominal level for every statistic and design (\Cref{tab:e5}, \Cref{fig:knockoff}), as \Cref{thm:knockoff} guarantees. Among valid statistics, SelectPFN is the most powerful. On the designs with $|S_0|=20$ its power at $\alpha_{\FDR}=0.1$ is $\EfivePowbig$, against $\EfivePowLassobig$ for the lasso statistic, with FDR $\EfiveFDRbig$; at $\alpha_{\FDR}=0.2$ its power is $\EfivePowbigTwo$, with FDR $\EfiveFDRbigTwo$. On designs A--D at $\alpha_{\FDR}=0.2$ its power is $\EfivePowsmallTwo$, against $\EfivePowLassosmallTwo$ for the lasso statistic. With only five to eight active variables, knockoff+ at $\alpha_{\FDR}=0.1$ cannot reach its threshold, which requires at least ten rejections, so power at that level is near zero for every statistic. With second-order knockoffs the FDR rises slightly, to $0.08$ at $\alpha_{\FDR}=0.1$, consistent with the guarantee holding to the extent that the knockoffs are valid.

\begin{figure}[t]
\centering
\includegraphics[width=0.85\textwidth]{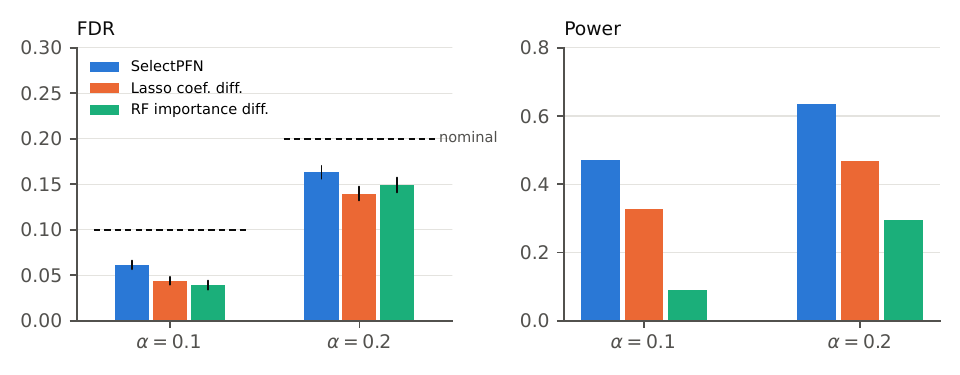}
\caption{Knockoff mode: FDR (left; whiskers $\pm2$ s.e.; dashed: nominal level) and power (right) of knockoff+ with exact knockoffs on the $|S_0|=20$ designs L20 and B20 ($n=300$, $p\in\{100,200\}$; 1{,}600 data sets).}\label{fig:knockoff}
\end{figure}

\begin{table}[t]
\centering\scriptsize
\caption{Knockoff+ FDR (standard error) and power at levels $\alpha_{\FDR}=0.1$ and $0.2$ (@0.1, @0.2); reps: number of data sets. With $|S_0|\le8$, knockoff+ at $\alpha_{\FDR}=0.1$ cannot reach its threshold, so power at $\alpha_{\FDR}=0.1$ is near zero for every statistic. Bold: highest power within each block of designs and knockoff type.}\label{tab:e5}
\begin{adjustbox}{max width=\linewidth}
\begin{tabular}{lllccccr}
\toprule
Designs & knockoffs & Statistic & FDR@0.1 (se) & power@0.1 & FDR@0.2 (se) & power@0.2 & reps \\
\midrule
$|S_0|\le8$ (A--D) & exact & SelectPFN (knockoff mode) & 0.018 (0.002) & \textbf{0.04} & 0.148 (0.004) & \textbf{0.49} & 3200 \\
$|S_0|\le8$ (A--D) & exact & Lasso coef.\ diff. & 0.008 (0.001) & 0.01 & 0.086 (0.003) & 0.18 & 3200 \\
$|S_0|\le8$ (A--D) & exact & RF importance diff. & 0.012 (0.001) & 0.02 & 0.126 (0.004) & 0.21 & 3200 \\
$|S_0|\le8$ (A--D) & 2nd-order & SelectPFN (knockoff mode) & 0.015 (0.002) & 0.03 & 0.143 (0.004) & \textbf{0.48} & 3200 \\
$|S_0|\le8$ (A--D) & 2nd-order & Lasso coef.\ diff. & 0.010 (0.001) & 0.02 & 0.094 (0.003) & 0.20 & 3200 \\
$|S_0|\le8$ (A--D) & 2nd-order & RF importance diff. & 0.038 (0.003) & \textbf{0.05} & 0.234 (0.005) & 0.33 & 3200 \\
$|S_0|=20$ (L20, B20) & exact & SelectPFN (knockoff mode) & 0.061 (0.002) & \textbf{0.47} & 0.163 (0.004) & \textbf{0.64} & 1600 \\
$|S_0|=20$ (L20, B20) & exact & Lasso coef.\ diff. & 0.044 (0.002) & 0.33 & 0.140 (0.004) & 0.47 & 1600 \\
$|S_0|=20$ (L20, B20) & exact & RF importance diff. & 0.039 (0.002) & 0.09 & 0.149 (0.004) & 0.29 & 1600 \\
$|S_0|=20$ (L20, B20) & 2nd-order & SelectPFN (knockoff mode) & 0.080 (0.003) & \textbf{0.51} & 0.194 (0.004) & \textbf{0.66} & 1600 \\
$|S_0|=20$ (L20, B20) & 2nd-order & Lasso coef.\ diff. & 0.059 (0.003) & 0.36 & 0.169 (0.004) & 0.49 & 1600 \\
$|S_0|=20$ (L20, B20) & 2nd-order & RF importance diff. & 0.086 (0.004) & 0.17 & 0.230 (0.005) & 0.38 & 1600 \\
\bottomrule
\end{tabular}
\end{adjustbox}

\end{table}

\subsubsection{Computation time and robustness to training randomness}\label{sec:exp-runtime}

SelectPFN needs no fitting, tuning or convergence diagnostics for a new data set. At $p=500$ its selection pass takes $\RuntimeSPTwos$\,s per data set for SelectPFN ($\RuntimeSPTwol$\,s for SelectPFN-L), against $\RuntimeDart$\,s for DART, a speed-up of $\RuntimeSpeedTwos\times$ ($\RuntimeSpeedTwol\times$). A model-averaged prediction over 64 sampled subsets adds $\RuntimeBMA$\,s, because predictions are cached for each distinct subset. All times were measured on one RTX~4090.

The results do not depend on a lucky training run. Over three independent training seeds of SelectPFN, the AUPRC of the main study has standard deviation \SeedSD, an order of magnitude below the effect of the interaction-only family ($\EthreePriorGain$; \Cref{tab:seeds}).

\begin{table}[t]
\centering\footnotesize
\caption{Training-seed variability of SelectPFN ($d=128$; seed 0 is the model reported elsewhere).}\label{tab:seeds}
\begin{adjustbox}{max width=\linewidth}
\begin{tabular}{lccccc}
\toprule
Seed & AUPRC & exact recovery & AUPRC, design C & ECE & exact recovery, C, $n{=}800$ \\
\midrule
0 & 0.754 & 0.192 & 0.609 & 0.0005 & 0.244 \\
1 & 0.747 & 0.179 & 0.589 & 0.0007 & 0.187 \\
2 & 0.749 & 0.185 & 0.594 & 0.0005 & 0.176 \\
mean $\pm$ SD & 0.750 $\pm$ 0.003 & 0.185 $\pm$ 0.007 & 0.597 $\pm$ 0.010 & 0.001 $\pm$ 0.000 & 0.202 $\pm$ 0.037 \\
\bottomrule
\end{tabular}
\end{adjustbox}

\end{table}

\subsection{Summary}\label{sec:exp-summary}
Across 7{,}200 simulated data sets from nonlinear designs that it never saw during training, SelectPFN ranks variables more accurately than every penalized, tree-based, sparse-network and foundation-model competitor, by 0.23--0.28 in AUPRC over the established methods, and its selected sets have an FDP near 0.07 against 0.38--0.92 for the competitors (except the sparse Bayesian network, which selects almost nothing). It does so in a few hundredths of a second per data set, without tuning. Its inclusion probabilities are close to the exact Bayesian posterior and well calibrated, and their accuracy improves with training compute. Predictions that use only the selected variables are as accurate as a foundation model that sees every column. In knockoff mode the FDR is controlled exactly, with the highest power among the statistics compared. Finally, the controlled prior-weight study confirms the main theoretical prediction: what the meta-prior contains determines which variables can be found, and giving weight to a structure makes the variables that use it detectable at no cost in false selections.

\section{Real data analysis}\label{sec:exp-e7}

The simulations use Gaussian or Gaussian-copula covariates. Real covariate matrices are harder: their columns are strongly dependent, often skewed, heavy-tailed or discrete, and their dependence is not of the form used in the meta-prior. We therefore evaluate SelectPFN on the design matrices of four public OpenML data sets from different domains: superconductivity of materials (superconduct, $p=77$), community crime rates (us\_crime, $p=98$), near-infrared spectra of meat samples (tecator, $p=25$) and molecular descriptors (topo\_2\_1, $p=189$).

\subsection{Setup}\label{sec:real-setup}
We kept numeric columns without missing values and rows with an observed response, removed constant columns and near-duplicates ($|r|>0.99$), dropped columns that are constant in a training sample, and standardized the columns. The remaining columns are strongly dependent: the median, over columns, of the largest absolute correlation with another column ranges from 0.70 (tecator) to 0.95 (topo\_2\_1). Two studies are run on these matrices.
\begin{itemize}
\item \emph{Known truth on real covariates.} The response is generated from designs A--D applied to randomly chosen real columns (clipped at $\pm4$, rank-transformed to $(0,1)$ for design A, with noise variance set from the empirical $\Var f$ and the SNR), with $n\in\{100,300\}$ ($n=100$ only for tecator), $\mathrm{SNR}\in\{1,5\}$ and 25 replicates: 1{,}400 data sets. The metrics and competitors are those of \Cref{sec:experiments}.
\item \emph{Real responses with known null variables.} Both $X$ and $y$ are real, so the true active set is unknown. To create variables that are known to be irrelevant, we append $\min(p,500-p)$ columns that are independent row permutations of real columns; they are independent of $y$ and of all other columns, so any selected permuted column is a false selection. We record the number of selected variables, the share of permuted columns among them, and the held-out RMSE ($n\in\{100,300\}$, 50 replicates).
\end{itemize}

\subsection{Results}\label{sec:real-results}

\paragraph{Known truth on real covariates.} SelectPFN keeps its advantage on real covariate matrices (\Cref{tab:e7a}). SelectPFN-L reaches AUPRC $\EsevenAUPRCTwol$, against $\EsevenBestAUPRC$ for the best competitor (\EsevenBest), and it is ahead on every design and every data set. Its paired advantage over the competitors ranges from $\EsevenMinGap$ to $\EsevenMaxGap$ (Holm-adjusted over the seven comparisons, $p\le{}$\EsevenMaxP). Strong dependence makes selection harder for every method, but much less so for SelectPFN: its FDP is $\EsevenFDRTwol$, against $0.74$--$0.92$ for the competitors' selected sets. As a Bayesian procedure, SelectPFN shares inclusion probability among near-collinear columns (\Cref{sec:metaprior}) instead of selecting all of them, which keeps its selected sets small.

\paragraph{Real responses.} With real responses, SelectPFN selects parsimonious models with few false selections (\Cref{tab:e7b}). SelectPFN-L selects $\EsevenbSelTwol$ variables on average, and only $\EsevenbNullTwol$ of them are permuted nulls. The competitors select many more variables and many more nulls: the null share is $\EsevenbNullLassoNet$ for LassoNet ($\EsevenbSelLassoNet$ variables selected), $\EsevenbNullLasso$ for the lasso and $\EsevenbNullBART$ for BART. Parsimony does not necessarily cost predictive accuracy. On superconduct, SelectPFN-L predicts better than the penalized and network competitors (RMSE 0.59 against 0.64--0.67) using 1--4 variables, against 15--35 for the lasso. Averaged over data sets its RMSE is 30--33\% above that of TabPFN on all columns (22--25\% for the lasso, SCAD and LassoNet); most of this difference comes from tecator, whose response is spread over many correlated spectral columns, a dense structure that the sparse meta-prior does not generate.

\begin{table}[t]
\centering\scriptsize
\caption{Known truth on real covariates: selection with real covariate matrices and responses from designs A--D on randomly chosen columns (1{,}400 data sets). Per-data-set columns are mean AUPRC. FDP: false discovery proportion of each method's selected set. Bold: best value in the AUPRC and per-data-set columns.}\label{tab:e7a}
\begin{adjustbox}{max width=\linewidth}
\begin{tabular}{lrcccccc}
\toprule
Method & data sets & AUPRC & superconduct & us\_crime & tecator & topo\_2\_1 & FDP \\
\midrule
SelectPFN & 1400 & 0.517 {\scriptsize[0.506, 0.530]} & 0.490 & 0.613 & 0.694 & 0.361 & 0.21 {\scriptsize[0.19, 0.22]} \\
SelectPFN-L & 1400 & \textbf{0.535} {\scriptsize[0.522, 0.547]} & \textbf{0.501} & \textbf{0.639} & \textbf{0.710} & \textbf{0.376} & 0.23 {\scriptsize[0.21, 0.25]} \\
LassoNet & 1400 & 0.365 {\scriptsize[0.354, 0.375]} & 0.331 & 0.428 & 0.563 & 0.236 & 0.81 {\scriptsize[0.80, 0.82]} \\
BART & 1400 & 0.352 {\scriptsize[0.340, 0.365]} & 0.332 & 0.423 & 0.537 & 0.210 & 0.92 {\scriptsize[0.91, 0.92]} \\
DART & 1400 & 0.331 {\scriptsize[0.318, 0.344]} & 0.307 & 0.417 & 0.540 & 0.163 & 0.78 {\scriptsize[0.77, 0.79]} \\
Lasso & 1400 & 0.279 {\scriptsize[0.269, 0.288]} & 0.260 & 0.296 & 0.513 & 0.164 & 0.81 {\scriptsize[0.80, 0.82]} \\
SCAD & 1400 & 0.270 {\scriptsize[0.261, 0.279]} & 0.250 & 0.287 & 0.504 & 0.155 & 0.74 {\scriptsize[0.73, 0.75]} \\
Group-MCP NN$^\dagger$ & 1400 & 0.310 {\scriptsize[0.300, 0.321]} & 0.286 & 0.335 & 0.549 & 0.191 & 0.85 {\scriptsize[0.85, 0.86]} \\
HSIC-Lasso & 1400 & 0.276 {\scriptsize[0.267, 0.285]} & 0.252 & 0.301 & 0.510 & 0.158 & -- \\
\bottomrule
\end{tabular}
\end{adjustbox}

\end{table}

\begin{table}[t]
\centering\scriptsize
\caption{Real responses with appended permuted null columns (350 data sets). ``Selected'' is the mean size of the selected set, and the next two columns count permuted nulls in it. RMSE/TabPFN is the geometric mean, over data sets, of the held-out RMSE ratio to TabPFN fitted on all columns. Bold: lowest null share and lowest RMSE ratio.}\label{tab:e7b}
\begin{adjustbox}{max width=\linewidth}
\begin{tabular}{lrcccc}
\toprule
Method & data sets & selected & permuted nulls selected & share of selection that is null & RMSE / TabPFN \\
\midrule
SelectPFN & 350 & 1.6 & 0.31 & 0.116 {\scriptsize[0.087, 0.145]} & 1.325 {\scriptsize[1.266, 1.391]} \\
SelectPFN-L & 350 & 2.0 & 0.22 & \textbf{0.110} {\scriptsize[0.080, 0.140]} & 1.299 {\scriptsize[1.239, 1.367]} \\
LassoNet & 350 & 28.2 & 13.06 & 0.276 {\scriptsize[0.248, 0.303]} & 1.245 {\scriptsize[1.205, 1.290]} \\
BART & 350 & 142.1 & 70.89 & 0.499 {\scriptsize[0.497, 0.501]} & 1.361 {\scriptsize[1.281, 1.451]} \\
DART & 350 & 20.6 & 8.65 & 0.401 {\scriptsize[0.375, 0.425]} & 1.347 {\scriptsize[1.272, 1.431]} \\
Lasso & 350 & 18.9 & 9.13 & 0.390 {\scriptsize[0.364, 0.417]} & 1.225 {\scriptsize[1.172, 1.292]} \\
SCAD & 350 & 11.6 & 5.43 & 0.317 {\scriptsize[0.289, 0.343]} & \textbf{1.217} {\scriptsize[1.177, 1.272]} \\
Group-MCP NN$^\dagger$ & 350 & 42.4 & 26.81 & 0.602 {\scriptsize[0.582, 0.623]} & 1.649 {\scriptsize[1.508, 1.819]} \\
\bottomrule
\end{tabular}
\end{adjustbox}

\end{table}

\subsection{Summary}
On real covariate matrices from four domains, SelectPFN identifies the variables that generate the response far more accurately than the competing methods, with gaps of 0.17--0.26 in AUPRC and an FDP a quarter to a third of theirs. With real responses it returns small models in which permuted noise variables are rare, and its predictions remain competitive where the signal is sparse.

\section{Discussion}\label{sec:discussion}

SelectPFN shows that variable selection can be amortized for a broad class of sparse nonlinear regression models. It does so without giving up the two properties that distinguish integrated selection from post-hoc importance: predictions that provably use only the selected variables, and, in knockoff mode, exact finite-sample FDR control that does not depend on the meta-prior. On out-of-prior designs it recovers active sets markedly better than penalized, Bayesian-tree and sparse-neural-network methods, at a small fraction of their cost.

\paragraph{What the theory says about the prior.} \Cref{thm:csharp} turns a modelling intuition into a quantitative statement. A variable is detected once its information index exceeds about $(2/\lambda_j)\log(1/\omega)$, where $\omega$ is the prior weight of structures that can use the variable as the truth does. Breadth of the meta-prior therefore matters, and it matters logarithmically: a structure the prior never generates ($\omega=0$) is essentially never found, whatever the sample size. Raising its weight from negligible to moderate shifts the detection threshold by a computable amount. Our ablations confirm this mechanism: without interaction-only structures in the meta-prior, variables that act only through an interaction are essentially never found, and in a controlled sweep of their weight detection rises monotonically with the weight, at no cost in false selections (\Cref{sec:exp-e8}). For a trained network the effect is even stronger than for the exact posterior, because rarely generated structures are also learned less well. It also gives a principled way to decide which families a foundation prior for selection must contain.

\paragraph{What amortization cannot give for free.} Training minimizes an average over the prior. Proposition~\ref{prop:notransfer} shows that such average-case guarantees do not transfer to a fixed truth, so frequentist selection consistency of an amortized posterior requires a uniform-amortization assumption that training does not certify. Knockoff mode sidesteps this for error control. For power and calibration at a fixed truth, the assumption remains the main open problem.

\paragraph{Limitations.}
\begin{itemize}
\item \emph{Exactness.} The median probability model agrees exactly with that of an enumerable posterior on $\EoneLMPM$ of the tasks for the larger model; the remaining disagreements concentrate at near-ties, which cost little Bayes risk (\Cref{prop:regret}), and agreement improves steadily with training compute (\Cref{sec:exp-e1}).
\item \emph{Factorized selection head.} Model averaging samples subsets from the product of the marginal PIPs and ignores posterior dependence between inclusions. An autoregressive selection head would remove this.
\item \emph{Scale.} All experiments use $p\le500$. The block-wise scheme of \Cref{sec:blockwise} is implemented but not validated beyond that.
\item \emph{Prior coverage.} Like any pretrained model, SelectPFN reflects what its meta-prior contains. Dense signals, such as spectra, are not generated by the current prior (\Cref{sec:exp-e7}); by \Cref{thm:csharp}, adding dense and low-rank families is the way to cover them.
\item \emph{Baselines.} PACE has no public code. LinearPFN supports only $p\le30$ and is compared on a $p=20$ slice. The spike-and-slab sparse neural network is our own implementation of \citet{sun2022consistent}; with $n\le800$ it selects conservatively, and a tuned implementation by its authors could do better.
\item \emph{Prediction at the larger size.} The larger model (width $d=256$ with 8 layers) improves selection but regresses in prediction on one design. About half of the regression comes from extrapolation to extreme responses, and the rest is unexplained (\Cref{sec:exp-e4}). We use the smaller model for prediction.
\end{itemize}

\paragraph{Extensions.} Binary and time-to-event responses require only a different likelihood in the meta-prior and a different prediction head. Test-point-specific selection and group-level inclusion for correlated blocks are natural next targets. On the theory side, the remaining open problems are spurious-inclusion control (assumption \Cref{ass:c6}) for nonparametric priors, an explicit second-order term in the detection threshold, and conditions under which training provably achieves uniform amortization.

\section*{AI Use Statement}
We used generative AI tools, including coding assistants, to write and debug experiment code, run and aggregate experiments, search and summarize related literature, assist with mathematical derivations, and draft and edit the manuscript. The research question, methods, mathematical formulations, and final claims were determined by the authors. All AI-assisted code was tested; reported results were verified against experiment logs and code; mathematical derivations and proofs were independently checked by the authors; citations were checked against their sources; and all text was reviewed by the authors. The authors take responsibility for the final content, including the text, claims, and artifacts produced with the aid of generative AI.

\bibliographystyle{apalike}
\bibliography{refs,refs_theory,refs_experiments}

\appendix
\crefalias{section}{appendix}\crefalias{subsection}{appendix}\crefalias{subsubsection}{appendix}
\section{Implementation details}\label{app:implementation}

\begin{table}[ht]
\centering
\caption{Meta-prior of SelectPFN. $\mathrm{LU}(a,b)$ is log-uniform on $[a,b]$. The ablation model of \Cref{sec:exp-e8} uses the same meta-prior without the interaction-only family, the weak-effect regime and interaction-only variables in the pairwise family, with family weights MLP 0.40, additive 0.15, pairwise 0.15, tree 0.10, GP 0.10 and linear 0.10.}\label{tab:metaprior}
\small
\begin{tabular}{@{}p{0.22\linewidth}p{0.74\linewidth}@{}}
\toprule
Component & Distribution \\
\midrule
Sample size & $N\sim\mathrm{LU}(50,1000)$ rows per task; a fraction $n/N\sim U(0.5,0.9)$ are training (context) rows, the remaining $m=N-n$ are query rows \\
Dimension & $p\sim\mathrm{LU}(5,p_{\max})$ with curriculum $p_{\max}=50\to200\to500$ in thirds of training \\
Sparsity & Beta--Bernoulli, $w\sim\mathrm{Beta}(\bar w\nu,(1-\bar w)\nu)$, $\E|S|=\bar w\,p\sim\mathrm{LU}(1,\min(20,p/3))$, $\nu\sim\mathrm{LU}(2,50)$; $|S|\le40$; an empty $\gamma$ is redrawn w.p.\ 0.9 \\
Dependence & independent; AR(1); block equicorrelation; latent factor (equal weights). Active variables at random positions \\
Marginals & Gaussian copula with per-column transform: identity, uniform, log-normal skew, Tukey-$h$ heavy tails ($h\sim U(0.1,0.4)$), discretized; columns standardized \\
Function family & sparse-input MLP with group spike-and-slab first layer, \cref{eq:gss} (0.30); additive (0.12); additive + pairwise, with no main effect for a variable w.p.\ 0.3 (0.12); trees/step functions (0.10); GP with ARD kernel, length-scales $\mathrm{LU}(0.3,2.5)$ (0.10); linear (0.14); interaction-only pairs/triples through functions odd in each argument (a single active variable gets a main effect), plus weak main effects w.p.\ 0.3 (0.12) \\
Effect sizes & w.p.\ 0.5, magnitudes $\sim\mathrm{LU}(0.1,1)$ (weak-effect regime); otherwise $U(0.3,1)$ \\
Effect-size check & permutation estimate of each active variable's total effect, $\tfrac12\E[(f(x)-f(x\text{ with }x_j\text{ permuted}))^2]/\Var f$, on 256 draws; redraw $f$ if any is below the threshold 0.01; after 4 failures freeze weak variables at 0 and relabel them inactive. GP draws skip the check \\
Noise & SNR $=\Var f/\sigma^2\sim\mathrm{LU}(0.25,20)$; Gaussian (0.6), $t_3$ (0.2), heteroscedastic (0.2) \\
Response & standardized with context-set statistics \\
\bottomrule
\end{tabular}
\end{table}

\paragraph{Function families.} Each family builds $f$ from the $s=|S|$ active covariates only, and is evaluated on standardized covariates. Several families use random one- and two-dimensional Fourier-feature functions,
\[
g_1(u)=\textstyle\sum_{k=1}^{8}a_k\cos(\xi_ku+\vartheta_k),\qquad g_2(u,v)=\textstyle\sum_{k=1}^{12}a_k\cos(\xi_{k1}u+\xi_{k2}v+\vartheta_k),
\]
with $a_k\sim N(0,1/8)$ for $g_1$ and $a_k\sim N(0,1/12)$ for $g_2$, $\vartheta_k\sim U(0,2\pi)$ and frequencies $\xi\sim N(0,c^2)$, where the bandwidth is $c\sim\mathrm{LU}(0.5,3)$ for $g_1$ and $c\sim\mathrm{LU}(0.5,2)$ for $g_2$; every component is centred and scaled to unit variance on a reference sample from the task's covariate law. The overall scale of $f$ does not matter, because the noise level is then set from $\Var f$ and the SNR. Effect magnitudes $\beta_j$ are drawn as in \Cref{tab:metaprior}, and $\mathrm{LU}(a,b)$ is log-uniform. Integer ranges below are uniform.
\begin{itemize}
\item \emph{Sparse-input network.} A fully connected network with 1--3 hidden layers of width 8--64. First-layer weights follow \cref{eq:gss}: rows of inactive inputs are zero and rows of active inputs are $N(0,\sigma_w^2I_H)$ with $\sigma_w^2=1/s$, $H$ the first hidden width; deeper weights are $N(0,1/\text{fan-in})$ and biases $N(0,0.25)$. Each hidden layer draws its activation from $\{\tanh,\mathrm{ReLU},\sin,\mathrm{GELU},\mathrm{softplus}\}$, and its pre-activations are standardized on the reference sample so that the network neither saturates nor collapses.
\item \emph{Additive.} $f=\sum_{j\in S}\beta_jg_{1,j}(x_j)$, a sum of independent random smooth functions, one per active variable.
\item \emph{Pairwise.} The additive part plus $1$ to $s$ randomly chosen pairs, each entering through either a random $g_2$ or a plain product $x_jx_k$ (probability $\tfrac12$ each), with weights $\mathrm{LU}(0.5,1.5)$. Each variable loses its main effect with probability 0.3 and is then given at least one interaction partner.
\item \emph{Interaction-only.} The active variables are split at random into pairs and occasional triples. A pair $(a,b)$ enters through a function that is odd in each argument, chosen from $x_ax_b$, $\sin(cx_ax_b)$, $\tanh(cx_a)\tanh(cx_b)$ and $\mathrm{sign}(x_a)\mathrm{sign}(x_b)$ with $c\sim\mathrm{LU}(0.5,2)$; a triple $(a,b,e)$ multiplies such a function of $(x_a,x_b)$ by $\tanh(x_e)$. (In rare cases a leftover fourth variable joins a triple without entering it; the effect-size check of \Cref{tab:metaprior} then removes it from the active set.) Group weights are $\mathrm{LU}(0.3,1)$. With probability 0.3 every variable also receives a weak main effect $g_1$ of weight $\mathrm{LU}(0.05,0.3)$. With symmetric covariates, a variable without a main effect then has no marginal association with $y$. If $s=1$, the single variable enters through $g_1$.
\item \emph{Trees and step functions.} A sum of trees of depth 1--4 each. The number of trees is drawn from 1--5, and further trees are added until every active variable has been used as a split variable; thresholds are values of the split variable in the reference sample, and the $2^r$ leaf values of a depth-$r$ tree are $N(0,1)$.
\item \emph{Gaussian process.} $f$ is drawn jointly at the task's rows from a zero-mean Gaussian process with unit variance and squared-exponential ARD kernel on the active variables, with length-scales $\mathrm{LU}(0.3,2.5)$ drawn independently for each variable.
\item \emph{Linear.} $f=\sum_{j\in S}\beta_jx_j$ with random signs.
\end{itemize}

\paragraph{Architecture.} Cell encoder: each standardized value is mapped to $\mathbb R^d$ by a small MLP; the response embedding is added to every cell of a training (context) row, and each row also has a dedicated target cell (the response embedding, or a learned missing-response token for query rows), read by the prediction head. Backbone: pre-LayerNorm transformer layers alternating feature attention within rows and row attention within features (rows attend to context rows only), MLP ratio 4. Main models: $d=128$, 6 layers, 4 heads (1.9M parameters), and $d=256$, 8 layers, 8 heads (9.5M parameters). Selection readout: mean and attention pooling over context rows, concatenated with the pooled response column and a global token carrying $(\log n,\log p)$, then an MLP. Prediction readout: 256-bin bar distribution with half-normal tails.

\paragraph{Training.} Objective \cref{eq:objective} with $\alpha=1$ and $(\rho_1,\rho_2)=(0.4,0.3)$; self-sampling after 10\% of training; self-sampled subsets capped at 64 columns. AdamW, learning rate $3\times10^{-4}$, $\beta=(0.9,0.95)$, weight decay 0.01, 2\% warm-up, cosine decay, gradient clipping at 1, bf16 autocast, gradient checkpointing for large tables. All tasks in a step share $(N,n,p)$, and the batch size adapts to keep about $3$--$4\times10^5$ cells per step. Tasks are generated on the fly by parallel CPU data-loader workers, so no task is seen twice. Main models: 80k steps (4.9--5.1 h, $d=128$) and 120k steps (17.8 h, width $d=256$ with 8 layers) on one RTX 4090. Knockoff models: 50--60k steps on $2p$ interleaved columns. GP-subset model (\Cref{sec:exp-e1}): $p\in\{6,\dots,12\}$, $n\in\{50,\dots,300\}$ training rows, 25k + 75k warm-started steps.

\paragraph{Relation to \Cref{ass:design}.} The generator departs from \Cref{ass:design} in minor ways. The function $f$ is rescaled to unit variance and its effect sizes are checked on reference rows drawn from the same task distribution $P_X$, so $f$ depends on the task-level parameters of $P_X$ (but not on the observed rows), and columns are standardized jointly over the observed and reference rows. The noise standard deviation is set from the sample variance of $f$ over the rows of the task and the drawn SNR, so the rows are not exactly conditionally independent given the covariates; since this variance depends on the covariates only through the active columns, the likelihood still depends on the data only through $D_S$ and $x_{*,S}$, and the conclusion of \Cref{prop:restricted} is unaffected. In addition, with label relabelling after freezing switched off (the default), heteroscedastic noise can depend on a variable that was frozen and labelled inactive, and a variable whose only interaction partner was frozen remains labelled active. They affect only the training targets, not the knockoff guarantee of \Cref{sec:knockoff}, which holds for any network weights.

\paragraph{Exactness tests.} Unit tests verify \cref{eq:invariance} (the prediction is bit-identical when excluded columns are changed), invariance to padding, and column-permutation equivariance of $q$.

\paragraph{Software.} The PyTorch implementation, task generators, trained weights and all evaluation scripts will be released publicly.

\section{Proofs}\label{app:proofs}

Every result in this appendix is proved analytically under its stated assumptions. The proofs are self-contained
except for the following published results, which are marked \emph{cited} where they are used:
\begin{itemize}
\item the model-X exchangeability lemma and the knockoff+ supermartingale argument \citep{candes2018panning,barber2015knockoff};
\item the existence of tests against convex Hellinger sets \citep{lecam1986,birge1983,ghosal2008nonparametric};
\item a group of facts about Gaussian processes, used only to verify the hypotheses of \Cref{thm:csharp} for
Gaussian-process priors \citep{borell1975,vdvvz2008rkhs,vdvvz2009adaptive,aronszajn1950,steinwart2008svm}.
\end{itemize}
Paragraphs headed \emph{Numerical illustration} report computations from the accompanying code. They are not part
of any proof.

\subsection{Proofs for Section~\ref{sec:method}}\label{app:proofs-method}

\begin{proof}[Proof of \Cref{prop:restricted}]
Fix $\gamma$ and write $S=S(\gamma)$.

\emph{Step 1 (average over $f$).} By the law of total probability, and because under \Cref{ass:design} the response $y_*$ is conditionally independent of the training responses given $(f,\gamma)$ and the covariates,
\[
p(y_*\mid x_*,\gamma,D)=\int p(y_*\mid x_*,f,\gamma)\,p(f\mid\gamma,D,x_*)\,df .
\]

\emph{Step 2 (the likelihood of $y_*$ uses only active coordinates).} Under \Cref{ass:design}, $y_*=f(x_{*,S})+\varepsilon$, so $p(y_*\mid x_*,f,\gamma)=p_\varepsilon(y_*-f(x_{*,S}))$, which does not involve $x_{*,-S}$.

\emph{Step 3 (the posterior of $f$ uses only active columns).} By Bayes' rule,
\[
p(f\mid\gamma,D,x_*)\ \propto\ \Pi_f(f\mid\gamma)\;p(x_1,\dots,x_n,x_*)\;\prod_{i=1}^np_\varepsilon\big(y_i-f(x_{i,S})\big).
\]
By \Cref{ass:design}, the covariates are independent of $(\gamma,f)$ and $\Pi_f(\cdot\mid\gamma)$ does not depend on them, so the factor $p(x_1,\dots,x_n,x_*)$ does not involve $f$ and cancels on normalization. The remaining factors involve the data only through $y_i$ and $x_{i,S}$. Hence $p(f\mid\gamma,D,x_*)=p(f\mid\gamma,D_S)$.

\emph{Conclusion.} Substituting Steps 2 and 3 into Step 1,
\[
p(y_*\mid x_*,\gamma,D)=\int p_\varepsilon\big(y_*-f(x_{*,S})\big)\,p(f\mid\gamma,D_S)\,df=p(y_*\mid x_{*,S},D_S,\gamma),
\]
which proves the first statement. The second follows from $p(y_*\mid x_*,D)=\sum_\gamma p(\gamma\mid D)\,p(y_*\mid x_*,\gamma,D)$.
\end{proof}

\begin{proof}[Proof of \Cref{prop:minimizers}]
\emph{Step 1 (pointwise minimization).} Each loss term has the form
\[
\E[\ell(a(Z),T)]=\E\big[\E[\ell(a(Z),T)\mid Z]\big],
\]
where $Z$ is the input of a head, $T$ its target and $a(Z)$ its output. Since each head can represent any measurable function of its input, the expected loss is minimized by minimizing the conditional expectation separately for every value $z$ of $Z$.

\emph{Step 2 (selection head).} Here $Z=D$ and $T=\gamma_j$. For fixed $D$,
\[
\E\big[-\gamma_j\log q-(1-\gamma_j)\log(1-q)\,\big|\,D\big]=\operatorname{Ent}(\pi_j)+\operatorname{kl}(\pi_j\,\|\,q),
\]
with $\pi_j=\E[\gamma_j\mid D]$. Since $\operatorname{kl}\ge0$ with equality only at $q=\pi_j$, the minimizer is $q_j(D)=\pi_j(D)$, which proves (a).

\emph{Step 3 (prediction head).} Here $T=y_*$. For fixed $z$, $\E[-\log g(y_*)\mid Z=z]=\operatorname{Ent}\big(p(\cdot\mid z)\big)+\KL\big(p(\cdot\mid z)\,\|\,g\big)$, which is finite by assumption and is minimized uniquely by $g=p(y_*\mid Z=z)$. For the bar-distribution head, the same argument applied to the bin of $y_*$ shows that the minimizing bin probabilities are those of $p(y_*\mid Z=z)$.

\emph{Step 4 (identifying the input).} Under teacher forcing, $Z=(x_{*,S},D_S)$ with $S=S(\gamma)$. The prediction pass sees these columns without their original indices; by the exchangeability condition, the conditional law of $y_*$ given the columns does not depend on which indices they carry, so the minimizer of Step 3 is $p(y_*\mid x_{*,S},D_S,S(\gamma)=S)$. By \Cref{prop:restricted} this equals $p(y_*\mid x_*,\gamma,D)$, which proves (b). For a general kernel, $Z=(x_{*,\tilde S},D_{\tilde S},\tilde S)$; because the kernel does not depend on $\theta$, the joint law of $(Z,y_*)$ is fixed during minimization, and Step 3 gives the minimizer $p(y_*\mid Z)$, which by construction depends only on the columns in $\tilde S$. This proves (c).
\end{proof}

\subsection{Two tools: an evidence lower bound and tests}\label{app:proofs-tools}

Consider $Y=f(X)+\varepsilon$, $\varepsilon\sim N(0,\sigma^2)$, with $\sigma$ and $P_X$ fixed. Write $p_f$ for the density of
one observation $(X,Y)$, $\Lambda_f=\prod_{i=1}^np_f(X_i,Y_i)/p_{f_0}(X_i,Y_i)$, $\Delta_f=f_0-f$, and $\|\cdot\|$ for the
$L_2(P_X)$ norm. $P_f^n$ is the law of $n$ i.i.d.\ observations under $f$, $P_0^n=P_{f_0}^n$, and $\E_0,\Var_0$ denote expectation and variance under $P_0^n$.

\begin{lemma}[Evidence lower bound]\label{lem:evidence}
Let $B$ be a set of functions with $\KL(p_{f_0}\|p_f)\le\varrho^2$ and
$\mathcal V(p_{f_0}\|p_f):=\Var_{f_0}\log(p_f/p_{f_0})\le\varrho^2$ for all $f\in B$, and $\Pi(B)>0$; we call $\varrho$ the radius of $B$. Then for every $C>0$,
\[
P_0^n\Big(\int\Lambda_f\,d\Pi<\Pi(B)\,e^{-(1+C)n\varrho^2}\Big)\le\frac1{C^2n\varrho^2}.
\]
\end{lemma}
\begin{proof}
Let $\Pi_B=\Pi(\cdot\cap B)/\Pi(B)$ and $Z=\int\log\Lambda_f\,d\Pi_B$. By Jensen,
$\int\Lambda_f\,d\Pi\ge\Pi(B)\int\Lambda_f\,d\Pi_B\ge\Pi(B)e^Z$. The observations are i.i.d., so by Fubini
$\E_0Z=-n\int\KL\,d\Pi_B\ge-n\varrho^2$. The variance of an average is at most the average of the variances
(Cauchy--Schwarz), so $\Var_0Z\le\int\Var_0(\log\Lambda_f)\,d\Pi_B=n\int\mathcal V\,d\Pi_B\le n\varrho^2$. By Chebyshev,
$P_0^n(Z<-(1+C)n\varrho^2)\le P_0^n(Z-\E_0Z<-Cn\varrho^2)\le n\varrho^2/(Cn\varrho^2)^2$.
\end{proof}
This is the centred-variance form of \citet[Lemma 8.1]{ghosal2000}. For Gaussian regression,
$\log(p_f/p_{f_0})=-(2\varepsilon \Delta_f(X)+\Delta_f(X)^2)/(2\sigma^2)$. Hence $\KL=\|\Delta_f\|^2/(2\sigma^2)$ and
\[
\mathcal V=\frac{4\sigma^2\|\Delta_f\|^2+\Var(\Delta_f(X)^2)}{4\sigma^4}\le\frac{\|\Delta_f\|^2}{\sigma^2}+\frac{\E \Delta_f(X)^4}{4\sigma^4},
\]
since the cross term vanishes because $\E[\varepsilon\mid X]=0$.

\begin{lemma}[Tests]\label{lem:tests}
Let $\mathcal F$ be a set of functions with $\|f-f_0\|_\infty\le b$, and let $c_b=(1-e^{-u})/(4u)$ with
$u=b^2/(2\sigma^2)$.
\begin{enumerate}
\item[(a)] The squared Hellinger distance $H^2(p_f,p_{f_0})=\int(\sqrt{p_f}-\sqrt{p_{f_0}})^2$ satisfies
$H^2\ge c_b\|f-f_0\|^2/\sigma^2$ for $f\in\mathcal F$.
\item[(b)] (Cited.) With $c_{\rm T}=9/32$, for every $r>0$ there are tests $\varphi_n\in\{0,1\}$ with
\[
P_0^n\varphi_n\le2\,\mathcal N_n(r)\,e^{-c_{\rm T}c_bnr^2/\sigma^2},\qquad
\sup_{f\in\mathcal F:\|f-f_0\|\ge r}P_f^n(1-\varphi_n)\le2\,e^{-c_{\rm T}c_bnr^2/\sigma^2}.
\]
Here $\mathcal N_n(r)$ is the number of Hellinger balls of radius $\sqrt{c_b}\,r/(4\sigma)$, with centres in the set,
needed to cover $\{f\in\mathcal F:\|f-f_0\|\ge r\}$.
\end{enumerate}
\end{lemma}
\begin{proof}
(a) $H^2=2(1-\E e^{-\Delta_f(X)^2/(8\sigma^2)})$. Since $\Delta_f^2/(8\sigma^2)\le u$, and $1-e^{-v}\ge v(1-e^{-u})/u$ on $[0,u]$ by
concavity, $H^2\ge2\frac{1-e^{-u}}u\cdot\frac{\|\Delta_f\|^2}{8\sigma^2}=c_b\|\Delta_f\|^2/\sigma^2$.

(b) By (a), $\|f-f_0\|\ge r$ implies $H(p_f,p_{f_0})\ge\sqrt{c_b}\,r/\sigma$. For each cover centre $f_i$, the
Hellinger ball of radius $H(p_{f_i},p_{f_0})/4$ around $p_{f_i}$, taken over all densities, is convex, because
$\tilde p\mapsto\int\sqrt{\tilde pp_{f_i}}$ is concave. By the triangle inequality its distance from $p_{f_0}$ is at least
$\frac34H(p_{f_i},p_{f_0})\ge\frac34\sqrt{c_b}\,r/\sigma$. For a convex set $\mathcal P$ of densities, there is a test with both
error probabilities at most $\exp(-\frac12nH^2(p_{f_0},\mathcal P))$ (cited: \citealp{lecam1986,birge1983};
\citealp[Lemma 6.1]{ghosal2008nonparametric}). This gives the exponent $\frac12\cdot\frac9{16}c_bnr^2/\sigma^2$, that is
$c_{\rm T}=9/32$. Cover balls of radius $\sqrt{c_b}\,r/(4\sigma)$ lie inside these balls. Take the maximum of the tests over the
cover. Converting a randomized test $\tilde\varphi$ into $\ind\{\tilde\varphi\ge\frac12\}$ costs a factor 2, by Markov.
\end{proof}

\subsection{Proofs for Section~\ref{sec:theory-regret}}\label{app:proofs-regret}

\begin{proof}[Proof of \Cref{lem:decomp}]
Given $D$, $\gamma_j\sim\Bern(\pi_j)$, so $\E[\ell(q_j,\gamma_j)\mid D]=\operatorname{Ent}(\pi_j)+\operatorname{kl}(\pi_j\|q_j)$. Average over
$j$ and $D$. The decomposition over sizes is the law of total expectation.
\end{proof}

\begin{proof}[Proof of \Cref{prop:regret}]
(i) Given $D$, the two rules differ only on $\mathcal E_j$, where the posterior expected Hamming loss differs by
$|2\pi_j-1|$. On $\mathcal E_j$, $q_j$ and $\pi_j$ straddle $\frac12$, so $|m_j|\le|e_j|$.

(ii) On $\mathcal E_j$, $|2\pi_j-1|=2|m_j|\le2|e_j|$. Then use Pinsker's inequality $|e_j|\le\sqrt{\operatorname{kl}(\pi_j\|q_j)/2}$
and Jensen.

(iii) Split on $\{|m|\le t\}$. The first part is at most $\bar\delta\sqrt{C_{\rm m}t^{\alpha_{\rm m}}}$ by Cauchy--Schwarz. The second is at
most $\bar\delta^2/t$, since on $\{|m|>t\}\cap\{|m|\le|e|\}$ we have $|e|\le e^2/t$. Balancing at
$t^*=(\bar\delta/\sqrt{C_{\rm m}})^{2/(2+\alpha_{\rm m})}$ gives the bound. If $t^*>\frac12$ the bound equals
$4C_{\rm m}t^{*1+\alpha_{\rm m}}>2C_{\rm m}2^{-\alpha_{\rm m}}\ge2$ and is trivially true. Taking $t=\frac12$ in the margin condition gives
$C_{\rm m}\ge2^{\alpha_{\rm m}}$.

(iv) $\Prob_J(\mathcal E_J)\le\Prob_J(|m_J|\le t)+\Prob_J(|e_J|>t)$ and Chebyshev.
\end{proof}

\paragraph{Cost-weighted variant.}
With false-inclusion cost $c$ and false-exclusion cost 1, the Bayes threshold is $c/(1+c)$. Set
$m_j=\pi_j-c/(1+c)$. The regret weight on $\mathcal E_j$ is $|\pi_j-c(1-\pi_j)|=(1+c)|m_j|$, so in (i)--(ii) the factor
2 becomes $1+c$. We do not claim (iii)--(iv) for $c\ne1$.

\paragraph{Numerical illustration (not part of the proof).}
On the exact-posterior tasks of \Cref{sec:exp-e1} (GP-subset prior; 17{,}961 variables pooled over tasks):
\begin{itemize}
\item The regret (i) is $0.0222$ for the 25k-step model and $0.0154$ for the 100k-step model.
\item The refined bound (ii) is $0.041$ and $0.031$; $2\E|e|$ is $0.094$ and $0.074$; and $\sqrt{2\Delta}$ is $0.290$ and $0.242$.
\item Every disagreement satisfies $|m_j|\le|e_j|$.
\item Near-ties ($|m|<0.1$) make up 4.9\% of variables but 36\% and 44\% of disagreements.
\item The mean $|m|$ on disagreements is $0.22$ and $0.18$. Disagreements are enriched at near-ties but not
confined to them.
\end{itemize}
Pooling variables gives a $p$-weighted average, whereas $R$ and $\Delta$ weight by $1/p$ per task; every inequality
holds under either weighting.

\paragraph{Model averaging.}
The BMA samples $S\sim\prod_j\Bern(q_j)$, a product of marginals even when $q=\pi$. In addition, the prediction pass
targets the predictive under the training kernel, not the exact restricted predictive. A regret statement for the BMA
would have to address both points, and we make none.

\subsection{Proofs for Section~\ref{sec:theory-knockoff}}\label{app:proofs-knockoff}

\paragraph{Setting.}
Rows are i.i.d.\ and $\mathcal H_0=\{j:Y\indep X_j\mid X_{-j}\}$. $\tilde X$ is an exact model-X knockoff:
$[X,\tilde X]_{\mathrm{swap}(S)}\overset d=[X,\tilde X]$ for all $S$, and $\tilde X\indep Y\mid X$. Let
$Z=[X,\tilde X]\in\mathbb R^{n\times2p}$, and let $\mathrm{sw}_S$ be the permutation of $\{1,\dots,2p\}$ that exchanges $j$ and
$j+p$ for $j\in S$.

\paragraph{Procedure.}
\begin{enumerate}
\item Draw $\Sigma$ uniformly from the permutations of $\{1,\dots,2p\}$, independently of $(Z,y)$.
\item Compute $F(Z_{\cdot,\Sigma},y,U)\in\mathbb R^{2p}$. $F$ is the trained selection pass, including the
standardization of $y$, and the input carries no marker of which columns are knockoffs. $U$ is any internal
randomness, independent of $(Z,y,\Sigma)$, with a law that does not depend on them.
\item Map back: $Q_{\Sigma(k)}=F(Z_{\cdot,\Sigma},y,U)_k$.
\item Set $W_j=Q_j-Q_{j+p}$, and write $W=G(Z,y,\Sigma,U)$.
\item Select $\{j:W_j\ge T\}$ with the knockoff+ threshold at level $\alpha_{\FDR}$,
$T=\min\{t\in\{|W_j|:W_j\ne0\}:(1+\#\{j:W_j\le-t\})/\max(1,\#\{j:W_j\ge t\})\le\alpha_{\FDR}\}$ ($T=\infty$ if the set is empty). Zero statistics are never selected, and ties among
nonzero $|W_j|$ are broken by an independent uniform order.
\end{enumerate}

\begin{lemma}[Pointwise swap identity]\label{lem:swap}
For all $Z,y,S,u$ and permutations $\sigma$,
$G(Z_{\cdot,\mathrm{sw}_S},y,\mathrm{sw}_S\circ\sigma,u)=\mathrm{flip}_S\,G(Z,y,\sigma,u)$.
\end{lemma}
\begin{proof}
Let $Z'=Z_{\cdot,\mathrm{sw}_S}$. Since $\mathrm{sw}_S$ is an involution, $Z'_{\cdot,\mathrm{sw}_S\sigma(k)}=Z_{\cdot,\sigma(k)}$. So $F$
receives an identical input with the same $u$, and returns an identical output $r$. Then
$Q'_{\mathrm{sw}_S\sigma(k)}=r_k=Q_{\sigma(k)}$, i.e.\ $Q'=Q\circ\mathrm{sw}_S$. Hence $W'_j=-W_j$ for $j\in S$ and $W'_j=W_j$ otherwise.
\end{proof}

\begin{proof}[Proof of \Cref{thm:knockoff}]
(i) Fix $S\subseteq\mathcal H_0$. By the exchangeability lemma \citep[Lemma 3.2]{candes2018panning} (cited),
$(Z_{\cdot,\mathrm{sw}_S},y)\overset d=(Z,y)$. Because $(\Sigma,U)$ is independent of $(Z,y)$,
$(Z_{\cdot,\mathrm{sw}_S},y,\Sigma,U)\overset d=(Z,y,\Sigma,U)$. Because $\Sigma$ is uniform and independent of $(Z,y,U)$,
$(Z,y,\mathrm{sw}_S\circ\Sigma,U)\overset d=(Z,y,\Sigma,U)$. Hence
\begin{align*}
(G(Z,y,\Sigma,U),y)&\overset d=(G(Z_{\cdot,\mathrm{sw}_S},y,\mathrm{sw}_S\mathrm{sw}_S\Sigma,U),y)\\
&=(\mathrm{flip}_SG(Z,y,\mathrm{sw}_S\Sigma,U),y)\overset d=(\mathrm{flip}_SG(Z,y,\Sigma,U),y),
\end{align*}
where the middle equality is \Cref{lem:swap}.

(ii) Let $\xi\in\{\pm1\}^{\mathcal H_0}$ be uniform and independent of everything else, and let $\xi\odot W$
flip the null coordinates with $\xi_j=-1$. Applying (i) for each fixed $\xi$ and mixing gives
$W\overset d=\xi\odot W$. In $\xi\odot W$, the signs of the nonzero null coordinates are i.i.d.\ uniform and
independent of $|W|$ and of the non-null coordinates. Equality in distribution transfers this to $W$.

(iii) Given (ii), the supermartingale argument of \citet{barber2015knockoff}, as used in
\citet[Theorem 3.4]{candes2018panning}, applies verbatim (cited).
\end{proof}

\paragraph{Remarks.}
\begin{enumerate}
\item \emph{Fresh, data-independent randomization.} The guarantee averages over $\Sigma$, so $\Sigma$ must not depend on
the data and must be redrawn for each analysis. In our experiments the permutation seed is the replicate index. That
is independent of the data, which are seeded separately, and it varies across replicates.
\item \emph{Pair swaps suffice.} Draw $V\subseteq\{1,\dots,p\}$, with each $j\in V$ independently with probability
$\frac12$, feed $Z_{\cdot,\mathrm{sw}_V}$, and map back by $\mathrm{sw}_V$; write $G'(Z,y,V)$ for the resulting statistic. These swaps commute and $\mathrm{sw}_S\mathrm{sw}_{S\triangle V}=\mathrm{sw}_V$, so
$G'(Z_{\cdot,\mathrm{sw}_S},y,S\triangle V)=\mathrm{flip}_SG'(Z,y,V)$ pointwise. Also $S\triangle V$ has the same law as $V$,
independently of $(Z,y)$. The proof of (i) goes through unchanged.
\item \emph{Derandomization.} Averaging $Q$ over $K$ i.i.d.\ uniform interleavings, with internal randomness $U_k$ i.i.d.\ and independent of $(Z,y,\Sigma_{1:K})$, preserves (i). \Cref{lem:swap} holds
for each interleaving simultaneously, and $(\mathrm{sw}_S\Sigma_k,U_k)_{k\le K}\overset d=(\Sigma_k,U_k)_{k\le K}$ jointly and
independently of $(Z,y)$. See also derandomized knockoffs \citep{ren2023derandomizing}.
\item \emph{Approximate knockoffs.} With estimated second-order knockoffs, control is approximate
\citep{barber2020robust}.
\end{enumerate}

\paragraph{Numerical illustration (not part of the proof).}
On the trained knockoff network (bf16; 20 tables with $p\in[10,100]$):
\begin{itemize}
\item the forward pass is bit-reproducible, and under the compensating permutation of \Cref{lem:swap},
$\max|W'-\mathrm{flip}_SW|=0$;
\item without interleaving, relying on equivariance alone, the deviation from an exact sign flip is
$7.0\times10^{-3}$.
\end{itemize}

\subsection{Full statements and proofs for Section~\ref{sec:theory-detect}}\label{app:proofs-detect}

\paragraph{Standing assumptions.}
In addition to \Cref{ass:regression}, the results below use some of the following conditions.
\begin{assumption}[Labels and sieve]\label{ass:labels}
\begin{enumerate}
\item[(a)] \emph{Labels to functions:} $\{\gamma_j=0\}\subseteq\{f\in N_j\}$ $\Pi$-a.s.
\item[(b)] \emph{Functions to labels:} $\{f\in N_j\}\subseteq\{\gamma_j=0\}$ $\Pi$-a.s.
\item[(c)] A sieve $\mathcal F_n$ with $\|f-f_0\|_\infty\le b$ on $\mathcal F_n$.
\end{enumerate}
\end{assumption}
The moment condition (M4) is: $\E_Xf(X)^4<\infty$ for $\Pi_{\omega_0}$-almost every $f$, with prior draws in $L_2(P_X)$, where $\Pi_{\omega_0}$ is the reference prior of \Cref{sec:theory-detect}. The structured setup of \Cref{sec:theory-detect} implies \Cref{ass:labels}(a) and~(b). Each result below uses only the conditions
it names. \Cref{prop:detect}(a)--(b) and \Cref{lem:ceiling} need only moment conditions (at most fourth moments), so unbounded truths and prior draws are allowed.

\paragraph{Gap between this setting and the implementation.}
The implementation differs from this setting in several ways:
\begin{itemize}
\item $\sigma$ is unknown, $y$ is standardized, and $P_X$ is random;
\item 40\% of tasks have non-Gaussian noise;
\item Gaussian-process and MLP prior draws are unbounded.
\end{itemize}
Standardization replaces $f_0$ by $f_0/\mathrm{sd}(y)$ and $\sigma^2$ by $1/(1+\mathrm{SNR})$, where $\mathrm{sd}(y)=(\Var f_0(X)+\sigma^2)^{1/2}$ and $\mathrm{SNR}=\Var f_0(X)/\sigma^2$. This leaves $t_j$ invariant
but changes the prior mass around the truth. In the generator, \Cref{ass:labels}(a) holds exactly. \Cref{ass:labels}(b) holds except on an event
$E_4$ of small prior probability, created by a fallback that freezes weak variables without re-checking their
partners. A flag in the generator removes this event, and Remark~\ref{rem:c4} shows how its probability
$\varepsilon_4$ enters the bounds. \emph{Numerical illustration:} in about 114{,}000 simulated non-GP tasks, one
violation was found, giving $\varepsilon_4\approx5\times10^{-6}$ (approximate 95\% upper bound $1.5\times10^{-4}$).

\begin{proof}[Proof of \Cref{prop:detect}(a)]
Let $E=\{j\in\hat S\}$, and let $P_{h_j^*}^n$ be the law of $n$ observations with regression function $h_j^*$.
$|P_0^n(E)-P_{h_j^*}^n(E)|\le\mathrm{TV}\le\sqrt{\KL/2}$ by Pinsker, and
$\KL(P_0^n\|P_{h_j^*}^n)=n\|f_0-h_j^*\|^2/(2\sigma^2)=t_j/2$.
\end{proof}

\begin{lemma}[Sparsity-prior ceiling]\label{lem:ceiling}
Assume \Cref{ass:regression}, and let $B_j,\bar k_j,\bar v_j$ be as in \Cref{prop:detect}(b) with $\Pi(B_j)>0$. For every $c_{\rm M}>1$ and
$s>0$, with $P_0^n$-probability at least $1-1/c_{\rm M}-1/s^2$,
\[
\log\frac{\Pi(N_j^c\mid D_n)}{\Pi(N_j\mid D_n)}\le\log c_{\rm M}+\log\frac{\Pi(N_j^c)}{\Pi(B_j)}+n\bar k_j+s\sqrt{n\bar v_j}.
\]
Under \Cref{ass:labels}(b), $\Pi(\gamma_j=1\mid D)\le\Pi(N_j^c\mid D)$, so $x_j$ is missed whenever the right-hand side is negative.
\end{lemma}
\begin{proof}
The posterior odds equal $\int_{N_j^c}\Lambda_f\,d\Pi\big/\int_{N_j}\Lambda_f\,d\Pi$.
\emph{Numerator:} by Fubini, $\E_0\int_{N_j^c}\Lambda_f\,d\Pi\le\Pi(N_j^c)$, so by Markov the numerator is at most
$c_{\rm M}\Pi(N_j^c)$ with probability at least $1-1/c_{\rm M}$.
\emph{Denominator:} $\int_{N_j}\ge\int_{B_j}$, and Jensen with $\Pi_B=\Pi(\cdot\cap B_j)/\Pi(B_j)$ gives
$\log\int_{B_j}\Lambda_f\,d\Pi\ge\log\Pi(B_j)+Z_n$ with $Z_n=\int\log\Lambda_f\,d\Pi_B$. For $f\in B_j$,
$\|\Delta_f\|\le\|f_0-h_j^*\|+\|h_j^*-f\|\le\kappa_j+\eta$ and $\E\Delta_f^4\le m_4$. By the moment formulas after \Cref{lem:evidence},
$\E_0Z_n\ge-n\bar k_j$ and $\Var_0Z_n\le n\bar v_j$. Chebyshev gives $Z_n\ge-n\bar k_j-s\sqrt{n\bar v_j}$ with
probability at least $1-1/s^2$. Finish with a union bound.
\end{proof}
$\Pi(B_j)>0$ holds whenever $h_j^*$ lies in the $L_2(P_X)$-support of $\Pi$ restricted to $N_j$ and prior draws have
$\E(f_0-f)^4(X)<\infty$ almost surely. This follows for all large $m_4$, by continuity of measure.

\begin{remark}[When \Cref{ass:labels}(b) fails on a small event]\label{rem:c4}
Let \Cref{ass:labels}(b) hold outside an event $E_4$ with $\Pi(E_4)=\varepsilon_4$. Then:
\begin{itemize}
\item \Cref{lem:ceiling} holds with $\Pi(N_j^c)$ replaced by $\Pi(N_j^c)+\varepsilon_4$ and $B_j$ by $B_j\setminus E_4$;
\item \Cref{prop:detect}(b) holds with $\omega_j^I$ replaced by $\omega_j^I+\varepsilon_4$ and $B_j$ by $B_j\setminus E_4$, provided
\Cref{ass:c6} is restated with denominator $\int_{N_j\setminus E_4}\Lambda_f\,d\Pi$.
\end{itemize}
\emph{Proof.} Outside $E_4$, $\{\gamma_j=1\}\subseteq\{f\notin N_j\}$ and $\{f\in N_j\}\subseteq\{\gamma_j=0\}$. By Fubini,
$\E_0\int_{E_4}\Lambda_f\,d\Pi=\varepsilon_4$, and Markov applies to the enlarged numerator. The denominators are bounded
below over $B_j\setminus E_4$ as before.
\end{remark}

\begin{lemma}[Prior-mass detection bound]\label{lem:detection}
Assume \Cref{ass:regression} and \Cref{ass:labels}(a),(c). Let $B$ be a set as in \Cref{lem:evidence} with $\varrho=\bar\varepsilon$, let
$\ell_\Pi(\bar\varepsilon)=-\log\Pi(B)$, and let $\mathcal N_n$ be the covering number of \Cref{lem:tests}(b) at $r=\kappa_j$ on
$\mathcal F_n\cap N_j$. For every $C>0$,
\[
\E_0[1-\pi_j(D_n)]\le2\mathcal N_ne^{-c_{\rm T}c_bt_j}+\frac1{C^2n\bar\varepsilon^2}
+e^{(1+C)n\bar\varepsilon^2+\ell_\Pi(\bar\varepsilon)}\big(2e^{-c_{\rm T}c_bt_j}+\Pi(\mathcal F_n^c)\big).
\]
\end{lemma}
\begin{proof}
By \Cref{ass:labels}(a), $1-\pi_j\le\Pi(N_j\mid D)$. Every $f\in N_j$ satisfies $\|f-f_0\|\ge\kappa_j$. Let $\varphi_n$ be the test of
\Cref{lem:tests}(b) for $\mathcal F_n\cap N_j$ at $r=\kappa_j$. Since $n\kappa_j^2/\sigma^2=t_j$, its errors are
$P_0^n\varphi_n\le2\mathcal N_ne^{-c_{\rm T}c_bt_j}$ and $\sup_{f\in\mathcal F_n\cap N_j}P_f^n(1-\varphi_n)\le2e^{-c_{\rm T}c_bt_j}$. Let $\mathcal G$ be the
event of \Cref{lem:evidence}. Then
$\Pi(N_j\mid D)\le\varphi_n+\ind_{\mathcal G^c}+(1-\varphi_n)\ind_{\mathcal G}\Pi(N_j\mid D)$. On $\mathcal G$ the denominator is at least
$\Pi(B)e^{-(1+C)n\bar\varepsilon^2}$. By Fubini, the numerator satisfies
$\E_0[(1-\varphi_n)\int_{N_j}\Lambda_f\,d\Pi]\le\sup_{f\in\mathcal F_n\cap N_j}P_f^n(1-\varphi_n)+\Pi(\mathcal F_n^c)$. Take expectations.
\end{proof}
\emph{Matching bound.} If $\ell_\Pi^I(\bar\varepsilon):=-\log\sum_{k\in\mathcal K_j^I}(w_k/\omega_j^I)\Pi_k(B)<\infty$, then
$\Pi(B)\ge \omega_j^Ie^{-\ell_\Pi^I(\bar\varepsilon)}$. So $\ell_\Pi(\bar\varepsilon)\le\log(1/\omega_j^I)+\ell_\Pi^I(\bar\varepsilon)$, and $x_j$ is found once
$c_{\rm T}c_bt_j$ exceeds $\log(1/\omega_j^I)+\ell_\Pi^I(\bar\varepsilon)+(1+C)n\bar\varepsilon^2+\log\mathcal N_n$ by a margin.

\paragraph{Which partitions can satisfy \Cref{ass:c6}.}
If some $k\in\mathcal K_j$ has $\inf_{f\in\operatorname{supp}\Pi_k}\|f_0-f\|<\kappa_j$, its Bayes factor against $N_j$ grows
exponentially in $n$ (cited: \citealp[Theorem 2]{chatterjee2018bayesfactor}). In that theorem the two models are $\Pi_k$ and $\Pi$ restricted to $N_j$, and the divergence rate of $f$ is $\KL(p_{f_0}\|p_f)=\|f_0-f\|^2/(2\sigma^2)$; under its regularity conditions on both priors, which we do not verify here, $n^{-1}$ times the log Bayes factor converges almost surely to the difference of the smallest rates of the two priors, which is at least $(\kappa_j^2-\inf_{f\in\operatorname{supp}\Pi_k}\|f_0-f\|^2)/(2\sigma^2)>0$. Such a structure must be placed in
$\mathcal K_j^I$. This happens for additive-in-$x_j$ structures whenever $\E[f_0\mid X_V]$ depends on $x_j$, for example
under correlated $P_X$. For $f_0=\sin(x_1x_2)+0.8x_3$ with $X\sim N(0,I)$, every additive-in-$x_1$ structure and every
wrong-partner pair $(1,l)$, $l\ne2$, has projection $h_1^*$, and so is zero-gain.

\begin{proof}[Proof of \Cref{prop:detect}(b)]
Each $\Pi_k$ with $k\in\mathcal K_j$ gives zero mass to $N_j$, and each $\Pi_k$ with $k\notin\mathcal K_j$ is carried by
$N_j$. Hence
\[
\frac{\pi_j}{1-\pi_j}=\frac{\sum_{k\in\mathcal K_j^S}w_k\Lambda(k)}{\int_{N_j}\Lambda_f\,d\Pi}+\frac{\sum_{k\in\mathcal K_j^I}w_k\Lambda(k)}{\int_{N_j}\Lambda_f\,d\Pi}.
\]
The first term is at most $\tau_n$ on the event of \Cref{ass:c6}. By Fubini, $\E_0\Lambda(k)=1$ for every $k$, so the second
numerator has mean $\omega_j^I$ and is at most $c_{\rm M}\omega_j^I$ with probability at least $1-1/c_{\rm M}$. As in \Cref{lem:ceiling},
$\int_{N_j}\Lambda_f\,d\Pi\ge\Pi(B_j)e^{-n\bar k_j-s\sqrt{n\bar v_j}}$ with probability at least $1-1/s^2$. Finish with a
union bound.
\end{proof}

\begin{proof}[Proof of \Cref{cor:prior-weight}]
\emph{Invariances.} Under $\Pi_\omega$, every structure outside $\mathcal K_j^I$ has $(1-\omega)/(1-\omega_0)$ times its weight under
$\Pi_{\omega_0}$. Hence $\Pi_\omega(B_j)\ge\Pi_{\omega_0}(B_j)$ for $\omega\le \omega_0$. The ratio in \Cref{ass:c6} does not depend on $\omega$, because its
numerator and denominator scale by the same factor, so $\tau=\sup_n\tau_n$ and $\delta_{\rm sp}=\sup_n\delta_n$, with $\tau_n,\delta_n$ as in \Cref{ass:c6}, do not depend on $\omega$.

(i) Fix $\eta>0$; then $\Pi_{\omega_0}(B_j)>0$ for large $m_4$ by the support argument after \Cref{lem:ceiling}. Choose $c_{\rm M}$
and $s$ with $1/c_{\rm M}+1/s^2+\delta_{\rm sp}<\frac12$. Write $n\bar k_j=\frac t2(1+\eta/\kappa_j)^2$ and $s\sqrt{n\bar v_j}=sa\sqrt t$
with $a=\sqrt{\bar v_j\sigma^2/\kappa_j^2}$, and let $c_0=\log c_{\rm M}-\log\Pi_{\omega_0}(B_j)-\log(1-\tau)$. By
\Cref{prop:detect}(b), every $n$ whose $t$ satisfies $\frac t2(1+\eta/\kappa_j)^2+sa\sqrt t<\log(1/\omega)-c_0$ has
$P_0^n(\pi_j>\frac12)<\frac12$. So $t_{50}(\omega)\ge T(\omega)$, the positive root of the corresponding equation. Since
$T(\omega)/\log(1/\omega)\to2/(1+\eta/\kappa_j)^2$, and neither $t_{50}$ nor \Cref{ass:c6} depends on $\eta$, let $\eta\to0$.

(ii) Precisely, assume for $\Pi_{\omega_0}$: (U1) $\Pi(\mathcal F^c)=0$ for a set $\mathcal F$ with $\|f-f_0\|_\infty\le b$;
(U2) the covering number $\mathcal N$ of \Cref{lem:tests}(b) at $r=\kappa_j$ on $\mathcal F\cap N_j$ is finite; (U3)
$\ell_\Pi^I(\bar\varepsilon)<\infty$ for every $\bar\varepsilon>0$, with $B$ the maximal set of \Cref{lem:evidence}. These hold for
every $\Pi_\omega$, since all $\Pi_\omega$ share support and within-group proportions. Fix $C>0$ and $\bar\varepsilon>0$ with
$\vartheta:=c_{\rm T}c_b-(1+C)\bar\varepsilon^2\sigma^2/\kappa_j^2>0$. With $n\bar\varepsilon^2=t\bar\varepsilon^2\sigma^2/\kappa_j^2$, \Cref{lem:detection} and
the matching bound give
\[
\E_0[1-\pi_j]\le2\mathcal Ne^{-c_{\rm T}c_bt}+\frac{\kappa_j^2}{C^2\bar\varepsilon^2\sigma^2t}+2\exp\big(-\vartheta t+\log(1/\omega)+\ell_\Pi^I(\bar\varepsilon)\big).
\]
By Markov, $P_0^n(\pi_j\le\frac12)\le2\E_0[1-\pi_j]$. Each term is at most $\frac1{12}$ once
$t\ge\max\{t_1,t_2,(\log(1/\omega)+\ell_\Pi^I(\bar\varepsilon)+\log24)/\vartheta\}$, with $t_1,t_2$ independent of $\omega$. Consecutive values
of $t$ differ by $\kappa_j^2/\sigma^2$. Hence $\limsup t_{50}/\log(1/\omega)\le1/\vartheta$, and $\bar\varepsilon\downarrow0$ gives
$1/(c_{\rm T}c_b)$. With $c_{\rm T}=9/32$, $1/(c_{\rm T}c_b)=\frac{32}9\cdot\frac{4u}{1-e^{-u}}\to\frac{128}9$ as $u\to0$.
\end{proof}

\subsubsection{A uniform bound on the integrated likelihood}

\paragraph{Notation.}
\begin{itemize}
\item For $V\subseteq\{1,\dots,p\}$, $L(V)=\{f:f(x)=h(x_V)\ P_X\text{-a.s.}\}$, so $N_j=L(\{j\}^c)$.
\item $\operatorname{clip}_c(u)=\max(-c,\min(c,u))$, $Y_c=\operatorname{clip}_c(Y)$, and $\mathcal X_R=[-R,R]^p$.
\item $\|g\|_{\infty,R}=\sup_{x\in \mathcal X_R}|g(x)|$, and $\mathcal N(\delta,\mathcal S,\|\cdot\|_{\infty,R})$ is the covering
number of a set $\mathcal S$ of functions with centres in $\mathcal S$.
\item $\Pi^I=\sum_{k\in\mathcal K_j^I}(w_k/\omega_j^I)\Pi_k$ is the normalized informative part, which under $\Pi_\omega$ does not
depend on $\omega$.
\end{itemize}

\begin{lemma}[Uniform upper bound on the integrated likelihood]\label{lem:ULLN}
Assume \Cref{ass:regression}. Let $A$ be a measurable set of functions, with $(f,x)\mapsto f(x)$ jointly measurable, $\mu$ a measure with $\mu(A)\le1$, and $r\ge0$. Assume:
\begin{enumerate}
\item[(U$^\sharp$1)] (clipped population gap) for every $\delta_0>0$ there are $c,R$ with
$\inf_{f\in A}\E\big[\ind_{\mathcal X_R}(X)(Y_c-\operatorname{clip}_cf(X))^2\big]\ge\sigma^2+r^2-\delta_0$;
\item[(U$^\sharp$2)] (sieve) there is $\beta>r^2/\sigma^2$ such that for all $\delta'>0$, $c$ and $R$ there are measurable sets $\mathcal G_n$
(possibly depending on $\delta',c,R$) with $\mu(A\setminus \mathcal G_n)\le e^{-\beta n}$ for large $n$ and
$\log\mathcal N(\delta',\operatorname{clip}_c(A\cap \mathcal G_n),\|\cdot\|_{\infty,R})=o(n)$.
\end{enumerate}
Then for every $\delta>0$, $P_0^n\big(\int_A\Lambda_f\,d\mu>2e^{-n(r^2-\delta)/(2\sigma^2)}\big)\to0$.
\end{lemma}
\begin{proof}
\emph{Step 1 (clipping).} Write $\|u\|_n^2=n^{-1}\sum_iu_i^2$. For Gaussian errors,
$n^{-1}\log\Lambda_f=-(\|Y-f\|_n^2-\|\varepsilon\|_n^2)/(2\sigma^2)$. Since $\operatorname{clip}_c$ is 1-Lipschitz,
$(Y_i-f(X_i))^2\ge(Y_{c,i}-\operatorname{clip}_cf(X_i))^2$. Dropping the indices with $X_i\notin \mathcal X_R$,
\[
\|Y-f\|_n^2\ge|Y_c-g|_R^2,\qquad g=\operatorname{clip}_cf,\qquad |u|_R^2:=n^{-1}\textstyle\sum_i\ind_{\mathcal X_R}(X_i)u_i^2 .
\]
Every function now involved is bounded by $c$, so unbounded $f_0$, prior draws and $P_X$ are allowed.

\emph{Step 2 (tolerances).} Given $\delta>0$, choose $\delta_0,\delta'>0$ with $\delta'<\sqrt{\sigma^2+r^2-2\delta_0}$ and
$(\sqrt{\sigma^2+r^2-2\delta_0}-\delta')^2-\sigma^2-\delta_0\ge r^2-\delta$. Take $c,R$ from (U$^\sharp$1), and $\mathcal G_n$ from
(U$^\sharp$2) for these $\delta',c,R$.

\emph{Step 3 (cover).} Let $g_1,\dots,g_{K_n}$ be a $\delta'$-cover of $\operatorname{clip}_c(A\cap \mathcal G_n)$ in $\|\cdot\|_{\infty,R}$,
with $g_k=\operatorname{clip}_cf_k$ and $f_k\in A\cap \mathcal G_n$. If $\|g-g_k\|_{\infty,R}\le\delta'$ then $|g-g_k|_R\le\delta'$. By the
triangle inequality for the seminorm $|\cdot|_R$, $|Y_c-g|_R\ge\sqrt{\hat Q_k}-\delta'$ with $\hat Q_k=|Y_c-g_k|_R^2$.

\emph{Step 4 (concentration).} $\hat Q_k$ is an average of i.i.d.\ terms in $[0,4c^2]$ with mean at least
$\sigma^2+r^2-\delta_0$, by (U$^\sharp$1). Hoeffding's inequality \citep{hoeffding1963} and a union bound give
$P_0^n(\min_k\hat Q_k<\sigma^2+r^2-2\delta_0)\le K_ne^{-n\delta_0^2/(8c^4)}\to0$. By the law of large numbers,
$P_0^n(\|\varepsilon\|_n^2>\sigma^2+\delta_0)\to0$.

\emph{Step 5 (inside the sieve).} Outside these two events, Steps 1--3 and the choice in Step 2 give
$\inf_{f\in A\cap \mathcal G_n}(\|Y-f\|_n^2-\|\varepsilon\|_n^2)\ge r^2-\delta$, hence
$\int_{A\cap \mathcal G_n}\Lambda_f\,d\mu\le e^{-n(r^2-\delta)/(2\sigma^2)}$. The integral is measurable by joint measurability of
$(f,D)\mapsto\Lambda_f(D)$, the same standing assumption that defines the posterior. The exceptional event is defined
through $\min_k\hat Q_k$ and $\|\varepsilon\|_n^2$.

\emph{Step 6 (outside the sieve).} By Fubini, $\E_0\int_{A\setminus \mathcal G_n}\Lambda_f\,d\mu=\mu(A\setminus \mathcal G_n)\le e^{-\beta n}$.
By Markov, $\int_{A\setminus \mathcal G_n}\Lambda_f\,d\mu\le e^{-\beta n/2}$ except on an event of probability at most $e^{-\beta n/2}$.
Since $\beta>r^2/\sigma^2$, $e^{-\beta n/2}\le e^{-n(r^2-\delta)/(2\sigma^2)}$.
\end{proof}

\begin{lemma}[The population gap is automatic for coordinate classes]\label{lem:U0}
Assume \Cref{ass:regression} and $\E f_0(X)^2<\infty$. If $A\subseteq L(V)$, then (U$^\sharp$1) holds with
$r^2=\E\Var(f_0(X)\mid X_V)$. In particular, $A\subseteq N_j$ satisfies (U$^\sharp$1) with $r=\kappa_j$. If
$A=\bigcup_kA_k$, a finite union with $A_k\subseteq L(V_k)$, (U$^\sharp$1) holds with
$r^2=\min_k\E\Var(f_0\mid X_{V_k})$.
\end{lemma}
\begin{proof}
Write $\mathcal X_R=\mathcal X^V\times \mathcal X^{V^c}$, with both factors cubes, and use regular conditional distributions given $X_V$. For
$f\in A$, $\operatorname{clip}_cf$ is a function of $x_V$, hence a constant $u$ given $X_V=x$. So
\[
\E[\ind_{\mathcal X_R}(Y_c-\operatorname{clip}_cf)^2]\ge\E[\ind_{\mathcal X^V}(X_V)\chi_{c,R}(X_V)],\qquad
\chi_{c,R}(x)=\inf_{u\in\mathbb R}\E[\ind_{\mathcal X^{V^c}}(X_{V^c})(Y_c-u)^2\mid X_V=x],
\]
and the right-hand side does not depend on $f$. The function inside the infimum is $au^2-2bu+e_2$, with
$a=P(X_{V^c}\in \mathcal X^{V^c}\mid x)$, $b=\E[\ind_{\mathcal X^{V^c}}Y_c\mid x]$ and $e_2=\E[\ind_{\mathcal X^{V^c}}Y_c^2\mid x]$. So
$\chi_{c,R}=e_2-b^2/a$, with $\chi_{c,R}=0$ if $a=0$. Let $c=R\to\infty$. Since $|Y_c|\le|Y|$ and $\E[Y^2\mid x]<\infty$
a.s., dominated convergence gives $a\to1$, $b\to\E[Y\mid x]$ and $e_2\to\E[Y^2\mid x]$, so
$\chi_{c,R}(x)\to\Var(Y\mid X_V=x)$. Since $0\le\chi_{c,R}\le\E[Y^2\mid X_V]$, which is integrable,
$\E[\ind_{\mathcal X^V}\chi_{c,R}]\to\E\Var(Y\mid X_V)=\sigma^2+\E\Var(f_0\mid X_V)$, using \Cref{ass:regression}(a). For a finite union, take the
largest of the finitely many thresholds $c_0(k,\delta_0)$.
\end{proof}

\begin{lemma}[Sufficient conditions for the sieve]\label{lem:sieve}
Each of the following implies (U$^\sharp$2), for every $\beta$.
\begin{enumerate}
\item[(a)] $\mu$ is carried by a measurable set $\mathcal F$ with $\log\mathcal N(\delta',\mathcal F,\|\cdot\|_{\infty,R})<\infty$ for every
$\delta'$ and $R$; for example, a parametric family $\{f_\theta:\theta\in\Theta\}$ with $\Theta$ compact and
$\theta\mapsto f_\theta$ continuous in $\|\cdot\|_{\infty,R}$ for every $R$.
\item[(b)] $\mu$ is a finite mixture $\sum_kw_k\,\mathrm{GP}(0,k_k)$, $\sum_kw_k\le1$, of centred Gaussian processes with
continuous sample paths. A degenerate component $k_k\equiv0$ is allowed. For every $R$, the unit ball
$\mathbb H_1^{(k)}$ of the RKHS of $k_k$ restricted to $\mathcal X_R$ satisfies
$\log\mathcal N(\epsilon,\mathbb H_1^{(k)},\|\cdot\|_{\infty,R})=o(\epsilon^{-2})$.
\item[(c)] $\mu$ is the law of $f_\theta$ with $\theta\sim\nu_\theta$, for a parametric family $\{f_\theta:\theta\in\mathbb R^{d_\theta}\}$, where $\nu_\theta$ has
sub-Gaussian tails, $\nu_\theta(|\theta|>M)\le C_0e^{-a M^2}$ for all $M>0$, and which is locally Lipschitz in $\theta$ with
polynomial growth: for every $R$ there are $L_R$ and $k_{\rm g}\ge0$ with
$\|f_\theta-f_{\theta'}\|_{\infty,R}\le L_R(1+|\theta|+|\theta'|)^{k_{\rm g}}|\theta-\theta'|$.
\item[(d)] $\mu=\sum_kw_k\mu_k$, a finite sum with $\sum_kw_k\le1$, and each $\mu_k$ satisfying (a), (b) or (c).
\end{enumerate}
Condition (c) covers fully connected networks of fixed architecture with Gaussian weights and biases and
activations $\phi$ satisfying $|\phi(u)-\phi(v)|\le|u-v|$ up to a constant and $|\phi(u)|\le a_0+a_1|u|$ (tanh, ReLU, sine,
GELU, softplus), with $k_{\rm g}$ at most the number of hidden layers $L$: by induction over layers, on $\mathcal X_R$ the $\ell$-th hidden layer is bounded by $C_R(1+|\theta|)^{\ell}$ and each partial derivative of $f_\theta(x)$ in $\theta$ by $C_R(1+|\theta|)^{L}$ (almost everywhere, for ReLU), and integrating along the segment from $\theta$ to $\theta'$ gives the bound. It also covers linear
models and random-feature expansions with Gaussian frequencies, and random weight scales in a bounded range (the
parameter is then a bounded scale mixture of Gaussians, still sub-Gaussian). With (d), it covers a spike-and-slab mixture over
active sets and over finitely many architectures. Gaussian-process components are covered by (b) only for finitely
many kernels, not for continuously distributed length-scales. Families whose draws are discontinuous in $\theta$, such as trees
with random split points, are not covered. For networks with unbounded activations, (M4) holds if $\E|X|^4<\infty$, since for each $\theta$,
$|f_\theta(x)|\le C(\theta)(1+|x|)$.
(Cited.) The entropy condition in (b) holds for squared-exponential kernels, with entropy
$\lesssim(\log1/\epsilon)^{d+1}$ for $d$ inputs \citep[Lemma 4.5]{vdvvz2009adaptive}. That lemma is stated on $[0,1]^d$;
an affine rescaling of $\mathcal X_R$ only changes the constant. It also holds for sums of such kernels, whose RKHS unit ball
lies in the sum of the unit balls \citep{aronszajn1950}, and for kernels with a finite-dimensional RKHS.
\end{lemma}
\begin{proof}
(a) Take $\mathcal G_n=\mathcal F$. Clipping is a contraction, so it does not increase covering numbers: a $\delta'/2$-cover with
arbitrary centres gives a $\delta'$-cover with centres in the set.

(b) Fix $\delta',R$ and a component $k$. A degenerate component is carried by $\{0\}$; take
$\mathcal G_n^{(k)}=\{f:f|_{\mathcal X_R}=0\}$. Otherwise write $\|\cdot\|=\|\cdot\|_{\infty,R}$, and view the restricted process $F$ as a
Gaussian random element of the separable Banach space $C(\mathcal X_R)$.
\begin{itemize}
\item Let $p_0=P(\|F\|\le\delta'/4)$. Then $p_0>0$, because a centred Gaussian measure on a separable Banach space
charges every ball around $0$ (cited: \citealp[Lemma 5.1]{vdvvz2008rkhs}).
\item By Borell's inequality (cited: \citealp{borell1975}; \citealp[Theorem 5.1]{vdvvz2008rkhs}),
$P(F\notin M\mathbb H_1^{(k)}+\frac{\delta'}4\mathbb B_1)\le1-\Phi(\Phi^{-1}(p_0)+M)$ for $M>0$, where $\mathbb B_1$ is the unit ball of $(C(\mathcal X_R),\|\cdot\|)$ and $\Phi$ the standard normal distribution function. With $M_n=\sqrt{2\beta n}-\Phi^{-1}(p_0)$
and $1-\Phi(x)\le e^{-x^2/2}$ for $x\ge0$, this is at most $e^{-\beta n}$.
\item Let $\mathcal G_n^{(k)}=\{f:f|_{\mathcal X_R}\in M_n\mathbb H_1^{(k)}+\frac{\delta'}4\mathbb B_1\}$ and $\mathcal G_n=\bigcup_k\mathcal G_n^{(k)}$. Then
$\mu(\mathcal G_n^c)\le e^{-\beta n}$. Each $\mathcal G_n^{(k)}$ is closed (a compact set plus a closed ball), so $\mathcal G_n$ is measurable.
\item A $\delta'/4$-cover of $M_n\mathbb H_1^{(k)}$ is a $\delta'/2$-cover of $\mathcal G_n^{(k)}$. Clipping preserves this, and
passing to centres in $\operatorname{clip}_c(A\cap \mathcal G_n)$ doubles the radius. So
$\log\mathcal N(\delta',\operatorname{clip}_c(A\cap \mathcal G_n),\|\cdot\|)\le\log\#\{k\}+\max_k\log\mathcal N(\delta'/(4M_n),\mathbb H_1^{(k)},\|\cdot\|)$.
\item Since $M_n\asymp\sqrt n$, the $o(\epsilon^{-2})$ hypothesis makes this $o(n)$.
\end{itemize}
(c) Let $M_n^2=(\beta n+\log^+C_0)/a$ and $\mathcal G_n=\{f_\theta:|\theta|\le M_n\}$, so $\mu(\mathcal G_n^c)\le e^{-\beta n}$. On the ball
$\{|\theta|\le M_n\}$ the map $\theta\mapsto f_\theta$ is Lipschitz into $\|\cdot\|_{\infty,R}$ with constant
$\bar L_n=L_R(1+2M_n)^{k_{\rm g}}$. A Euclidean $\delta'/(4\bar L_n)$-net of the ball, of size at most $(12M_n\bar L_n/\delta')^{d_\theta}$,
therefore gives a $\delta'/4$-cover of $\mathcal G_n$; clipping preserves it, and passing to centres in
$\operatorname{clip}_c(A\cap \mathcal G_n)$ doubles the radius. Hence
$\log\mathcal N(\delta',\operatorname{clip}_c(A\cap \mathcal G_n),\|\cdot\|_{\infty,R})\le d_\theta\log(12M_n\bar L_n/\delta')=O(\log n)=o(n)$.

(d) Let $\mathcal G_n=\bigcup_k\mathcal G_n^{(k)}$ with $\mathcal G_n^{(k)}$ the sets constructed for $\mu_k$. Then
$\mu(A\setminus \mathcal G_n)\le\sum_kw_k\mu_k(\mathcal G_n^{(k)c})\le e^{-\beta n}$, and covering numbers of a finite union add, so the
logarithm increases by at most $\log\#\{k\}$.
\end{proof}

\paragraph{Full hypotheses of \Cref{thm:csharp}.}
Work in the structured setup with the parametrization $\Pi_\omega$, and assume (M4).
\begin{itemize}
\item Part (i) assumes $\psi<\kappa_j$, and that $\mu=\Pi_{\omega_0}|_{N_j}$ satisfies (U$^\sharp$2) with $A=N_j$ and $r=\kappa_j$;
(U$^\sharp$1) then holds automatically, by \Cref{lem:U0}.
\item Part (ii) assumes \Cref{ass:c6} under $\Pi_{\omega_0}$ with $\tau=\sup_n\tau_n<1$ and $\delta_{\rm sp}=\sup_n\delta_n<\frac12$, and that $h_j^*$ lies in
the $L_2(P_X)$-support of $\Pi_{\omega_0}|_{N_j}$.
\item For $\psi>0$, part (ii) also assumes that $\Pi^I$ is carried by a set $A^I$ satisfying (U$^\sharp$1)--(U$^\sharp$2) with
$r=\psi$. If these hold only with some $r\le\psi$, the same proof gives $\liminf\ge2\kappa_j^2/(\kappa_j^2-r^2)$.
\end{itemize}
By \Cref{lem:U0}, the clipped-gap hypothesis holds when $A^I=\bigcup_kA_k$ with $A_k\subseteq L(V_k)$, finitely many $k$,
and the \emph{saturation} condition $\psi^2=\min_k\E\Var(f_0\mid X_{V_k})$ holds. It need not hold for restricted informative
classes, for example additive-in-$x_j$ structures under correlated $P_X$. The hypothesis sets of \Cref{thm:csharp}(i)
and \Cref{cor:prior-weight}(ii) are not nested: the theorem replaces (U1)--(U3) by (U$^\sharp$2) on the $N_j$ part, (M4)
and the support condition.

\begin{proof}[Proof of \Cref{thm:csharp}]
From the proof of \Cref{prop:detect}(b),
$\pi_j/(1-\pi_j)=\sum_{k\in\mathcal K_j}w_k\Lambda(k)/\int_{N_j}\Lambda_f\,d\Pi_\omega$, and
$\sum_{k\in\mathcal K_j^I}w_k\Lambda(k)=\omega\,\Lambda^I$ with $\Lambda^I=\int\Lambda_f\,d\Pi^I$. Also
$\Pi_\omega|_{N_j}=\frac{1-\omega}{1-\omega_0}\Pi_{\omega_0}|_{N_j}$, which lies between $\Pi_{\omega_0}|_{N_j}$ and $(1-\omega_0)^{-1}\Pi_{\omega_0}|_{N_j}$
for $\omega\le \omega_0$. Throughout, $t=n\kappa_j^2/\sigma^2$.

(i)
\begin{itemize}
\item \emph{Prior mass at the gap.} Fix $\eta>0$. Some $f^*\in\operatorname{supp}\Pi^I$ has $\|f_0-f^*\|<\psi+\eta/2$, and
$\Pi^I(\|f-f^*\|<\eta/2)>0$. By (M4) and continuity of measure, $B^I=\{f:\|f_0-f\|\le\psi+\eta,\ \E(f_0-f)^4\le m_4\}$ has
$\Pi^I(B^I)>0$ for large $m_4$.
\item \emph{Numerator.} By the argument of \Cref{lem:ceiling} with $s=2$, with probability at least $\frac34$,
$\log \Lambda^I\ge\log\Pi^I(B^I)-n(\psi+\eta)^2/(2\sigma^2)-2\sqrt{nv^I}$, where $v^I=(\psi+\eta)^2/\sigma^2+m_4/(4\sigma^4)$.
\item \emph{Denominator.} Fix $\delta>0$. By \Cref{lem:U0} and \Cref{lem:ULLN},
$\int_{N_j}\Lambda_f\,d\Pi_{\omega_0}\le2e^{-n(\kappa_j^2-\delta)/(2\sigma^2)}$ with probability at least $1-\zeta_n$, where
$\zeta_n\to0$ does not depend on $\omega$.
\item \emph{Combine.} Choose $n_0$ with $\zeta_n\le\frac18$ for $n\ge n_0$. For $n\ge n_0$, with probability at least
$\frac58$,
$\log\frac{\pi_j}{1-\pi_j}\ge\frac t2\lambda_{\eta,\delta}-2a\sqrt t-\log\frac1\omega-c_1$, where
$\lambda_{\eta,\delta}=1-((\psi+\eta)^2+\delta)/\kappa_j^2>0$ for small $\eta,\delta$, $a=\sqrt{v^I\sigma^2/\kappa_j^2}$ and
$c_1=\log\frac2{1-\omega_0}-\log\Pi^I(B^I)$.
\item \emph{Threshold.} Let $T(\omega)$ be the largest root of $\frac T2\lambda_{\eta,\delta}-2a\sqrt T=\log(1/\omega)+c_1$. Every
$n\ge n_0$ with $t>T(\omega)$ has $P_0^n(\pi_j>\frac12)\ge\frac58$. So
$t_{50}(\omega)\le\max\{n_0\kappa_j^2/\sigma^2,T(\omega)\}+\kappa_j^2/\sigma^2$, and $T(\omega)/\log(1/\omega)\to2/\lambda_{\eta,\delta}$. Let
$\eta,\delta\downarrow0$.
\end{itemize}

(ii), case $\psi>0$.
\begin{itemize}
\item By \Cref{lem:ULLN} with $A=A^I$, $\mu=\Pi^I$ and $r=\psi$, $\Lambda^I\le2e^{-n(\psi^2-\delta)/(2\sigma^2)}$ with probability
at least $1-\zeta'_n$, $\zeta'_n\to0$.
\item As in \Cref{lem:ceiling}, $\int_{N_j}\Lambda_f\,d\Pi_\omega\ge\Pi_{\omega_0}(B_j)e^{-n\bar k_j-s\sqrt{n\bar v_j}}$ with probability at
least $1-1/s^2$.
\item On these events and the event of \Cref{ass:c6},
$\frac{\pi_j}{1-\pi_j}\le\tau+\omega\frac2{\Pi_{\omega_0}(B_j)}\exp(\frac t2\lambda'_{\eta,\delta}+sa'\sqrt t)$. Here
$\lambda'_{\eta,\delta}=(1+\eta/\kappa_j)^2-(\psi^2-\delta)/\kappa_j^2$ and $a'=\sqrt{\bar v_j\sigma^2/\kappa_j^2}$. This is below 1
when $\frac t2\lambda'_{\eta,\delta}+sa'\sqrt t<\log(1/\omega)-c_2$, with $c_2=\log\frac2{\Pi_{\omega_0}(B_j)}-\log(1-\tau)$.
\item Choose $s$ and $n_0'$ with $1/s^2+\zeta'_n+\delta_{\rm sp}<\frac12$ for $n\ge n_0'$, with $\delta_{\rm sp}$ as above (not the tolerance $\delta$). The finitely many $n<n_0'$ have $P_0^n(\pi_j>\frac12)<\frac12$ for small $\omega$, by
\Cref{prop:detect}(b) at fixed $n$.
\item Hence $t_{50}(\omega)\ge T'(\omega)$ with $T'(\omega)/\log(1/\omega)\to2/\lambda'_{\eta,\delta}$. Let $\eta,\delta\downarrow0$. If $\lambda'_{\eta,\delta}\le0$, which can happen only when $\psi\ge\kappa_j$, the condition of the third step holds for all large $t$ once $\omega$ is small, so $t_{50}(\omega)=\infty$ for small $\omega$, consistent with reading $2/\lambda_j=\infty$.
\end{itemize}

(ii), case $\psi=0$: this is \Cref{cor:prior-weight}(i).
\end{proof}

\begin{remark}[The $L_2$ gap is not intrinsic when $\psi>0$]\label{rem:counterexample}
The clipped-gap hypothesis cannot be dropped.
\begin{itemize}
\item \emph{Construction.} For $\ell\ge1$, let $E_\ell$ be disjoint sets with $P_X(E_\ell)=e^{-\ell^4}$, let
$f_\ell=f_0+e^{\ell^4/2}\ind_{E_\ell}$, and let $\Pi^I=\sum_\ell c\,e^{-\ell^2}\delta_{f_\ell}$, with $c$ the normalizing constant.
\item \emph{The $L_2$ gap.} $\|f_\ell-f_0\|=1$ and $\|f_\ell-f_{\ell'}\|^2=2$ for $\ell\ne\ell'$, so $\psi=1$.
\item \emph{The sieve holds.} (U$^\sharp$2) holds for every $\beta$, with $\mathcal G_n=\{f_\ell:\ell\le\sqrt{\beta'n}\}$ for some $\beta'>\beta$.
\item \emph{The evidence does not see the gap.} With probability at least $1-n\,e^{-\ell^4}$, no sample point falls in $E_\ell$, and
then $\Lambda_{f_\ell}=1$. Taking $\ell\asymp(\log n)^{1/4}$ gives $\Lambda^I\ge e^{-o(n)}$.
\end{itemize}
So, under the hypotheses of \Cref{thm:csharp}(i) on the $N_j$ part, the true slope is 2, not $2/\lambda_j$, and (U$^\sharp$1) holds only with $r=0$. In general, the governing quantity is the
exponential rate of the informative evidence, which (U$^\sharp$1) at $r=\psi$ identifies with $\psi^2/(2\sigma^2)$.
\end{remark}

\begin{corollary}[An exactly enumerable toy prior]\label{cor:toy}
Let $p=6$ and $X\sim N(0,I_6)$. Inclusion is i.i.d.\ Bernoulli($0.3$). Given $S$, the structure is either an additive
squared-exponential GP over $S$, or that GP plus one product-kernel term on a pair in $S$. Kernel hyperparameters are fixed (unit length-scale and fixed variances), so the prior is a finite mixture of Gaussian processes and \Cref{lem:sieve} applies. Let
$f_0=\sin(x_1x_2)+0.8x_3$, $j=1$, and let $\mathcal K_1^I$ consist of the structures containing the pair term $(1,2)$.
Then $\limsup_{\omega\to0}t_{50}(\omega)/\log(1/\omega)\le2$, without \Cref{ass:c6}. If \Cref{ass:c6} holds under $\Pi_{\omega_0}$ with $\sup_n\tau_n<1$ and $\sup_n\delta_n<\frac12$
uniformly in $n$, the limit is 2.
\end{corollary}
\begin{proof}
Apply \Cref{thm:csharp} with $\lambda_j=1$.
\begin{itemize}
\item \emph{Structured setup.} For a structure with active set $S$ and $j\in S$, $f-\E[f\mid X_{-j}]$ is a centred Gaussian element of $L_2(P_X)$ with non-zero covariance, because the kernel term in $x_j$ is non-degenerate. Some linear functional of it is therefore a non-degenerate normal variable, so $\Pi_k(N_j)=0$: each structure uses exactly its variables, and \Cref{ass:labels}(a)--(b) hold.
\item \emph{Sieve.} $\Pi_{\omega_0}|_{N_1}$ is a finite mixture of centred GPs whose kernels are sums of one-dimensional
squared-exponential kernels plus at most one product of two of them. A product of two one-dimensional SE kernels is a
two-dimensional SE kernel, so \Cref{lem:sieve}(b) applies.
\item \emph{(M4).} By Fubini, $\E_\Pi\E_Xf(X)^4=3\E_Xk(X,X)^2<\infty$, with $k$ the kernel of the structure, since the kernels are bounded; and $\E f_0^4<\infty$.
\item \emph{$\psi=0$ and the support of $h_1^*$} (cited). The $L_2(P_X)$-support of a centred GP with
$\E\|f\|^2<\infty$ is the closure of its RKHS in $L_2(P_X)$ \citep[Lemma 5.1 and Section 2]{vdvvz2008rkhs}. The RKHS
of a sum kernel is the sum of the RKHSs \citep{aronszajn1950}. Gaussian RKHSs are dense in $L_2(\mu)$ for every finite
Borel measure $\mu$ \citep[Theorem 4.63]{steinwart2008svm}.
\item Hence $\sin(x_1x_2)$ and $0.8x_3$ are $L_2$-limits of RKHS elements of the $(1,2)$-product and $x_3$
components, so $f_0\in\operatorname{supp}\Pi^I$. Similarly, $h_1^*=0.8x_3$ lies in the support of the structure $\{3\}$.
\end{itemize}
\end{proof}

\paragraph{Numerical illustration (not part of any proof).}
The toy prior of \Cref{cor:toy} has 304 structures, enumerated exactly.
\begin{itemize}
\item \emph{Design.} The grid is $n\in\{25,\dots,400\}$ and $\sigma\in\{0.35,\dots,2.5\}$, with 60 replicates per cell and
$\kappa_1^2=0.272$. $t_{50}$ is obtained from a logistic fit of detection on $\log t$. This varies $t$ through $\sigma$ as
well as $n$.
\item \emph{Results.} As $\log(1/\omega)$ increases over $3.9,5.6,7.2,8.8,10.4,12.0,13.6$, $t_{50}$ takes the values $29.2$,
$35.1$, $40.1$, $44.2$, $48.8$, $51.2$, $54.9$. The ratio $t_{50}/\log(1/\omega)$ falls from 7.4 to 4.0 and stays above 2.
\item \emph{Fit.} A fit with the slope fixed at 2, $t_{50}=2\log(1/\omega)+a\sqrt{\log(1/\omega)}+b$, gives RMSE 0.82, against 0.99
for a free straight line. The $\sqrt{\cdot}$ form is a heuristic suggested by the Chebyshev term; the theorem gives no
rate. A fit with free slope is not identified on this range.
\item \emph{\Cref{ass:c6}.} The fraction of data sets with ratio in \Cref{ass:c6} below $\frac12$ is 0.78 for every $\omega$; the invariance in $\omega$ is
guaranteed by the proof and so checks only the code. This pooled figure cannot establish \Cref{ass:c6} for each $n$.
\end{itemize}
For the trained networks on design C of \Cref{sec:experiments}, $x_3,x_4$ (in $\sin(x_3x_4)$) and $x_5$ (in
$\ind\{x_5>0\}x_6$) have similar $\kappa^2$ ($0.275$ and $0.249$) and zero marginal covariance with $Y$. Yet at
$n=800$ and SNR 1, the ablation network without the interaction-only family (\Cref{sec:exp-e8}) detects $x_5$ in 98\% of data sets and $x_3,x_4$ in none. Information alone
(\Cref{prop:detect}(a)) cannot explain this contrast. A difference in informative weight can, because threshold
interactions are covered by the tree family and odd products are not covered by any family of that meta-prior. The controlled prior-weight study of \Cref{sec:exp-e8} tests this mechanism directly.

\subsection{Full statements and proofs for Section~\ref{sec:theory-consistency}}\label{app:proofs-consistency}

\begin{proof}[Proof of \Cref{thm:consistency}]
On the intersection of the \Cref{ass:D1} and \Cref{ass:uniform} events, which has probability tending to one, $\MPM_q=\MPM_\pi=S_0$ once
$\eta_n<\frac12$.
\end{proof}

\paragraph{Setting for \Cref{prop:notransfer}.}
The regression model of \Cref{app:proofs-tools}, with a joint prior $\Pi$ on $(f,\gamma)$, $D\mid(f,\gamma)\sim P_f^n$, and
$\sigma,P_X$ fixed. This idealizes the meta-prior, in which $(P_X,\sigma)$ are also random. The size $(n,p)$ is fixed,
and statements are made about $\Delta_{n,p}$; by \Cref{lem:decomp} they transfer to $\Delta$ with the factor $w(n,p)$.

\begin{assumption}[Sieve and tests]\label{ass:D3}
There are sets $\mathcal F_n$ with $\|f-f_0\|_\infty\le b$ on $\mathcal F_n$, and radii $r_n\to0$ with
$\bar\beta_n:=2\mathcal N_n(r_n)e^{-c_{\rm T}c_bnr_n^2/\sigma^2}\to0$, in the notation of \Cref{lem:tests}.
\end{assumption}

\begin{lemma}[Rates for \Cref{prop:notransfer}(1)]\label{lem:rates}
\begin{enumerate}
\item[(a)] (Nonparametric.) Suppose $-\log\Pi(\|f-f_0\|<r)\ge c\,r^{-a}$ for small $r$, $\Pi(\mathcal F_n^c)\le e^{-n^{c'}}$, and
$\log\mathcal N_n(r)\le C'r^{-a'}$ uniformly in $n$. Then for $r_n=n^{-\upsilon}$ with $0<\upsilon<1/(2+a')$, \Cref{ass:D3} holds and
$\mu_n\le e^{-cn^{a\upsilon}}+e^{-n^{c'}}+2e^{-c_{\rm T}c_bn^{1-2\upsilon}/\sigma^2}$, which is smaller than every power of $n$. The same conclusion holds if the small-ball condition is weakened to $-\log\Pi(\|f-f_0\|<r_n)/\log n\to\infty$, which covers the polylogarithmic small-ball exponents of squared-exponential Gaussian processes.
\item[(b)] (Parametric.) Let $f_0=f_{\theta_0}$ in a family with $\Theta\subset\mathbb R^{d_\theta}$ compact, $|f_\theta|\le b/2$,
$\|f_\theta-f_{\theta'}\|\le L|\theta-\theta'|$, and $\|f_\theta-f_{\theta_0}\|\ge\min(c_0|\theta-\theta_0|,c_1)$. Let the prior
density be bounded by $\bar p_\Theta$, and take $\mathcal F_n=\{f_\theta\}$. Then \Cref{ass:D3} holds for $r_n=n^{-\upsilon}$ with any $\upsilon<\frac12$,
and $\mu_n=O(n^{-\upsilon d_\theta})$. Hence $\Delta_{n,p}=O(n^{-d_\theta/2+\epsilon})$ for every $\epsilon>0$ does not imply \Cref{ass:uniform}.
\end{enumerate}
\end{lemma}

\begin{proof}[Proof of \Cref{prop:notransfer} and \Cref{lem:rates}]
Let $\varphi_n$ be the tests of \Cref{lem:tests}(b) at radius $r_n$ on $\mathcal F_n$, and $A_n=\{\varphi_n=0\}$.

\emph{(1), mass of $A_n$.} $P_f^n(A_n)=P_f^n(1-\varphi_n)$. Split $\Pi$ into $\{\|f-f_0\|<r_n\}$, where this is at most 1;
$\mathcal F_n^c$, where it is at most 1; and $\mathcal F_n\cap\{\|f-f_0\|\ge r_n\}$, where \Cref{lem:tests}(b) bounds it by
$2e^{-c_{\rm T}c_bnr_n^2/\sigma^2}$. So $P_{D,n}(A_n)\le\mu_n$, where $P_{D,n}$ is the prior predictive law of $D$ at size $(n,p)$. Also $P_0^n(A_n)\ge1-\bar\beta_n$.

\emph{(1), construction.} Fix $\varpi_n\in(0,\frac12)$. Let $q^{(n)}=\pi$ off $A_n$, and
$q^{(n)}_j=\max(\varpi_n,\min(1-\varpi_n,1-\pi_j))$ for every $j$ on $A_n$; at all other sizes let $q=\pi$. Off $A_n$ the
divergence is 0. On $A_n$, $\operatorname{kl}(\pi_j\|q_j)\le\max(\log\frac1{q_j},\log\frac1{1-q_j})\le\log(1/\varpi_n)$. So
$\Delta_{n,p}\le\mu_n\log(1/\varpi_n)$, and the bound on $\Delta$ follows from \Cref{lem:decomp}. On the \Cref{ass:D1} event no $\pi_j$
equals $\frac12$, and the clip moves every coordinate to the other side of $\frac12$. So
$P_0^n(\MPM_{q^{(n)}}\ne\MPM_\pi)\ge1-\bar\beta_n-P_0^n(\text{\Cref{ass:D1} fails})\to1$.

\emph{\Cref{lem:rates}(a).} $\log\mathcal N_n(r_n)\le C'n^{a'\upsilon}=o(n^{1-2\upsilon})$ when $\upsilon<1/(2+a')$, so $\bar\beta_n\to0$. The
three terms of $\mu_n$ follow by substitution. Under the weakened condition the first term decays faster than every power of $n$, and the other two are unchanged.

\emph{\Cref{lem:rates}(b).} $\|f_\theta-f_0\|<r<c_1$ forces $|\theta-\theta_0|<r/c_0$, so
$\Pi(\|f-f_0\|<r)\le\bar p_\Theta v_{d_\theta}(r/c_0)^{d_\theta}$, with $v_{d_\theta}$ the volume of the unit ball of $\mathbb R^{d_\theta}$. A Euclidean $\delta$-net of $\Theta$ is an $L\delta$-net in $\|\cdot\|$. The
Hellinger distance is at most $\|\cdot\|/(2\sigma)$, since $1-e^{-v}\le v$ in the formula of \Cref{lem:tests}(a). So
$\log\mathcal N_n(r)\le d_\theta\log(C''/r)$ for a constant $C''$ depending only on $\Theta$, $L$ and $c_b$, and an external net converts to an internal one at twice the radius. Finally,
$d_\theta\log(C''n^\upsilon)=o(n^{1-2\upsilon})$ for every $\upsilon<\frac12$.

\emph{(2).} On $E_n$ some $j$ has $q_j$ and $\pi_j$ on opposite sides of $\frac12$, with $|\pi_j-\frac12|\ge\frac12-\eta_n$.
So $|q_j-\pi_j|\ge\frac12-\eta_n$, and Pinsker gives $\operatorname{kl}(\pi_j\|q_j)\ge2(\frac12-\eta_n)^2$. Hence
$\Delta_{n,p}\ge\frac{2(\frac12-\eta_n)^2}p\,P_{D,n}(E_n)$. Gaussian likelihoods are mutually absolutely continuous, so
$dP_{D,n}/dP_0^n=\int\Lambda_f\,d\Pi$ and $P_{D,n}(E_n)=\E_0[\ind_{E_n}\int\Lambda_f\,d\Pi]$. Let
$\mathcal G_n=\{\int\Lambda_f\,d\Pi\ge\Pi(B_n)e^{-(1+C)n\bar\varepsilon_n^2}\}$. By \Cref{lem:evidence},
$P_0^n(\mathcal G_n^c)\le1/(C^2n\bar\varepsilon_n^2)$, so $P_{D,n}(E_n)\ge\Pi(B_n)e^{-(1+C)n\bar\varepsilon_n^2}(P_0^n(E_n)-P_0^n(\mathcal G_n^c))$.
\end{proof}

\paragraph{Interpretation and open problems.}
Parts (1) and (2) bracket the prior-predictive mass of $P_0$-typical data sets: it is at least
$\Pi(B_n)e^{-(1+C)n\bar\varepsilon_n^2}$, and it can be as small as $\mu_n$. We have not shown that the two sides match in the
nonparametric case. For a parametric family with prior density bounded below at $\theta_0$, both sides are
heuristically of BIC order $n^{-d_\theta/2}$ up to $n^\epsilon$; we do not prove this. \Cref{ass:uniform} must therefore come from the
architecture.

\emph{Conjecture (linear spike-and-slab, fixed $p$).} Under a $g$-prior, $\pi_j$ is an explicit function of
$(n^{-1}X^\top X,n^{-1}X^\top y,n^{-1}y^\top y,n)$, which row attention can compute. For an active $j$ the log-odds grow
like $n\kappa_j^2/(2\sigma^2)$. For a null $j$ they are $\frac12\chi^2_1-\frac12\log n+O(1)$, so the network must resolve
statistics at scale $\log n/n$ and represent the $-\frac12\log n$ penalty. That penalty requires $n$ as an input, which
SelectPFN's selection readout receives through its global token. We conjecture that such a construction exists with
width independent of $n$ (of order $2^p$) and depth polylogarithmic in $n$, building on in-context least squares
\citep{bai2023transformers}. Whether training finds such weights is open.

\emph{\Cref{ass:D1} for the mixture meta-prior} is also open. A route through \citet{ghosal2008nonparametric}, with the model index
taken as (family, structure), appears feasible for linear and Gaussian-process families. With several nonparametric
families approximating $f_0$ at similar rates it is harder, because small-ball upper bounds are not known for
neural-network priors.

\section{Additional experimental results}\label{app:experiments}

\paragraph{Gaussian-process subset prior.}\label{app:gp-subset} The exact-posterior study (\Cref{sec:exp-e1}) uses 2{,}000 tasks from the following prior, under which the posterior can be enumerated over all $2^p$ subsets and the hyperparameter grid. Each task has $p$ uniform on $\{6,\dots,12\}$ and $n$ uniform on $\{50,\dots,300\}$. The inclusion indicators are $\gamma_j\mid w\overset{\text{iid}}{\sim}\Bern(w)$ with $w\sim\mathrm{Beta}(1,2)$. The covariates are, with probability $\frac12$ each, independent $N(0,1)$ or Gaussian $\mathrm{AR}(1)$ with correlation $\rho\sim U(0,0.9)$, standardized. Given $\gamma$, $f\sim\mathrm{GP}\big(0,\sigma_f^2\exp(-\|x_S-x'_S\|^2/(2\ell^2))\big)$ on the active coordinates ($f$ is constant if $S=\varnothing$), and $y=f(x)+\varepsilon$ with $\varepsilon\sim N(0,\sigma_n^2)$. The hyperparameters are uniform on the 27-point grid $(\ell,\sigma_f^2,\sigma_n^2)\in\{0.5,1,2\}\times\{0.5,1,2\}\times\{0.05,0.2,0.5\}$.

\paragraph{Reproducibility.} Every number in \Cref{sec:experiments} and this appendix is regenerated from the stored per-data-set results:
\begin{itemize}
\item \texttt{manuscript/make\_tables.py} regenerates the tables and inline numbers;
\item \texttt{manuscript/make\_figures.py} regenerates the figures;
\item \texttt{experiments/report\_1c.py} produces the summary report.
\end{itemize}
Data sets are generated deterministically: each one is seeded by a hash of its design, $n$, $p$, $\rho$, SNR and replicate index. Paired comparisons therefore use identical data for every method.

\paragraph{Implementation choices that affect the comparisons.}
\begin{itemize}
\item \emph{TabPFN.} TabPFN is ranked by permutation importance on a 30\% split of the training data, and only on 10 replicates per setting, because KernelSHAP at $p=500$ is infeasible for thousands of data sets. TabPFN+SHAP (permutation SHAP with $2p+1$ evaluations per explained row, 20 background rows and 30 explained rows) is run only on the $p=20$ slice.
\item \emph{DART.} DART is \texttt{BART::wbart} with \texttt{sparse=TRUE}. Its PIP is the fraction of posterior draws in which the variable is used; variables are ranked by this PIP, with ties broken by the posterior splitting probability.
\item \emph{GP-ARD.} GP-ARD ranks variables by inverse length-scale.
\item \emph{HSIC-Lasso.} HSIC-Lasso uses 30{,}000 subsampled kernel pairs and ranks variables only.
\item \emph{LassoNet.} LassoNet ranks variables by the penalty level at which each one leaves the regularization path. Variables that never leave the path before it stops are ranked by the magnitude of their skip-connection weights. An earlier version left them tied at $+\infty$; those results were discarded.
\item \emph{Sparse BNN.} The sparse Bayesian network$^\dagger$ is fitted by MAP with an annealed spike variance and pruned at a weight inclusion probability of $\tfrac12$. The prior inclusion probability is chosen from $\{0.5,10^{-1},10^{-2},10^{-4},10^{-6},10^{-8}\}$ on a 20\% validation split.
\item \emph{Group-MCP network.} The group-MCP network$^\dagger$ is a simplified version of GCRNet \citep{luo2025gcsnn} (proximal gradient over a 10-value penalty path, no ridge tuning); it selects its penalty on a 20\% validation split.
\item \emph{BART.} BART is \texttt{BART::wbart} with \texttt{sparse=FALSE}, which fixes the splitting probabilities at $1/p$. Variables are ranked by their average share of splits \citep{bleich2014bart}, and the selected set is $\{j:\mathrm{PIP}_j>\tfrac12\}$. An earlier version ranked by the fixed splitting probabilities, which tie across variables; those results were discarded.
\end{itemize}

\begin{table}[h]
\centering\scriptsize
\caption{Main simulation study: all paired AUPRC comparisons (Holm-adjusted over this table).}\label{tab:e3-paired-full}
\begin{adjustbox}{max width=\linewidth}
\begin{tabular}{lrccc}
\toprule
Comparison & matched & $\Delta$AUPRC [95\% CI] & Wilcoxon $p$ & Holm $p$ \\
\midrule
SelectPFN $-$ DART & 7200 & 0.254 [0.250, 0.259] & $<10^{-300}$ & $<10^{-300}$ \\
SelectPFN $-$ BART & 7200 & 0.216 [0.211, 0.220] & $<10^{-300}$ & $<10^{-300}$ \\
SelectPFN $-$ Lasso & 7200 & 0.242 [0.237, 0.246] & $<10^{-300}$ & $<10^{-300}$ \\
SelectPFN $-$ Elastic net & 7200 & 0.263 [0.258, 0.268] & $<10^{-300}$ & $<10^{-300}$ \\
SelectPFN $-$ SCAD & 7200 & 0.240 [0.236, 0.245] & $<10^{-300}$ & $<10^{-300}$ \\
SelectPFN $-$ MCP & 7200 & 0.244 [0.239, 0.249] & $<10^{-300}$ & $<10^{-300}$ \\
SelectPFN $-$ DC-SIS & 7200 & 0.254 [0.250, 0.258] & $<10^{-300}$ & $<10^{-300}$ \\
SelectPFN $-$ HSIC-Lasso & 7200 & 0.233 [0.230, 0.237] & $<10^{-300}$ & $<10^{-300}$ \\
SelectPFN $-$ GP-ARD & 7200 & 0.250 [0.245, 0.256] & $<10^{-300}$ & $<10^{-300}$ \\
SelectPFN $-$ Group-MCP NN$^\dagger$ & 7200 & 0.267 [0.263, 0.272] & $<10^{-300}$ & $<10^{-300}$ \\
SelectPFN $-$ LassoNet & 7200 & 0.211 [0.207, 0.216] & $<10^{-300}$ & $<10^{-300}$ \\
SelectPFN $-$ Sparse BNN$^\dagger$ & 7200 & 0.395 [0.389, 0.401] & $<10^{-300}$ & $<10^{-300}$ \\
SelectPFN $-$ TabPFN + perm.\ imp. & 1440 & 0.147 [0.140, 0.155] & $1.2\times10^{-183}$ & $2.3\times10^{-183}$ \\
SelectPFN-L $-$ DART & 7200 & 0.271 [0.266, 0.277] & $<10^{-300}$ & $<10^{-300}$ \\
SelectPFN-L $-$ BART & 7200 & 0.233 [0.228, 0.237] & $<10^{-300}$ & $<10^{-300}$ \\
SelectPFN-L $-$ Lasso & 7200 & 0.259 [0.254, 0.263] & $<10^{-300}$ & $<10^{-300}$ \\
SelectPFN-L $-$ Elastic net & 7200 & 0.280 [0.276, 0.285] & $<10^{-300}$ & $<10^{-300}$ \\
SelectPFN-L $-$ SCAD & 7200 & 0.257 [0.252, 0.262] & $<10^{-300}$ & $<10^{-300}$ \\
SelectPFN-L $-$ MCP & 7200 & 0.261 [0.256, 0.266] & $<10^{-300}$ & $<10^{-300}$ \\
SelectPFN-L $-$ DC-SIS & 7200 & 0.271 [0.267, 0.275] & $<10^{-300}$ & $<10^{-300}$ \\
SelectPFN-L $-$ HSIC-Lasso & 7200 & 0.250 [0.246, 0.254] & $<10^{-300}$ & $<10^{-300}$ \\
SelectPFN-L $-$ GP-ARD & 7200 & 0.267 [0.262, 0.272] & $<10^{-300}$ & $<10^{-300}$ \\
SelectPFN-L $-$ Group-MCP NN$^\dagger$ & 7200 & 0.284 [0.280, 0.289] & $<10^{-300}$ & $<10^{-300}$ \\
SelectPFN-L $-$ LassoNet & 7200 & 0.228 [0.224, 0.232] & $<10^{-300}$ & $<10^{-300}$ \\
SelectPFN-L $-$ Sparse BNN$^\dagger$ & 7200 & 0.412 [0.406, 0.418] & $<10^{-300}$ & $<10^{-300}$ \\
SelectPFN-L $-$ TabPFN + perm.\ imp. & 1440 & 0.164 [0.156, 0.172] & $1.8\times10^{-190}$ & $5.5\times10^{-190}$ \\
SelectPFN-L $-$ SelectPFN & 7200 & 0.017 [0.016, 0.018] & $6.1\times10^{-133}$ & $6.1\times10^{-133}$ \\
\bottomrule
\end{tabular}
\end{adjustbox}

\end{table}

\begin{table}[h]
\centering\scriptsize\setlength{\tabcolsep}{2.5pt}
\caption{Main simulation study: mean AUPRC by design [95\% bootstrap interval]. Bold: best method in each design.}\label{tab:e3-design}
\begin{adjustbox}{max width=\linewidth}
\begin{tabular}{lccccccc}
\toprule
Design & SelectPFN & SelectPFN-L & DART & Lasso & HSIC-Lasso & Group-MCP NN$^\dagger$ & TabPFN + perm.\ imp. \\
\midrule
A & \textbf{0.89} {\scriptsize[0.88, 0.90]} & \textbf{0.89} {\scriptsize[0.88, 0.90]} & 0.78 {\scriptsize[0.77, 0.79]} & 0.74 {\scriptsize[0.73, 0.74]} & 0.69 {\scriptsize[0.68, 0.70]} & 0.69 {\scriptsize[0.68, 0.70]} & 0.78 {\scriptsize[0.76, 0.81]} \\
B & 0.84 {\scriptsize[0.83, 0.85]} & \textbf{0.85} {\scriptsize[0.84, 0.86]} & 0.55 {\scriptsize[0.53, 0.56]} & 0.47 {\scriptsize[0.46, 0.48]} & 0.65 {\scriptsize[0.64, 0.66]} & 0.43 {\scriptsize[0.42, 0.44]} & 0.67 {\scriptsize[0.64, 0.69]} \\
C & 0.61 {\scriptsize[0.60, 0.62]} & \textbf{0.65} {\scriptsize[0.64, 0.67]} & 0.22 {\scriptsize[0.21, 0.24]} & 0.23 {\scriptsize[0.22, 0.24]} & 0.30 {\scriptsize[0.29, 0.31]} & 0.26 {\scriptsize[0.25, 0.26]} & 0.41 {\scriptsize[0.38, 0.44]} \\
D & 0.68 {\scriptsize[0.67, 0.69]} & \textbf{0.69} {\scriptsize[0.68, 0.70]} & 0.44 {\scriptsize[0.43, 0.45]} & 0.61 {\scriptsize[0.60, 0.62]} & 0.45 {\scriptsize[0.44, 0.46]} & 0.57 {\scriptsize[0.56, 0.58]} & 0.56 {\scriptsize[0.54, 0.58]} \\
\bottomrule
\end{tabular}
\end{adjustbox}

\end{table}

\begin{table}[h]
\centering\footnotesize
\caption{Calibration: maximum reliability deviation by stratum. A stratum is \emph{observable} if it is defined by $n$, $p$ or the dependence type, and \emph{latent} otherwise. \texttt{null} denotes tasks with $|S|=0$.}\label{tab:e2-strata}
\begin{adjustbox}{max width=\linewidth}
\begin{tabular}{llrcc}
\toprule
Stratum & type & pairs & SelectPFN & SelectPFN-L \\
\midrule
dependence=ar1 & obs. & 269,745 & 0.004 & 0.009 \\
dependence=block & obs. & 264,496 & 0.006 & 0.008 \\
dependence=factor & obs. & 256,654 & 0.014 & 0.012 \\
dependence=indep & obs. & 269,412 & 0.005 & 0.007 \\
n=100-299 & obs. & 390,579 & 0.007 & 0.004 \\
n=<100 & obs. & 372,394 & 0.011 & 0.009 \\
n=$\ge$300 & obs. & 297,334 & 0.006 & 0.011 \\
p=101-500 & obs. & 851,714 & 0.002 & 0.007 \\
p=21-100 & obs. & 174,132 & 0.011 & 0.009 \\
p=5-20 & obs. & 34,461 & 0.025 & 0.015 \\
SNR=1-5 & latent & 355,221 & 0.021 & 0.021 \\
SNR=<1 & latent & 324,855 & 0.021 & 0.017 \\
SNR=$\ge$5 & latent & 303,142 & 0.046 & 0.047 \\
SNR=null & latent & 77,089 & 0.183 & 0.174 \\
family=additive & latent & 115,683 & 0.012 & 0.010 \\
family=gp & latent & 101,376 & 0.041 & 0.043 \\
family=interact & latent & 112,441 & 0.026 & 0.022 \\
family=linear & latent & 143,692 & 0.013 & 0.008 \\
family=mlp & latent & 293,708 & 0.032 & 0.028 \\
family=null & latent & 77,089 & 0.183 & 0.174 \\
family=pairwise & latent & 119,322 & 0.013 & 0.014 \\
family=tree & latent & 96,996 & 0.009 & 0.010 \\
$|S|=0$ & latent & 77,089 & 0.183 & 0.174 \\
$|S|=1-3$ & latent & 620,794 & 0.060 & 0.057 \\
$|S|=4-10$ & latent & 187,803 & 0.022 & 0.021 \\
$|S|\ge 10$ & latent & 174,621 & 0.101 & 0.096 \\
\bottomrule
\end{tabular}
\end{adjustbox}

\end{table}

\begin{table}[h]
\centering\scriptsize
\caption{Support recovery: exact recovery, false-negative rate (FNR) and false-positive rate (FPR) of the MPM, by design and $n$.}\label{tab:e6-full}
\begin{adjustbox}{max width=\linewidth}
\begin{tabular}{lrcccccc}
\toprule
Design & $n$ & SelectPFN exact & SelectPFN FNR & SelectPFN FPR & SelectPFN-L exact & SelectPFN-L FNR & SelectPFN-L FPR \\
\midrule
A & 100 & 0.21 & 0.35 & 0.004 & 0.23 & 0.35 & 0.005 \\
A & 300 & 0.62 & 0.07 & 0.003 & 0.65 & 0.07 & 0.002 \\
A & 800 & 0.94 & 0.00 & 0.001 & 0.96 & 0.00 & 0.000 \\
B & 100 & 0.07 & 0.47 & 0.004 & 0.11 & 0.44 & 0.004 \\
B & 300 & 0.51 & 0.11 & 0.002 & 0.53 & 0.10 & 0.002 \\
B & 800 & 0.85 & 0.02 & 0.001 & 0.85 & 0.02 & 0.001 \\
C & 100 & 0.00 & 0.73 & 0.004 & 0.01 & 0.68 & 0.005 \\
C & 300 & 0.07 & 0.38 & 0.004 & 0.15 & 0.32 & 0.004 \\
C & 800 & 0.24 & 0.22 & 0.002 & 0.41 & 0.16 & 0.001 \\
D & 100 & 0.00 & 0.67 & 0.005 & 0.00 & 0.66 & 0.005 \\
D & 300 & 0.02 & 0.35 & 0.006 & 0.04 & 0.34 & 0.005 \\
D & 800 & 0.23 & 0.16 & 0.003 & 0.36 & 0.13 & 0.003 \\
L & 100 & 0.44 & 0.21 & 0.006 & 0.47 & 0.20 & 0.005 \\
L & 300 & 0.82 & 0.01 & 0.002 & 0.87 & 0.01 & 0.002 \\
L & 800 & 0.92 & 0.00 & 0.001 & 0.97 & 0.00 & 0.000 \\
L10 & 100 & 0.01 & 0.51 & 0.013 & 0.02 & 0.49 & 0.012 \\
L10 & 300 & 0.28 & 0.16 & 0.008 & 0.33 & 0.16 & 0.007 \\
L10 & 800 & 0.53 & 0.05 & 0.004 & 0.59 & 0.05 & 0.002 \\
\bottomrule
\end{tabular}
\end{adjustbox}

\end{table}

\begin{table}[h]
\centering\scriptsize
\caption{Support recovery: test RMSE relative to the noise floor, $\mathrm{RMSE}/\sigma$ (geometric mean). BMA: SelectPFN model average. Oracle: SelectPFN's prediction pass given the true $S_0$. OLS oracle: least squares on the true $S_0$. The SelectPFN-L regression in design B appears in the oracle column as well, so it lies in the prediction pass.}\label{tab:e6-rmse}
\begin{adjustbox}{max width=\linewidth}
\begin{tabular}{lrccccc}
\toprule
Design & $n$ & SelectPFN BMA & SelectPFN oracle & SelectPFN-L BMA & SelectPFN-L oracle & OLS oracle \\
\midrule
A & 100 & 1.280 & 1.180 & 1.268 & 1.173 & 1.326 \\
A & 300 & 1.104 & 1.093 & 1.091 & 1.079 & 1.298 \\
A & 800 & 1.056 & 1.056 & 1.043 & 1.043 & 1.290 \\
B & 100 & 1.364 & 1.253 & 1.366 & 1.276 & 1.658 \\
B & 300 & 1.143 & 1.131 & 1.261 & 1.252 & 1.622 \\
B & 800 & 1.091 & 1.089 & 1.268 & 1.266 & 1.600 \\
C & 100 & 1.589 & 1.358 & 1.536 & 1.298 & 1.853 \\
C & 300 & 1.307 & 1.213 & 1.303 & 1.216 & 1.804 \\
C & 800 & 1.219 & 1.157 & 1.229 & 1.187 & 1.785 \\
D & 100 & 1.424 & 1.279 & 1.414 & 1.274 & 1.337 \\
D & 300 & 1.224 & 1.178 & 1.212 & 1.170 & 1.294 \\
D & 800 & 1.141 & 1.127 & 1.127 & 1.117 & 1.281 \\
L & 100 & 1.130 & 1.059 & 1.137 & 1.068 & 1.031 \\
L & 300 & 1.033 & 1.029 & 1.051 & 1.048 & 1.010 \\
L & 800 & 1.020 & 1.020 & 1.038 & 1.038 & 1.004 \\
L10 & 100 & 1.266 & 1.090 & 1.248 & 1.090 & 1.061 \\
L10 & 300 & 1.068 & 1.041 & 1.074 & 1.049 & 1.019 \\
L10 & 800 & 1.031 & 1.026 & 1.044 & 1.039 & 1.007 \\
\bottomrule
\end{tabular}
\end{adjustbox}

\end{table}

\begin{table}[h]
\centering\scriptsize
\caption{Low-dimensional slice ($p=20$): paired AUPRC differences on matched data sets (Holm-adjusted over this table).}\label{tab:lowdim-paired}
\begin{adjustbox}{max width=\linewidth}
\begin{tabular}{lrccc}
\toprule
Comparison & matched & $\Delta$AUPRC [95\% CI] & Wilcoxon $p$ & Holm $p$ \\
\midrule
SelectPFN $-$ LinearPFN & 2400 & 0.050 [0.046, 0.054] & $2.0\times10^{-123}$ & $6.0\times10^{-123}$ \\
SelectPFN $-$ TabPFN + SHAP & 960 & 0.090 [0.083, 0.096] & $3.5\times10^{-106}$ & $3.5\times10^{-106}$ \\
SelectPFN $-$ DART & 2400 & 0.173 [0.166, 0.181] & $2.0\times10^{-283}$ & $1.0\times10^{-282}$ \\
SelectPFN $-$ Lasso & 2400 & 0.197 [0.191, 0.204] & $<10^{-300}$ & $<10^{-300}$ \\
SelectPFN-L $-$ LinearPFN & 2400 & 0.057 [0.053, 0.061] & $9.8\times10^{-163}$ & $3.9\times10^{-162}$ \\
SelectPFN-L $-$ TabPFN + SHAP & 960 & 0.098 [0.091, 0.105] & $4.9\times10^{-108}$ & $9.7\times10^{-108}$ \\
SelectPFN-L $-$ DART & 2400 & 0.181 [0.173, 0.188] & $3.9\times10^{-291}$ & $2.3\times10^{-290}$ \\
SelectPFN-L $-$ Lasso & 2400 & 0.204 [0.198, 0.211] & $<10^{-300}$ & $<10^{-300}$ \\
\bottomrule
\end{tabular}
\end{adjustbox}

\end{table}

\paragraph{Numerical verification of the prior-weight theory (not part of any proof).}
\Cref{fig:t50} uses a structured GP prior on $p=6$ variables (304 structures) with truth $f_0=\sin(x_1x_2)+0.8x_3$, for which posteriors can be computed exactly. Here $\omega$ is the prior weight of the informative structures for $x_1$, the pair term $(1,2)$, and $t_{50}$ is the information index at which $x_1$ is detected half the time.
\begin{itemize}
\item The threshold $t_{50}$ grows linearly in $\log(1/\omega)$, as \Cref{cor:prior-weight} predicts.
\item The ratio $t_{50}/\log(1/\omega)$ falls from 7.4 to 4.0 across the range, consistent with the limiting slope 2 of \Cref{cor:toy} being approached from above.
\item The fitted curve $2L+a\sqrt{L}+b$, with $L=\log(1/\omega)$, is a heuristic. The theorem proves only $t_{50}=(2+o(1))\log(1/\omega)$.
\end{itemize}

\begin{figure}[h]
\centering
\includegraphics[width=0.48\textwidth]{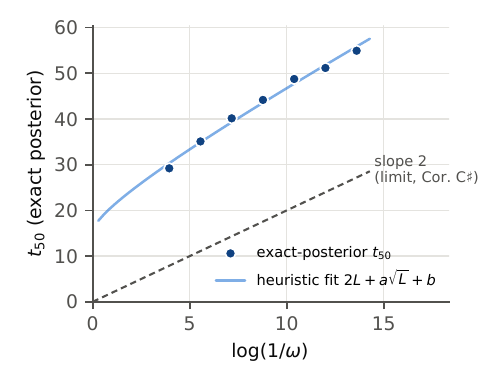}
\caption{Exact-posterior toy problem: detection threshold $t_{50}$ against $\log(1/\omega)$, with the line of slope 2 (the limit in \Cref{cor:toy}) and a heuristic fit. This is a numerical verification and is not used in any proof.}\label{fig:t50}
\end{figure}

\end{document}